\documentclass{article}

\usepackage{paper}

\usepackage[utf8]{inputenc}
\usepackage[T1]{fontenc}
\usepackage{url}
\usepackage{booktabs}
\usepackage{amsfonts}
\usepackage{nicefrac}
\usepackage{microtype}
\usepackage[usenames,dvipsnames]{xcolor}

\usepackage[ruled]{algorithm2e}
\usepackage{float}
\usepackage{enumitem}
\usepackage{amsmath}

\usepackage{amsmath}
\usepackage{amsthm}
\usepackage{amssymb}

\theoremstyle{plain}
\newtheorem{theorem}{Theorem}[section]

\newtheorem{lemma}[theorem]{Lemma}

\newtheorem{corollary}[theorem]{Corollary}

\newtheorem{fact}[theorem]{Fact}
\newtheorem{assumption}[theorem]{Assumption}

\theoremstyle{definition}
\newtheorem{definition}[theorem]{Definition}

\theoremstyle{remark}
\newtheorem{remark}[theorem]{Remark}

\usepackage{mathabx}

\newcommand{\rbr}[1]{\left(#1\right)}
\newcommand{\sbr}[1]{\left[#1\right]}
\newcommand{\cbr}[1]{\left\{#1\right\}}

\DeclareMathOperator*{\argmin}{argmin}

\newcommand{\pair}[1]{\langle{#1}\rangle}

\makeatletter
\newenvironment{algorithm-scheduler}[1][htb]{\renewcommand{\ALG@name}{Scheduler}\begin{algorithm}[#1]}{\end{algorithm}}
\makeatother

\newcommand{\E}{\mathbb{E}}

\DeclareMathOperator*{\ind}{\mathbb I}

\newcommand{\tnorm}[1]{\| #1 \|}
\newcommand{\norm}[1]{\left\| #1 \right\|}

\newcommand{\tfloor}[1]{\lfloor\, {#1}\,\rfloor}

\newcommand{\tsum}{\textstyle\sum}

\newcommand\N{\mathbb N}
\newcommand\R{\mathbb R}
\newcommand\Rp{\mathbb R_{+}}

\newcommand{\Sphere}{\mathbb{S}^{k-1}}
\newcommand{\Ball}{\mathbb{B}^{k}}

\newlength{\pgmtab}
\def\qedsketch{\ifmmode\Box\else{\unskip\nobreak\hfil
\penalty50\hskip1em\null\nobreak\hfil$\Box$
\parfillskip=0pt\finalhyphendemerits=0\endgraf}\fi}

\newlength{\tpush}
\newcommand{\handout}[5]{
   \noindent
   \begin{center}
   \framebox{ \vbox{ \hbox to \textwidth { {\bf \coursenum\ :\  \coursename} \hfill #5 }
       \vspace{3mm}
       \hbox to \textwidth { {\Large \hfill #2  \hfill} }
       \vspace{1mm}
       \hbox to \textwidth { {\it #3 \hfill #4} }
     }
   }
   \end{center}
   \vspace*{4mm}
   \newcommand{\lecturenum}{#1}
   \addcontentsline{toc}{chapter}{Lecture #1 -- #2}
}

\newcommand{\normop}[1]{{\left\vert\kern-0.25ex\left\vert\kern-0.25ex\left\vert #1 
		\right\vert\kern-0.25ex\right\vert\kern-0.25ex\right\vert}}

\newcommand{\cA}{\mathcal{A}}
\newcommand{\cB}{\mathcal{B}}

\newcommand{\cD}{\mathcal{D}}

\newcommand{\cF}{\mathcal{F}}

\newcommand{\cK}{\mathcal{K}}

\newcommand{\cS}{{\mathcal{S}}}

\newcommand{\cW}{\mathcal{W}}

\newcommand{\wtl}{\widetilde{l}}

\newcommand{\wtu}{\widetilde{u}}

\newcommand{\eR}{\mathfrak{R}}

\newcommand{\eD}{\mathfrak{D}}

\newcommand{\wtd}{\widetilde{d}}

\newcommand{\wtf}{\widetilde{f}}
\newcommand{\wtx}{\widetilde{x}}

\newcommand{\wtg}{\widetilde{g}}

\newcommand{\wtddual}{\widetilde d^{\star}}
\newcommand{\wtsigdual}{\widetilde \sigma^{\star}}

\newcommand{\wty}{\widetilde{y}}

\newcommand{\wtdelta}{\widetilde{\delta}}

\newcommand{\wtsig}{\widetilde{\sigma}}

\newcommand{\ddual}{d^{\star}}
\newcommand{\sigdual}{\sigma^{\star}}
\newcommand{\ddualmax}{\ddual_{\textnormal{max}}}
\newcommand{\sigdualmax}{\sigdual_{\textnormal{max}}}

\newcommand{\dtot}{d_{\textnormal{tot}}}

\newcommand{\dmax}{d_{\textnormal{max}}}

\newcommand{\sigmax}{\sigma_{\textnormal{max}}}

\newcommand{\myeq}[1]{\stackrel{\text{(#1)}}{=}}
\newcommand{\myle}[1]{\stackrel{\text{(#1)}}{\le}}

\newcommand{\htgdelta}{\widehat{g}^{\delta}}
\newcommand{\wtgdelta}{\widetilde{g}^{\delta}}

\newcommand{\Unif}{\textnormal{Unif}}

\newcommand{\wteta}{\eta}

\newcommand{\deltatot}{\delta_{\textnormal{tot}}}
\newcommand{\wthtgdelta}{\widetilde{\widehat{g}}^{\delta}}
\newcommand{\ydelta}{y^{\delta}}
\newcommand{\wtydelta}{\widetilde{y}^{\delta}}

\newcommand{\wtvdelta}{\widetilde{v}^{\delta}}
\newcommand{\wtfdelta}{\widetilde{f}^{\delta}}

\usepackage{graphicx}
\usepackage{tikz}
\usetikzlibrary{positioning,arrows.meta,calc}
\usetikzlibrary{decorations.pathreplacing}

\newcommand{\sz}{z}
\newcommand{\wtI}{\widetilde{I}}

\newcommand{\epa}{l}
\newcommand{\epb}{r}
\newcommand{\wtepa}{\widetilde{l}}
\newcommand{\wtepb}{\widetilde{r}}

\newcommand{\Rtwo}{R^{\textnormal{2p}}}
\newcommand{\wbRtwo}{\widebar{R}^{\textnormal{2p}}}

\newcommand{\zcur}{\widebar{z}}

\usepackage{enumitem}
\usepackage{multirow}

\newcommand{\restatetheorem}[2]{
\begingroup
\def\thetheorem{\ref{#1}}
\begin{theorem}[Restated]
#2
\end{theorem}
\addtocounter{theorem}{-1}
\endgroup
}

\newcommand{\restatelemma}[2]{\begingroup
\renewcommand{\thetheorem}{\ref{#1}}\begin{lemma}[Restated]
#2
\end{lemma}
\addtocounter{theorem}{-1}\endgroup
}

\newcommand{\restatecorollary}[2]{\begingroup
\renewcommand{\thetheorem}{\ref{#1}}\begin{corollary}[Restated]
#2
\end{corollary}
\addtocounter{theorem}{-1}\endgroup
}

\usepackage{yfonts}
\usepackage{mathrsfs}
\renewcommand{\cS}{\mathcal{S}}
\renewcommand{\cB}{\mathcal{B}}

\newcommand{\dk}{k}

\usepackage{soul}

\definecolor{DarkRed}{rgb}{0.368,0.097,0.078}
\definecolor{DarkBlue}{rgb}{0.2,0.2,0.6}

\usepackage[
    colorlinks=true,
    linkcolor=DarkBlue,
    anchorcolor=DarkBlue,
    citecolor=DarkBlue,
    filecolor=DarkBlue,
    urlcolor=DarkRed
]{hyperref}

\usepackage[capitalise]{cleveref}
\crefname{lemma}{lemma}{lemmas}
\Crefname{lemma}{Lemma}{Lemmas}
\crefname{corollary}{Corollary}{Corollaries}
\Crefname{corollary}{Corollary}{Corollaries}
\crefname{theorem}{Theorem}{Theorems}
\Crefname{theorem}{Theorem}{Theorems}
\crefname{proposition}{Proposition}{Propositions}
\Crefname{proposition}{Proposition}{Propositions}
\crefname{remark}{Remark}{Remarks}
\Crefname{remark}{Remark}{Remarks}
\crefname{definition}{Definition}{Definitions}
\Crefname{definition}{Definition}{Definitions}
\crefname{assumption}{Assumption}{Assumptions}
\Crefname{assumption}{Assumption}{Assumptions}
\crefname{fact}{Fact}{Facts}
\Crefname{fact}{Fact}{Facts}

\usepackage{caption}
\title{A Reduction from Delayed to Immediate Feedback for Online Convex Optimization with Improved Guarantees}

\author{
Alexander Ryabchenko\thanks{University of Toronto and Vector Institute; \texttt{alex.rbch.research@gmail.com}.}
\qquad
Idan Attias \thanks{Institute for Data, Econometrics, Algorithms, and Learning (IDEAL), hosted by the University of Illinois at Chicago and Toyota Technological Institute at Chicago; \texttt{idanattias88@gmail.com}.}
\qquad
Daniel M. Roy\thanks{University of Toronto and Vector Institute; \texttt{daniel.roy@utoronto.ca}.}
}

\begin{document}

\maketitle

\begin{abstract}
This paper develops a reduction-based framework for online learning with delayed feedback that recovers and improves upon existing results for first-order and bandit convex optimization.
We introduce a continuous-time model that enables an observation-centric rather than prediction-centric analysis.
Under this model, regret decomposes into a delay-independent learning term and a delay-induced drift term, yielding a delay-adaptive reduction that converts any online linear optimization algorithm into one that handles round-dependent delays.
For bandit convex optimization, we significantly improve existing regret bounds, with delay-dependent terms matching optimal first-order rates.
For first-order feedback, we recover optimal regret bounds via a simpler, unified analysis.

Quantitatively, for bandit convex optimization we obtain $O(\sqrt{d_{\text{tot}}} + T^{\frac{3}{4}}\sqrt{k})$ regret, improving the delay-dependent term from $O(\min\{\sqrt{T d_{\text{max}}},(Td_{\text{tot}})^{\frac{1}{3}}\})$ in prior work to $O(\sqrt{d_{\text{tot}}})$. Here, $k$, $T$, $d_{\text{max}}$, and $d_{\text{tot}}$ denote the dimension, time horizon, maximum and total delay, respectively.
Under strong convexity, we achieve $O(\min\{\sigma_{\text{max}} \ln T, \sqrt{d_{\text{tot}}}\} + (T^2\ln T)^{\frac{1}{3}} {k}^{\frac{2}{3}})$, 
improving the delay-dependent term from $O(d_{\textnormal{max}} \ln T)$ to $O(\min\{\sigma_{\text{max}} \ln T, \sqrt{d_{\text{tot}}}\})$,
where $\sigma_{\text{max}}$ denotes the maximum number of outstanding observations and may be considerably smaller than $d_{\text{max}}$.
\end{abstract}

\section{Introduction}

Online convex optimization (OCO) is a fundamental framework for sequential decision-making under uncertainty~\citep{hazan2023introductiononlineconvexoptimization, orabona_book}. In the classic OCO setting, a player repeatedly selects an action from a convex domain, incurs a loss specified by a convex loss function, and then receives feedback about that loss. In the full-information setting, the entire loss function is revealed. In the first-order setting, the player observes the gradient of the loss at the chosen action. In the zeroth-order (bandit) setting, the player observes only the scalar loss value.

In practice, observations are often \emph{delayed}: feedback from round $t$ may arrive only at round $t + d_t$, where $d_t$ is the round-dependent delay. This forces the player to make decisions without up-to-date information. For example, consider a recommendation system that must select content instantaneously for each arriving user. Feedback needed to update the system's parameters (such as clicks or purchases) arrives only after user sessions conclude. When many users are served concurrently, new arrivals must be handled while feedback from earlier users is pending, resulting in delays. Delayed feedback has been studied extensively in many variants of sequential decision-making, including online learning, bandits, and reinforcement learning (e.g., \citealp[]{mesterharm2005, cesa16, zimmert2020, vanderhoeven2023}).

A natural question is whether delayed algorithms can be obtained as reductions from non-delayed ones. Prior work has made significant progress on this question \citep{weinberger2002, joulani2013,joulani2016, flaspohler2021}. However, existing approaches either require advance knowledge of the delay process, incur memory overhead scaling with the amount of outstanding feedback, or are limited to full-information feedback. Developing a unified, delay-adaptive reduction---one that requires no advance knowledge of delays, adapts to function regularity (e.g., strong convexity), and accommodates different feedback models (first-order or bandit)---is the main goal of this work.

In this paper, we introduce a continuous-time model for delayed feedback that subsumes the standard discrete-time model and enables a clean reduction from delayed feedback to immediate feedback for both online convex optimization with first-order feedback (OCO) and bandit convex optimization (BCO). For BCO, this reduction yields improved regret rates, while for OCO it recovers optimal bounds via a simple, unified analysis. Our contributions are summarized as follows:

\begin{enumerate}[leftmargin=0.5cm, after=\vspace{-0.0cm}, before = \vspace{-0.1cm}, itemsep=0.05cm]
    \item \textbf{Continuous-time model (Section~\ref{sec:CTM}).}
    We introduce a continuous-time model for delayed feedback. Instead of the usual round-based view in which each round reveals an unordered batch of observations, we place predictions and observations as events on a shared timeline, and show that the effective delay structure is induced by their interleaving.
    
    Theorem~\ref{thm:CTM} characterizes this structure, identifying new invariants and equivalences among delay-related quantities that are widely used across the delayed-learning literature.
    Ordering the interaction by observation times provides a natural framework for analyzing algorithms that run single-instance updates upon each feedback arrival (e.g., \citet{joulani2016}), yielding a regret decomposition into a non-delayed term and a prediction-drift term (Lemma~\ref{lem:regret-decomposition}). Our analysis provides new tools for controlling this drift, which we exploit to improve regret bounds through a unified analysis, without requiring knowledge of problem parameters such as the maximum delay~$\dmax$, total delay~$\dtot$, maximum number of outstanding observations~$\sigmax$, or the horizon~$T$.

    \item \textbf{Reduction to non-delayed learning (Sections~\ref{sec:drift-penalization}--\ref{sec:OCO}).} Building on this decomposition, we reduce delayed OCO to online linear optimization with penalties for changing predictions, a delay-free problem related to online learning with switching costs \citep{cesabianchi2013switching}. We provide wrappers that transform any algorithm for this setting into a delayed OCO or BCO algorithm whose regret is controlled by the base algorithm's performance (Theorems~\ref{thm:OCO-decomp} and~\ref{thm:BCO-decomp}). The BCO wrapper incorporates single-point gradient estimation \citep{flaxman2004}.

    \item \textbf{First-order feedback (Section~\ref{sec:reduction-OCO}).} Wrapping Proximal Follow-The-Regularized-Leader (P-FTRL) and Online Mirror Descent (OMD) update rules recovers optimal bounds for delayed OCO: $O(\sqrt{\dtot} + \sqrt{T})$ for convex losses (Corollary~\ref{cor:OCO-regular}) and $O(\min\{\sigmax \ln T, \sqrt{\dtot}\} + \ln T)$ for strongly convex losses (Corollary~\ref{cor:OCO-sc}), matching \citet{quanrud2015} and \citet{qiu2025} respectively via an adaptive approach requiring no knowledge of delays or horizon.
    
    \item \textbf{Bandit feedback (Section~\ref{sec:BCO}).} For delayed BCO, wrapping P-FTRL achieves $O(\sqrt{\dtot} + T^{3/4}\sqrt{\dk})$ expected regret (Corollary~\ref{cor:BCO-regular}), where $k$ is the dimension of the domain, improving the delay dependence over $(T\dtot)^{1/3}$ of \citet{bistritz2022} and $\sqrt{T\dmax}$ of \citet{wan2024improvedregretbanditconvex}.  For strongly convex losses, we obtain $O(\min\{\sigmax \ln T, \sqrt{\dtot}\} + (T^2 \ln T)^{1/3} {\dk}^{2/3})$ regret (Corollary~\ref{cor:BCO-sc}), improving over $O(\dmax \ln T + (T^2 \ln T)^{1/3} {\dk}^{2/3})$ of \citet{wan2024improvedregretbanditconvex}. Note that $\sigmax$ is always at most $d_{\textnormal{max}}$ and can be much smaller in practice.\footnote{For example, if $d_1=\Theta(T)$ and $d_t=0$ for all $t\ge 2$, then $\dmax=\dtot=\Theta(T)$ while $\sigmax=1$.}
    See Table~\ref{tab:regret-summary} for a summary.

    \item \textbf{Two-point bandit feedback (Appendix~\ref{app:two-point-BCO}).} To further illustrate the applicability of our reduction, we provide the first regret bounds for delayed BCO under the two-point feedback model of \citet{agarwal2010}. Using the two-point gradient estimator \citep{shamir2015}, we obtain regret $O(\sqrt{\dtot} + \sqrt{T\dk})$ for convex losses and $O(\min\{\sigmax \ln T, \sqrt{\dtot}\} + k \ln T)$ under strong convexity.
    Our bounds simultaneously recover the known lower-bound dependences for delay-free two-point BCO and delayed OCO in both the convex and strongly convex regimes~\citep{duchi2014optimalrateszeroorderconvex,agarwal2010,weinberger2002}, establishing rate optimality for delayed two-point BCO.
    See Table~\ref{tab:two-point-regret-summary} for a summary.

    \item \textbf{Skipping scheme (Section~\ref{app:skipping}).} We incorporate the adaptive skipping technique of \citet{zimmert2020}, as an external wrapper (Algorithm~\ref{alg:skipping-scheme}), improving the delay term from $O(\sqrt{\dtot})$ to $O(\min_{Q\subseteq [T]}\{|Q| + ({\sum_{t \notin Q} d_t})^{\frac{1}{2}} \})$. This yields substantial gains when a few rounds have exceptionally large delays (Corollary~\ref{cor:skipping-OCO-BCO}).
\end{enumerate}

\paragraph{Related work.} Methodologically, the work closest to ours is \citet{joulani2016}, who studied delayed OCO with full-information feedback. They propose a black-box-style reduction that feeds observations to a base (non-delayed) algorithm in order of arrival, bounding regret in terms of the base algorithm's regret and prediction drift. With known $\dtot$, they achieve regret $O(\sqrt{T+\dtot})$. They also establish a bound of $O(\sqrt{T\tau^*})$, where $\tau^*$ is the maximum number of observations that can arrive while some earlier feedback is still outstanding, known in advance. Using the tools developed in this paper, we show that $\tau^* = \Theta(\dmax)$, thus expressing this result in standard notation as $O(\sqrt{T\dmax})$.

Delayed online learning traces back to \citet{weinberger2002}, whose subsampling argument for fixed delays $d_t = d_{0}$ implies minimax rates of order $\Theta(\sqrt{T(d_0+1)})$ and $\Theta((d_0+1)\ln T)$ when applied to non-delayed problems with square-root and logarithmic regret, respectively. These fixed-delay rates imply $\Omega(\sqrt{\dtot + T})$ and $\Omega((\sigmax + 1)\ln T)$ lower bounds for arbitrary delays under convex and strongly convex losses, respectively. Matching upper bounds for first-order OCO were obtained by \citet{quanrud2015} for convex losses and by \citet{qiu2025} for strongly convex losses; earlier, \citet{wan2021onlinestronglyconvexoptimization} achieved the rate $O(\dmax \ln T)$.
\citet{heliou2020} first extended the bandit gradient descent method of \citet{flaxman2004} to delayed BCO. \citet{bistritz2022} developed the first fully delay-adaptive algorithm using a delayed doubling trick, and \citet{wan2024improvedregretbanditconvex} improved the bound via a blocking strategy in the regime when $\dmax \lesssim \dtot^{2/3}/T^{1/3}$.

For additional related work on OCO and online learning with delayed feedback, see Appendix~\ref{app:related-work}.

\begin{table*}[t]
\caption{Notation: $T$ is the horizon, $\dk$ is the dimension, $\dtot = \sum_{t=1}^T d_t$ is the total delay, $\dmax = \max_{t \in [T]} d_t$ is the maximum delay, and $\sigma_{\textnormal{max}} \le \dmax$ is the maximum number of outstanding observations at any round.
The bounds of \citet{wan2024improvedregretbanditconvex} require advance knowledge of $\dmax$ and $T$; all other algorithms are fully adaptive.}
\centering
\setlength{\tabcolsep}{10pt}
\renewcommand{\arraystretch}{1.3}
\resizebox{\textwidth}{!}{
\begin{tabular}{|c|c|c|c|}
\hline
\multirow{2}{*}{\textbf{Loss type}} &
\multicolumn{3}{c|}{\textbf{Regret bounds for bandit convex optimization with delays}} \\ \cline{2-4}
& {\citet{bistritz2022}} &
  {\citet{wan2024improvedregretbanditconvex}} &
  {Our work} \\
\hline
Convex
  & $(T \dtot)^{1/3} \sqrt{\dk} + T^{3/4} \dk$
  & $\sqrt{T\dmax} + T^{3/4} \sqrt{\dk}$
  & $\sqrt{\dtot} + T^{3/4} \sqrt{\dk}$ \\[2pt]
\hline
Strongly convex
  & N/A
  & $d_{\textnormal{max}} \ln T + (T^2 \ln T)^{1/3} {\dk}^{2/3} $
  & $\min\{\sigma_{\textnormal{max}}\ln T,\sqrt{d_{\text{tot}}}\} + (T^2 \ln T)^{1/3} {\dk}^{2/3}$ \\[2pt]
\hline
\end{tabular}
}
\label{tab:regret-summary}
\end{table*}

\begin{table*}[t]
\caption{Regret bounds for delayed OCO, two-point BCO, and delayed two-point BCO.}
\centering
\setlength{\tabcolsep}{8pt}
\renewcommand{\arraystretch}{1.3}
\resizebox{\textwidth}{!}{
\begin{tabular}{|c|c|c|c|}
\hline
\multirow{3}{*}{\textbf{Loss type}} &
\multicolumn{3}{c|}{\textbf{Regret bounds}} \\ \cline{2-4}
& \textbf{Delayed OCO}
& \textbf{Two-point BCO}
& \textbf{Delayed two-point BCO} \\ \cline{2-4}
& \citet{quanrud2015,qiu2025}
& \citet{shamir2015}
& Our work \\
\hline
Convex
  & $\sqrt{\dtot}+\sqrt{T}$
  & $\sqrt{T\dk}$
  & $\sqrt{\dtot}+\sqrt{T\dk}$ \\[2pt]
\hline
Strongly convex
  & $\min\{\sigmax \ln T,\sqrt{\dtot}\}+\ln T$
  & $\dk\ln T$
  & $\min\{\sigmax \ln T,\sqrt{\dtot}\}+\dk\ln T$ \\[2pt]
\hline
\end{tabular}
}
\label{tab:two-point-regret-summary}
\end{table*}

\section{Problem Setting}\label{sec:setting}

Let $T \in \N$ denote the time horizon, and let $\cK \subset \R^{\dk}$ be a non-empty, convex, and closed domain for the dimension $\dk \in \N$. We equip $\R^{\dk}$ with a norm $\tnorm{.}$ and its dual norm $\tnorm{.}_{\star}$.

\begin{figure}[!b]
\centering
\resizebox{\columnwidth}{!}{\fbox{\parbox{\dimexpr1.05\columnwidth-2\fboxsep-2\fboxrule\relax}{
    \normalsize
    \textbf{Online Convex Optimization with Delays}
    \vspace{2pt}
    
    $\bullet$ \textit{Latent parameters:} number of rounds $T$.\\
    $\bullet$ \textit{Pre-game:} adversary selects loss functions $f_{t}:\cK \to \R$ and delays $d_t \in \{0,\ldots,T-t\}$ for $t \in [T]$.
    
    \vspace{3pt}
    \noindent For each round $t=1, 2, \ldots, T$:
    \begin{enumerate}[leftmargin=0.5cm,noitemsep,before=\vspace{-0.0cm},after=\vspace{-0.1cm}]
        \item The player predicts $x_t \in \cK$, incurring loss $f_t(x_t)$.
        \item For $s$ such that $s+d_s = t$, the environment reveals $(s, \nabla f_s(x_s))$ in the first-order feedback model and $(s, f_s(x_s))$ in the bandit feedback model.
    \end{enumerate}
}}
}
\captionsetup{font=footnotesize}
\caption{\small OCO with delays under first-order or bandit feedback.}
\label{fig:setting}
\end{figure}

The player interacts with an environment determined by an oblivious adversary over $T$ rounds (Figure~\ref{fig:setting}). Before the game begins, the adversary selects delays $d_t$ and convex, differentiable loss functions $f_t:\cK\to\R$ for each round. Every round $t$, the player predicts $x_t \in \cK$, receiving the corresponding feedback (gradient $\nabla f_t(x_t)$ or value $f_t(x_t)$, depending on the feedback model) at the end of round $t+d_t$. We assume $d_t \in \{0,\ldots,T-t\}$ for all $t \in [T]$, since feedback arriving after round $T$ cannot affect any prediction.

The regret against a comparator $u \in \cK$ is defined as $R_T(u) = \tsum_{t=1}^T (f_t(x_t) - f_t(u))$, and its expected regret  as $\widebar{R}_T(u) = \E[R_T(u)]$. Letting $x^* \in \argmin_{x\in\cK} \tsum_{t=1}^T f_t(x)$ denote the best action in hindsight, the objective is to minimize $R_T(x^*)$ almost surely or its expectation $\widebar{R}_T(x^*)$.

\paragraph{Regularity assumptions.} We impose the following standard regularity assumptions, which are necessary to obtain sublinear regret guarantees in online convex optimization.
\begin{assumption}\label{assump:finite-diameter} 
    The diameter of $\cK$ is bounded by $D$, i.e., $\sup_{x,y\in\cK} \norm{x-y} \le D$.
\end{assumption}
\begin{assumption}\label{assump:gradient-norm-bound}
    For all $t \in [T]$, $f_t$ has its gradient norm bounded by $G$, i.e., $\sup_{x\in\cK} \norm{\nabla f_t(x)}_{\star}\!\le\!G$.
\end{assumption}

Assumptions \ref{assump:finite-diameter} and \ref{assump:gradient-norm-bound} hold throughout the paper. When deriving stronger guarantees, we also impose the following strong convexity condition, stated explicitly when invoked.
\begin{assumption}\label{assump:strong-convexity}
For all $t \in [T]$, $f_t$ is $\lambda$-strongly convex for $\lambda > 0$, i.e., $f_t(y) \ge f_t(x) + \pair{\nabla f_t(x),\, y - x} + \tfrac{\lambda}{2}\norm{y - x}^2$ for all $x,y\in \cK$.
\end{assumption}

For BCO, we take $\tnorm{\cdot}$ to be the Euclidean norm $\tnorm{\cdot}_2$ and impose two additional assumptions.
\begin{assumption}\label{assump:max-val}
    For all $t \in [T]$, absolute values of $f_t$ are bounded by $M$, i.e., $\sup_{x\in\cK} |f_t(x)| \le M$.
\end{assumption}
\begin{assumption}\label{assump:balls}
    The domain satisfies $r\Ball \subseteq \cK \subseteq R\Ball$ for some radii $0 < r < R$. 
\end{assumption}

\paragraph{Steady algorithms.} An algorithm in this setting is a (possibly randomized) mapping from the history of received feedback to predictions. The class of \emph{steady algorithms} is formally defined below.
\begin{definition}[Steady Algorithm]\label{def:steady-alg}
    An algorithm for online learning with delays is \emph{steady} if it does not change its prediction after rounds without feedback: $x_t = x_{t+1}$ whenever $\{s:s+d_s=t\}=\emptyset$. It is \emph{steady in expectation} if $\E[x_t]=\E[x_{t+1}]$ whenever $\{s : s+d_s=t\}=\emptyset$.
\end{definition}

\paragraph{Notation.} For $n \in \N$, we let $[n] = \{1,\ldots,n\}$. For any sequence $\{\mathrm{a}_t\}_{t=1}^T$, we write $\mathrm{a}_{\textnormal{tot}} = \tsum_{t=1}^T \mathrm{a}_t$ and $\mathrm{a}_{\textnormal{max}} = \max_{t\in [T]} \mathrm{a}_t$. We denote by $\Ball$ and $\Sphere$ the unit Euclidean ball and sphere centered at the origin in $\R^{\dk}$, respectively. For any $S \subseteq \R^{\dk}$ and $c \in \R$, we let $cS = \{cx : x \in S\}$.

\section{Continuous Time Model for Delays}\label{sec:CTM}
The classical formulation of online learning with delays assumes a discrete, round-based time structure: a player makes a prediction in round $t$ and receives feedback in round $t + d_t$. Real-world systems, however, generate predictions and observations asynchronously. Rather than forcing these events into discrete rounds, we adopt a continuous time model that places them on a timeline $\Rp$.

\begin{definition}[Continuous Time Model (CTM)]\label{def:CTM}
     The timing of predictions and observations is determined by latency-intervals $\{I_t\}_{t=1}^T$, where $I_t = (\epa_t, \epb_t)\subset \Rp$. For every $t \in [T]$, the environment requests a prediction at time $\epa_t$ and reveals the corresponding observation at time $\epb_t$. All timestamps $\{\epa_t, \epb_t\}_{t=1}^T$ are distinct, and intervals are indexed in increasing order of
prediction times: $0 < \epa_1 < \ldots < \epa_T$.
\end{definition}

In the round-based model, the delay sequence $\{d_t\}_{t=1}^T$ fully specifies the time structure. In particular, it determines the number of outstanding observations in each round, $\sigma_t = \sum_{s=1}^{t-1} \ind(s + d_s \ge t)$, which is commonly used in the literature; for brevity, we refer to $\sigma_t$ as the \emph{backlog} in this paper. Under the CTM, both delays and backlogs arise directly from the geometry of the latency intervals.

Section~\ref{subsec:delay-structure} formalizes this structure. Building on it, Section~\ref{subsec:steady-algs} introduces an observation-centric view of the problem that complements the standard prediction-centric perspective. We show that steady algorithms admit a natural regret decomposition: non-delayed regret on observation-ordered losses plus delay-induced prediction drift (Lemma~\ref{lem:regret-decomposition}).

\subsection{Delays, Backlogs, and Their Duals}\label{subsec:delay-structure}

Under the CTM, rounds of the delayed problem (Fig.~\ref{fig:setting}) naturally correspond to time-intervals: round $t$ spans the interval $[\epa_t, \epa_{t+1})$, during which the $t$-th prediction is made at time $\epa_t$, and observations arrive at times $\{\epb_s: \epb_s \in (\epa_t, \epa_{t+1})\}$. Therefore, the discrete \emph{delay} $d_t$ and \emph{backlog} $\sigma_t$ arise as
\begin{align}\label{eq:delay-backlog-defn}
    d_t = |\{s:\, \epa_s \in I_t\}| \quad \text{and} \quad \sigma_t = |\{s:\, \epa_t \in I_s\}|,
\end{align}
where $d_t$ counts how many predictions are made between the $t$-th prediction and observing its feedback, and $\sigma_t$ counts how many previous predictions still await feedback immediately before making the $t$-th prediction.

The next lemma and remark show that these definitions recover the classical round-based formulation.

\begin{lemma}\label{lem:CTM-props}
    For quantities in \eqref{eq:delay-backlog-defn}, it holds that:
    \[
        \epb_t \in (\epa_{t+d_t}, \epa_{t+d_t+1}) \quad \text{and} \quad \sigma_t = |\{s: s < t \le s+d_s\}|.
    \]
\end{lemma}
\begin{proof}
    Exactly $|\{s\!:\!\epa_s\!<\!\epb_t\}|=|\{s\!:\epa_s\!\le\!\epa_t\}|\!+\!|\{s\!:\!\epa_t\!<\!\epa_s\!<\!\epb_t\}| = t+d_t$ prediction times precede time $\epb_t$, so $\epb_t \in (\epa_{t+d_t}, \epa_{t+d_t+1})$. Then, the identity $\sigma_t = |\{s: \epa_s < \epa_t < \epb_s\}|$ follows immediately.
\end{proof}

\begin{remark}[Consistency and Realizability in the CTM]
    Lemma~\ref{lem:CTM-props} confirms that our definitions are consistent with round-based protocols: $d_t$ equals the number of round-intervals $[\epa_{s}, \epa_{s+1})$ spanned by $I_t$, placing observation time $\epb_t$ in round $t+d_t$; $\sigma_t$ matches the classical definition as the number of outstanding delays. Conversely, any round-based delayed-feedback protocol can be represented within the CTM. Specifically, given a delay sequence $\{d_t\}_{t=1}^T$, we can set $(\epa_t, \epb_t) = (t, t + d_t + 1 - 2^{-t})$ so that round $t$ spans interval $[t,t+1)$ and observation time $\epb_t$ clearly falls within round $t+d_t$ spanned by interval $[t+d_t, t+d_t+1)$.
\end{remark}

To characterize the delay structure in the CTM, we introduce \emph{dual-delay} and \emph{dual-backlog} sequences
\begin{align}\label{eq:dual-delay-backlog-defn}
    \ddual_t = |\{s:\, \epb_s \in I_t\}| \quad \text{and} \quad \sigdual_t = |\{s:\, \epb_t \in I_s\}|.
\end{align}
Here, dual-delay $\ddual_t$ counts observations between the $t$-th prediction and its observation, while $\sigdual_t$ counts how many predictions still await feedback when receiving observation for the $t$-th prediction.

The following theorem establishes the fundamental properties and relations among the four sequences.

\begin{theorem}\label{thm:CTM}
    For quantities in \eqref{eq:delay-backlog-defn} and \eqref{eq:dual-delay-backlog-defn}, it holds that:
    \begin{enumerate}[label=(\alph*),ref=(\alph*),noitemsep,leftmargin=1cm,before=\vspace{-0.2cm},after=\vspace{-0.2cm}]
        \item $\sum_{t=1}^T d_t = \sum_{t=1}^T \sigma_t = \sum_{t=1}^T \ddual_t = \sum_{t=1}^T \sigdual_t$,
        \item $\sigmax = \sigdualmax, \quad \tfrac{1}{2} \dmax \le \ddualmax \le 2\dmax$,
        \item $\ddual_t = \sigma_t + \beta(t) - t, \quad \sigdual_t = d_t + t - \beta(t)$,
    \end{enumerate}
    where $\beta\!:\![T]\!\to\![T]$ is the permutation $\beta(t)\!=\!|\{s\!:\!\epb_s\!\le\!\epb_t\}|$.
\end{theorem} 

In particular, all four sequences sum to the total delay $\dtot$. Figure~\ref{fig:illustration} illustrates this in a concrete example: delays $d_t$ (dual-delays $\ddual_t$) correspond to black (white) triangles along intervals $I_t$, while backlogs $\sigma_t$ (dual-backlogs $\sigdual_t$) correspond to black (white) triangles at timestamps $\epa_t$ ($\epb_t$), with the total number of triangles of each color $\dtot = 6$.

\begin{figure}[t]
\centering

\makebox[\columnwidth][c]{\scalebox{1.05}{\begin{minipage}{\columnwidth}
\centering
\begin{minipage}[c]{0.54\columnwidth}
    \centering
    \makebox[\linewidth][l]{\hspace*{-0.7cm}\resizebox{1.05\linewidth}{!}{\newcommand{\xunit}{0.85cm}
\newcommand{\yunit}{0.35cm}
\newcommand{\tshift}{0.5}

\usetikzlibrary{decorations.pathreplacing}
\usetikzlibrary{shapes.geometric}

\newcommand{\drawinterval}[3]{
  \draw[thick] (#1+\tshift,#3) -- (#2+\tshift,#3);

  \node[fulldot]  at (#1+\tshift,#3) {};
  \node[opendot]  at (#2+\tshift,#3) {};

  \draw[proj] (#1+\tshift,#3) -- (#1+\tshift,0);
  \draw[proj] (#2+\tshift,#3) -- (#2+\tshift,0);

  \node[fulldot]  at (#1+\tshift,0) {};
  \node[opendot]  at (#2+\tshift,0) {};
}

\newcommand{\rowbrace}[4]{\draw[decorate,decoration={brace,mirror,amplitude=4pt}]
    (#1+\tshift -0.2 ,#3-0.2) -- (#2+\tshift -0.25,#3-0.2)
    node[midway, yshift=-8pt] {\tiny #4};
}

\begin{tikzpicture}[x=\xunit,y=\yunit]

  \tikzset{
    fulldot/.style = {circle, fill=black, draw=black, inner sep=1.5pt},
    opendot/.style = {circle, fill=white, draw=black, inner sep=1.5pt},
    fullsquare/.style = {regular polygon, regular polygon sides=3, fill=black, draw=black, inner sep=0.8pt, shape border rotate=180},
    opensquare/.style = {regular polygon, regular polygon sides=3, fill=white, draw=black, inner sep=0.8pt},
    proj/.style    = {dashed, shorten >=2pt, shorten <=2pt}
  }

  \draw[->] (0.4+\tshift,0) -- (10.5+\tshift,0) node[below] {};

  \node[opendot] at (0.4+\tshift,0) {};
  \node[below=3pt, align=center] at (0.4+\tshift,0) {\tiny $0$};

  \foreach \x/\lab in {
    1/{\tiny $\epa_1$},
    2/{\tiny $\epa_2$},
    3/{\tiny $\epa_3$},
    5/{\tiny $\epa_4$},
    8/{\tiny $\epa_5$},
  } {
    \node[below=2pt, align=center] at (\x+\tshift,0) {\lab};
  }

  \foreach \x/\lab in {
    4/{\tiny $\epb_3$},
    6/{\tiny $\epb_2$},
    7/{\tiny $\epb_4$},
    9/{\tiny $\epb_1$},
    10/{\tiny $\epb_5$},
  } {
    \node[below=4pt, align=center] at (\x+\tshift,0) {\lab};
  }

  \drawinterval{1}{9}{5}
  \drawinterval{2}{6}{4}
  \drawinterval{3}{4}{3}
  \drawinterval{5}{7}{2}
  \drawinterval{8}{10}{1}

  \node[left] at (1-0.05+\tshift,5) {\tiny $I_1$};
  \node[fullsquare] at (2+\tshift,5) {};
  \node[fullsquare] at (3+\tshift,5) {};
  \node[opensquare] at (4+\tshift,5) {};
  \node[fullsquare] at (5+\tshift,5) {};
  \node[opensquare] at (6+\tshift,5) {};
  \node[opensquare] at (7+\tshift,5) {};
  \node[fullsquare] at (8+\tshift,5) {};

  \node[left] at (2-0.05+\tshift,4) {\tiny $I_2$};
  \node[fullsquare] at (3+\tshift,4) {};
  \node[opensquare] at (4+\tshift,4) {};
  \node[fullsquare] at (5+\tshift,4) {};

  \node[left] at (3-0.05+\tshift,3) {\tiny $I_3$};

  \node[left] at (5-0.05+\tshift,2) {\tiny $I_4$};
  \node[opensquare] at (6+\tshift,2) {};

  \node[left] at (8-0.05+\tshift,1) {\tiny $I_5$};
  \node[opensquare] at (9+\tshift,1) {};

  \rowbrace{1}{2}{-1.2}{round 1};
  \rowbrace{2}{3}{-1.2}{round 2};
  \rowbrace{3}{5}{-1.2}{round 3};
  \rowbrace{5}{8}{-1.2}{round 4};
  \rowbrace{8}{10.5}{-1.2}{round 5};
  
\end{tikzpicture}}}
\end{minipage}
\hfill
\raisebox{1.5mm}{\hspace{-1.4cm}
\begin{minipage}[c]{0.385\columnwidth}
    \centering
    \resizebox{\linewidth}{!}{
      \begin{tabular}{|c|c|c|c|c|c|}
      \hline
              & $t=1$ & $t=2$ & $t=3$ & $t=4$ & $t=5$ \\
      \hline
      $d_t$      & $4$ & $2$ & $0$ & $0$ & $0$ \\
      \hline
      $\sigma_t$ & $0$ & $1$ & $2$ & $2$ & $1$ \\
      \hline
      $\ddual_t$   & $3$ & $1$ & $0$ & $1$ & $1$ \\
      \hline
      $\sigdual_t$ & $1$ & $2$ & $2$ & $1$ & $0$ \\
      \hline
      \multicolumn{6}{c}{}\\[-0.25cm]
      \hline
      $\beta(t)$ & $4$ & $2$ & $1$ & $3$ & $5$ \\
      \hline
      \end{tabular}
    }
\end{minipage}}\hspace*{0.5cm}
\end{minipage}}}

\caption{\footnotesize Example of latency-intervals for $T\!=\!5$. The table provides $d_t, \sigma_t, \ddual_t, \sigdual_t$ and $\beta:[T]\to[T]$ from Theorem~\ref{thm:CTM}.}
\label{fig:illustration}
\end{figure}

See Appendix~\ref{app:CTM} for proofs of the results in this section, as well as additional properties of the CTM.

\subsection{Observation-Ordering and Steady Algorithms}\label{subsec:steady-algs}

The CTM naturally suggests an observation-centric perspective, in which we order quantities by observation times rather than by prediction times. This viewpoint complements the standard prediction-centric view and is essential for our analysis of steady algorithms. To facilitate this, we introduce the notational convention for \emph{observation-ordering}, which we adopt throughout the paper.

\begin{definition}[Observation-Ordering]\label{def:observation-ordering}
    Under the CTM, let $\rho\!:\![T]\!\to\![T]$ be the permutation such that 
    $\epb_{\rho(1)}\!<\!...\!<\!\epb_{\rho(T)}$. 
    For any sequence $\{\mathrm{a}_t\}_{t=1}^T$ indexed by $t \in [T]$, define its observation-ordering $\{\widetilde{\mathrm{a}}_n\}_{n=1}^T$ by $\widetilde{\mathrm{a}}_n = \mathrm{a}_{\rho(n)}$ for $n \in [T]$.
\end{definition}

For ease of notation, we use subscripts $n \in [T]$ for observation-ordered quantities and $t \in [T]$ for prediction-ordered quantities throughout the paper. 

Under observation-ordering, the $n$-th observation arrives at time $\wtepb_n$ for the prediction $\wtx_n$ made at time $\wtl_n$ and evaluated on loss function $\wtf_n$. Setting $\wtepb_0\!=\!0$, we obtain a chronological sequence of observation times $0 = \wtepb_0 < \wtepb_1 < \ldots < \wtepb_T$ such that all timestamps lie within $[\wtepb_0, \wtepb_T]$. Since no two timestamps coincide, each prediction time $\wtepa_n$ lies strictly between two consecutive observation times $\{\wtepb_n\}_{n=0}^T$. The following lemma, the observation-centric analogue of Lemma~\ref{lem:CTM-props}, characterizes this interleaving and 
expresses the dual-backlog $\wtsigdual_n$ as the number of outstanding dual-delays.

\begin{lemma}\label{lem:CTM-dual-props}
    For quantities in \eqref{eq:dual-delay-backlog-defn}, it holds that:
    \[
        \wtepa_n\!\in\!(\wtepb_{n-\wtddual_n-1},\wtepb_{n-\wtddual_n}) \quad \text{and} \quad \wtsigdual_n\!=\!|\{m\!:\!m\!-\!\wtddual_m\!\le\!n\!<\!m\}|.
    \]
\end{lemma}
\begin{proof}
    Exactly $|\{m\!:\!\wtepb_m\!<\!\wtepa_n\}|\!=\!|\{m\!:\wtepb_m\!<\!\wtepb_n\}|\!-\!|\{m\!:\!\wtepa_n\!<\!\wtepb_m\!<\wtepb_n\}| = n\!-\!1\!-\!\wtddual_n$ observations precede time $\wtepa_n$, so $\wtepa_n \in (\wtepb_{n-\wtddual_n-1}, \wtepb_{n-\wtddual_n})$, and $\wtsigdual_n = |\{m: \wtepa_m < \wtepb_n < \wtepb_m\}|$ follows.
\end{proof}

\begin{figure}[t]
\centering

\makebox[\columnwidth][c]{\scalebox{1.05}{\begin{minipage}{\columnwidth}
\centering
\begin{minipage}[c]{0.5\columnwidth}
    \centering
    \makebox[\linewidth][l]{\hspace*{0pt}\resizebox{1.05\linewidth}{!}{\newcommand{\xunit}{0.8cm}
\newcommand{\yunit}{0.3cm}
\newcommand{\tshift}{0.5}

\usetikzlibrary{decorations.pathreplacing}
\usetikzlibrary{shapes.geometric}

\newcommand{\drawinterval}[3]{
  \draw[thick] (#1+\tshift,#3) -- (#2+\tshift,#3);

  \node[fulldot]  at (#1+\tshift,#3) {};
  \node[opendot]  at (#2+\tshift,#3) {};

  \draw[proj] (#1+\tshift,#3) -- (#1+\tshift,0);
  \draw[proj] (#2+\tshift,#3) -- (#2+\tshift,0);

  \node[fulldot]  at (#1+\tshift,0) {};
  \node[opendot]  at (#2+\tshift,0) {};
}

\newcommand{\rowbrace}[4]{\draw[decorate,decoration={brace,mirror,amplitude=4pt}]
    (#1+\tshift -0.2 ,#3-0.2) -- (#2+\tshift -0.25,#3-0.2)
    node[midway, yshift=-8pt] {\tiny #4};
}

\begin{tikzpicture}[x=\xunit,y=\yunit]

  \tikzset{
    fulldot/.style = {circle, fill=black, draw=black, inner sep=1.5pt},
    opendot/.style = {circle, fill=white, draw=black, inner sep=1.5pt},
    fullsquare/.style = {regular polygon, regular polygon sides=3, fill=black, draw=black, inner sep=0.8pt, shape border rotate=180},
    opensquare/.style = {regular polygon, regular polygon sides=3, fill=white, draw=black, inner sep=0.8pt},
    proj/.style    = {dashed, shorten >=2pt, shorten <=2pt}
  }

  \draw[->] (0+\tshift,0) -- (10.5+\tshift,0) node[below] {};

  \node[opendot] at (0+\tshift,0) {};
  \node[below=2pt, align=center] at (0.05+\tshift,0) {\tiny $\wtepb_0$};

  \foreach \x/\lab in {
    1/{\tiny $\epa_1$},
    2/{\tiny $\epa_2$},
    3/{\tiny $\epa_3$},
    5/{\tiny $\epa_4$},
    8/{\tiny $\epa_5$},
  } {
    \node[below=2pt, align=center] at (\x+\tshift+0.05,0) {\lab};
  }

  \foreach \x/\lab in {
    4/{\tiny $\wtepb_1$},
    6/{\tiny $\wtepb_2$},
    7/{\tiny $\wtepb_3$},
    9/{\tiny $\wtepb_4$},
    10/{\tiny $\wtepb_5$},
  } {
    \node[below=2pt, align=center] at (\x+\tshift+0.05,0) {\lab};
  }

  \drawinterval{1}{9}{5}
  \drawinterval{2}{6}{4}
  \drawinterval{3}{4}{3}
  \drawinterval{5}{7}{2}
  \drawinterval{8}{10}{1}

  \node[left] at (1-0.05+\tshift,5) {\tiny $I_1$};
  \node[left] at (2-0.05+\tshift,4) {\tiny $I_2$};
  \node[left] at (3-0.05+\tshift,3) {\tiny $I_3$};
  \node[left] at (5-0.05+\tshift,2) {\tiny $I_4$};
  \node[left] at (8-0.05+\tshift,1) {\tiny $I_5$};

  \rowbrace{0.05}{4.05}{-1.5}{$Z_1\!=\!\{1,\!2,\!3\}$};
  \rowbrace{4.05}{6.05}{-1.5}{$Z_2\!=\!\{4\}$};
  \rowbrace{6.05}{7.05}{-1.5}{$Z_3\!=\!\emptyset$};
  \rowbrace{7.05}{9.05}{-1.5}{$Z_4\!=\!\{5\}$};
  \rowbrace{9.05}{10.05}{-1.5}{$Z_5\!=\!\emptyset$};
  
\end{tikzpicture}}}
\end{minipage}
\hfill
\raisebox{1.5mm}{\begin{minipage}[c]{0.385\columnwidth}
    \centering
    \resizebox{\linewidth}{!}{
      \begin{tabular}{|c|c|c|c|c|c|}
      \hline
              & $n=1$ & $n=2$ & $n=3$ & $n=4$ & $n=5$ \\
      \hline
      $\widetilde{d}_n$      & $0$ & $2$ & $0$ & $4$ & $0$ \\
      \hline
      $\widetilde{\sigma}_n$ & $2$ & $1$ & $2$ & $0$ & $1$ \\
      \hline
      $\wtddual_n$           & $0$ & $1$ & $1$ & $3$ & $1$ \\
      \hline
      $\wtsigdual_n$         & $2$ & $2$ & $1$ & $1$ & $0$ \\
      \hline
      \multicolumn{6}{c}{}\\[-0.25cm]
      \hline
      $\rho(n)$ & $3$ & $2$ & $4$ & $1$ & $5$\\
      \hline
      \end{tabular}
    }
\end{minipage}}\hspace*{0.5cm}
\end{minipage}}}

\caption{\footnotesize Partition of rounds induced by the latency intervals in Fig.~\ref{fig:illustration}: if $t\in Z_n$ (equivalently, $l_t\in[\wtepb_{n-1},\wtepb_n)$), then a steady algorithm outputs $x_t=z_n$. The table provides corresponding values $\widetilde{d}_n, \widetilde{\sigma}_n, \wtddual_n, \wtsigdual_n$ and permutation $\rho:[T]\to[T]$.}
\label{fig:blackbox-illustration}
\end{figure}

The observation-centric perspective is particularly useful for analyzing steady algorithms (Definition~\ref{def:steady-alg}). Since a steady algorithm does not change its prediction in rounds without feedback, it admits a \emph{base-prediction sequence} $\{z_n\}_{n=1}^T \subset \cK$ such that the prediction remains constant between consecutive observation times: for each $n\in[T]$ and every prediction time $\epa_t \in (\wtepb_{n-1},\wtepb_n)$, the algorithm outputs $x_t = z_n$.  Equivalently, the rounds $[T]$ partition into sets $Z_n = \{t : l_t \in [\wtepb_{n-1},\wtepb_n)\}$, with $x_t = z_n$ whenever $t \in Z_n$. Figure~\ref{fig:blackbox-illustration} illustrates this for the example in Figure~\ref{fig:illustration}.

Lemma~\ref{lem:CTM-dual-props} links the observation-ordered predictions to the base-predictions through the dual-delays: $\wtx_n = z_{n-\wtddual_n}$. Substituting this identity, we can write the regret of any steady algorithm entirely in terms of the base sequence ${z_n}$ evaluated on the observation-ordered losses ${\wtf_n}$. For $G$-Lipschitz losses, this yields the following decomposition:
\begin{lemma}[Regret decomposition]\label{lem:regret-decomposition}
    Let $f_t:\cK\to\R$ be $G$-Lipschitz for every $t\in[T]$, with no convexity or smoothness required. For any steady algorithm with base-prediction sequence $\{z_n\}_{n=1}^T\subset\cK$ and every comparator $u\in\cK$,
    \begin{align*}
        R_T(u)
        &\le \underbrace{\tsum_{n=1}^T \bigl(\wtf_n(z_n)\!-\!\wtf_n(u)\bigr)}_{\textnormal{Non-Delayed Regret}}
        \;+\;
        G\,\underbrace{\tsum_{n=1}^T \tnorm{z_n\!-\!z_{n\!-\!\wtddual_n}}}_{\textnormal{Prediction Drift}}.
    \end{align*}
\end{lemma}
\begin{proof}
    Since $\rho$ is a permutation and $\wtx_n=z_{n-\wtddual_n}$, reindexing the losses gives
    \begin{align*}
        R_T(u)
        &= \tsum_{n=1}^T \bigl(\wtf_n(z_{n-\wtddual_n})-\wtf_n(u)\bigr) \\
        &= \tsum_{n=1}^T \bigl(\wtf_n(z_n)-\wtf_n(u)\bigr)
        + \tsum_{n=1}^T \bigl(\wtf_n(z_{n-\wtddual_n})-\wtf_n(z_n)\bigr).
    \end{align*}
    By $G$-Lipschitzness, each term in the second sum is at most $G\tnorm{z_{n-\wtddual_n}-z_n}$, proving the bound.
\end{proof}
Here, the \emph{non-delayed regret} is the regret of the sequence $\{z_n\}_{n=1}^T$ on the observation-ordered losses $\{\wtf_n\}_{n=1}^T$ and the \emph{prediction drift} captures the cost of the played predictions $\wtx_n = z_{n-\wtddual_n}$ lagging behind base-predictions $z_n$.
This decomposition underlies our reduction from online learning with delays to online learning with drift-penalization, which we define in the following section.

\section{Online Learning with Drift Penalization}\label{sec:drift-penalization}

This section introduces online linear optimization with drift penalization (Figure~\ref{fig:setting-DP}), a game that serves as the foundation for reductions in subsequent sections. Crucially, this formulation involves no delays; the connection to delayed feedback emerges only through the reductions.

In this game, a player selects predictions from $\cK$ (the same domain as in OCO with delays) over $N$ rounds, incurring in each round a linear loss and a penalty for moving too far from past predictions, measured against one of them. 
Specifically, in round $n$, the player chooses $z_n$ and incurs the loss $\pair{c_n,z_n}$ plus the drift penalty $\tnorm{z_n-z_{n-\lambda_n}}$, specified by the lag $\lambda_n$. The loss vectors and lags $\{c_n, \lambda_n\}_{n=1}^N$ are fixed by an adversary before the game begins. To update his strategy, after each round the player observes $(c_n,\lambda_n)$ as well as the number of outstanding lags $\nu_n = |\{m: m-\lambda_m\!\le n <m\}|$, which summarizes how many future penalties will involve predictions up to round $n$.

\begin{figure}[H]
\centering
\resizebox{\columnwidth}{!}{\fbox{\parbox{\dimexpr1.05\columnwidth-2\fboxsep-2\fboxrule\relax}{
    \normalsize
    \textbf{Online Linear Optimization with Drift Penalization}
    \\
    $\bullet$ \textit{Latent parameters:} number of rounds $N$.\\
    $\bullet$ \textit{Pre-game:} adversary picks loss vectors $c_n\in\R^{\dk}$ and lags $\lambda_n \in \{0, \ldots, n-1\}$ for $n \in [N]$.

    \vspace{2pt}
    \noindent For each round $n=1, 2, \ldots, N$:
    \begin{enumerate}[leftmargin=0.5cm,noitemsep]
        \item The player predicts $z_n \in \cK$, incurring loss $\langle c_n, z_n \rangle$ and drift penalty $\|z_{n-\lambda_n} - z_n\|$.
        \item The environment reveals a tuple $(c_n, \lambda_n, \nu_n)$, where $\nu_n = |\{m: m\!-\!\lambda_m\!\le\!n\!<\!m\}|$ denotes the number of outstanding lags.
    \end{enumerate}
}}
}
\captionsetup{font=footnotesize}
\caption{\small Online learning with drift penalization for linear losses.}\label{fig:setting-DP}
\end{figure}
To measure performance, we define the drift-penalized regret with weight $W \ge 0$ against $u \in \cK$ as
\begin{equation}\label{eq:drift-regret}
    \eR^{\textnormal{drift}}_{N}(u;W) = \eR_N(u) + W \eD_N,
\end{equation}
where $\eR_N(u) = \tsum_{n=1}^N \pair{c_n, z_n\!-\!u}$ is the standard regret and $\eD_N = \tsum_{n=1}^N \tnorm{\sz_{n\!-\!\lambda_n}\!-\!\sz_n}$ is the total drift penalty. 
Any online linear optimization algorithm can be applied to this setting, with its standard regret bounding~$\eR_N(u)$; moreover, stable algorithms naturally incur smaller drift penalty $\eD_N$.

\paragraph{P-FTRL and OMD Updates.}

We analyze how Proximal Follow-The-Regularized-Leader (P-FTRL) and Online Mirror Descent (OMD) perform in online learning with drift penalization. Given a non-increasing learning rate sequence $\{\wteta_n\}_{n=1}^N \subset (0, \infty)$ and initial point $z_1 \in \cK$, these algorithms select predictions sequentially according to the following respective update rules:
\begin{align}
    \textstyle\textnormal{P-FTRL}\colon \sz_{n+1} &=\textstyle\argmin_{\sz\in\cK} \sum_{m\!=\!1}^{n}\!(\pair{c_m,\sz}+\tfrac{\alpha_m \norm{\sz_m\!-\!\sz}^2}{2})\!,\label{eq:P-FTRL}\\
    \textstyle\textnormal{OMD}\colon \sz_{n+1} &=\textstyle \argmin_{\sz\in\cK} (\pair{c_n,\sz} + \tfrac{\norm{\sz_n\!-\!\sz}^2}{2 \wteta_{n}}),\label{eq:OMD}
\end{align}  
where $\{\alpha_n\}_{n=1}^N$ is given by $\alpha_1 = \tfrac{1}{\wteta_1}, \alpha_{n+1} = \tfrac{1}{\wteta_{n+1}}-\tfrac{1}{\wteta_n}$. The next theorem bounds regret and drift.

\begin{theorem}\label{thm:linear-OCO}
    For every $u\in \cK$, predictions $\{\sz_n\}_{n=1}^N$ generated by P-FTRL~\eqref{eq:P-FTRL} or OMD~\eqref{eq:OMD} satisfy
    \begin{align*}
        \eR_N(u) &\le \tsum_{n=1}^{N} \tfrac{{1}/{\wteta_{n}} - {1}/{\wteta_{n-1}}}{2} \norm{\sz_n\!-\!u}^2 + \frac12 \sum_{n=1}^{N} \wteta_n \tnorm{c_n}_{\star}^2\quad \text{and} \quad
        \eD_N \le \tsum_{n=1}^N \wteta_n \nu_n \tnorm{c_n}_{\star}.
    \end{align*}
\end{theorem}

The regret bound is standard (e.g., \citet{orabona_book}). The drift bound arises by decomposing drift into single-step terms satisfying $\tnorm{z_{n+1}-z_n} \le \wteta_n \tnorm{c_n}_{\star}$; the dependence on $\nu_n$ is the key to our delay-adaptive analysis. We also establish a novel drift guarantee for P-FTRL, which is essential for our BCO results.

\begin{lemma}\label{lem:P-FTRL-stability}
    P-FTRL~\eqref{eq:P-FTRL} generates $\{\sz_n\}_{n=1}^N$ such that $\eD_N \le \tsum_{n=1}^N \wteta_{n} \tnorm{\tsum_{m=n-\lambda_n}^{n-1} c_m}_{\star} + D\, H_{\eta}$,
    where $H_{\wteta} = \tsum_{n=1}^N (1 - \frac{\wteta_{n}}{\wteta_{n - \lambda_n}})$ and $D$ is an upper bound on the diameter of $\cK$ (Assumption~\ref{assump:finite-diameter}). 
\end{lemma}

\section{Blackbox Reductions and Main Results}\label{sec:OCO}

Here, we develop algorithms for delayed OCO and BCO via reduction to drift-penalized online linear optimization. The key construction is a wrapper $\cW_{\textnormal{OCO}}$ (Algorithm~\ref{alg:reduction-OCO}) that takes any algorithm $\cB$ for drift-penalized OLO as a blackbox and produces a new algorithm $\cW_{\textnormal{OCO}}(\cB)$ for delayed OCO. The wrapper feeds observation-ordered gradients $\wtg_n = \nabla \wtf_n(\wtx_n)$ as loss vectors to $\cB$, with dual-delays $\wtddual_n$ as lags and dual-backlogs $\wtsigdual_n$ as numbers of outstanding lags. The base algorithm $\cB$ produces base-predictions that the wrapper outputs as its own predictions. An analogous wrapper $\cW_{\textnormal{BCO}}$ (Algorithm~\ref{alg:reduction-BCO}) handles bandit feedback by using single-point gradient estimates from the classical work by \citet{flaxman2004}.

\begin{remark}[Feasibility]
    This reduction is feasible because (i) the dual-backlog $\wtsigdual_n$ equals the number of outstanding dual-delays by Lemma~\ref{lem:CTM-dual-props}, and (ii) the wrapper runs updates on $\cB$ only at observation times $\wtepb_n$, when all required quantities are available. Specifically, at time $\wtepb_n$, the gradient $\wtg_n = \nabla \wtf_n(\wtx_n)$ (or the value $\wtf_n(\wtx_n)$ in the case of bandit feedback) is revealed, while the dual-delay $\wtddual_n = |\{m : \wtepa_n < \wtepb_m < \wtepb_n\}|$ and the dual-backlog $\wtsigdual_n = |\{m : \wtepa_m < \wtepb_n < \wtepb_m\}|$ can be computed from the observed history up to time $\wtepb_n$.
\end{remark}

\subsection{Delayed OCO via Drift-Penalized OLO}\label{sec:reduction-OCO}

We now present the wrapper $\cW_{\textnormal{OCO}}$ for delayed OCO with first-order feedback. The wrapper maintains a current base-prediction $\zcur$ initialized by $\cB$, outputs $\zcur$ whenever a prediction is requested, and updates $\zcur$ by forwarding each received gradient to $\cB$ along with the dual-delay and dual-backlog.

\begin{algorithm}[H]
\caption{$\cW_{\textnormal{OCO}}(\cB)$}\label{alg:reduction-OCO}
\SetKwInOut{Input}{Input}
\SetKwInOut{Access}{Access}
\Access{ Algorithm $\cB$ for OLO with drift penalization.} 
\vspace{0.1cm}

Initialize $\zcur$ with the initial output of $\cB$.

\vspace{0.1cm}
\textbf{For} round $t = 1, 2, \ldots, T$:
\begin{enumerate}[leftmargin=0.5cm,noitemsep]
    \item Predict $x_t = \zcur$.
    \item \textbf{For} each $s$ such that $s+d_s = t$, in observation order:
    \begin{itemize}[leftmargin=0.5cm,before=\vspace{-0.1cm},after=\vspace{-0.1cm},noitemsep]
        \item Receive $(s, \nabla f_s(x_s))$, send $(\nabla f_s(x_s), \ddual_s, \sigdual_s)$ to $\cB$, and set $\zcur$ to the new output of $\cB$.
    \end{itemize}
\end{enumerate}
\end{algorithm}

This algorithm is steady. It updates only upon receiving feedback and outputs the current base-prediction otherwise. Effectively, the algorithm updates $\cB$ sequentially on observation-ordered tuples $(\wtg_n, \wtddual_n, \wtsigdual_n)$, where $\wtg_n = \nabla \wtf_n(\wtx_n)$ (i.e., $g_t = \nabla f_t(x_t)$), producing base-predictions $z_n$ at observation times $\wtepb_{n-1}$ that yield predictions $\wtx_m = z_{m-\wtddual_m}$ at prediction times $\wtepa_m \in (\wtepb_{m-\wtddual_m-1}, \wtepb_{m-\wtddual_m})$. 

The next theorem shows that the regret of $\cW_{\textnormal{OCO}}(\cB)$ is controlled by the drift-penalized regret of $\cB$.

\begin{theorem}[Regret of Algorithm~\ref{alg:reduction-OCO}]\label{thm:OCO-decomp}
    For any base algorithm $\cB$ and $u \in \cK$, $\cW_{\textnormal{OCO}}(\cB)$ guarantees
    \begin{align*}
        R_T(u) \le \eR^{\textnormal{drift}}_{T}(u; G),
    \end{align*}
    where $\eR^{\textnormal{drift}}_{T}(u; W)$ is the drift-penalized regret \eqref{eq:drift-regret} of $\cB$ for loss vectors $c_n = \wtg_n$ and lags $\lambda_n = \wtddual_n$. 
    
    Under $\lambda$-strong convexity (\ref{assump:strong-convexity}), it further holds that $R_T(u) \le \eR^{\textnormal{drift}}_{T}(u; 3G) - \tfrac{\lambda}{2} \tsum_{n=1}^{T} \tnorm{\sz_{n}\!-\!u}^2$.
\end{theorem}

Instantiating $\cB$ with \eqref{eq:P-FTRL} or \eqref{eq:OMD} with appropriate learning rates yields the following guarantees.

\begin{corollary}\label{cor:OCO-regular}
    $\cW_{\textnormal{OCO}}(\cB)$ with $\cB$ as P-FTRL~\eqref{eq:P-FTRL} or OMD~\eqref{eq:OMD} using $\wteta_n = \frac{D/G}{\sqrt{n + \sum_{m=1}^{n} \wtsigdual_m}}$, guarantees
    \begin{align*}
        R_T(x^*) = O( G D\, [\sqrt{\dtot} + \sqrt{T}]).
    \end{align*}
\end{corollary}

\begin{corollary}\label{cor:OCO-sc}
    Under $\lambda$-strong convexity (\ref{assump:strong-convexity}), $\cW_{\textnormal{OCO}}(\cB)$, with $\cB$ as P-FTRL~\eqref{eq:P-FTRL} using $\wteta_n = \tfrac{1}{n\lambda}$, guarantees $R_T(x^*) =  O (\tfrac{G^2}{\lambda}\sbr{\min\cbr{\sigmax\ln T, \, \sqrt{\dtot}} + \ln T})$.
    Running OMD~\eqref{eq:OMD} with $\wteta_n = \frac{1}{n\lambda}$ instead guarantees $R_T(x^*) =  O(\tfrac{G^2}{\lambda} (\sigmax\!+\!1)\ln T)$.
\end{corollary}
\vspace{-0.2cm}
In both corollaries, the learning rates form a deterministic sequence that can be computed online: $\wteta_n$ depends only on $n$ and the dual-backlogs $\wtsigdual_1, \ldots, \wtsigdual_{n}$, all of which are available at the $n$-th update. The proofs apply Theorem~\ref{thm:linear-OCO} and Lemma~\ref{lem:P-FTRL-stability} with our learning rates. Proofs are in Appendix~\ref{app:OCO-BCO-blackbox}.

\subsection{Delayed BCO via Drift-Penalized OLO}\label{sec:BCO}

We adapt the wrapper approach to bandit convex optimization, where only scalar loss values are observed, by using single-point gradient estimation from \citet{flaxman2004}.

Under Assumption~\ref{assump:balls}, i.e., $r\Ball \subseteq \cK \subseteq R\Ball$, we have $(1-\delta/r)\,\cK + \delta\,\Ball \subseteq \cK$ for all $\delta \in (0,r]$. This allows us to define, for any $\delta \in (0,r]$ and integrable $f:\cK\to\R$, the $\delta$-smoothing of $f$ as $f^{\delta}:(1-\delta/r)\cK\to\R$ given by $f^{\delta}(x) = \E_{v\sim \Unif(\Ball)}[f(x+\delta v)]$. The key properties of this construction are summarized in the following theorem.

\begin{theorem}[\citet{flaxman2004}]\label{thm:SPGE}
    For any $\delta\!\in\!(0,\!r]$ and integrable $f:\cK\!\to\!\R$:
    \begin{enumerate}[noitemsep,before=\vspace{-0.3cm},leftmargin=0.5cm,after=\vspace{-0.1cm}]
        \item The $\delta$-smoothing $f^{\delta}$ is differentiable with gradients $\nabla f^{\delta}(x) = \tfrac{\dk}{\delta}\E_{u\sim \Unif(\Sphere)}[f(x+\delta u)\,u]$.
        \item If $f$ is convex ($\lambda$-strongly convex), so is $f^{\delta}$.
        \item If $f$ is $G$-Lipschitz, then $|f^{\delta}(x) - f(x)| \le G\delta$ and $\tnorm{\nabla f^{\delta}(x)} \le G$ for $x \in (1-\delta/r)\cK$.
    \end{enumerate}
\end{theorem}

The wrapper $\cW_{\textnormal{BCO}}$ (Algorithm~\ref{alg:reduction-BCO}) follows the structure of $\cW_{\textnormal{OCO}}$, but replaces true gradients with single-point estimates. The algorithm takes as input a non-increasing sequence of smoothing parameters $(\delta_t)_{t\ge 1} \subset (0,r]$, that can be specified independently of horizon $T$ when it's unknown. In round $t$, given the current base-prediction $y_t = \zcur$ from $\cB$, the algorithm samples an exploration direction $u_t\!\sim\!\Unif(\Sphere)$, predicts $x_t = (1-\delta_t/r)y_t + \delta_t u_t$, and upon receiving loss $f_s(x_s)$ constructs the gradient estimate $\htgdelta_s = \tfrac{\dk}{\delta_s} f_s(x_s) u_s$, which it forwards to $\cB$ along with dual-delay $\ddual_s$ and dual-backlog $\sigdual_s$ to update $\zcur$. By Theorem~\ref{thm:SPGE}, we have $\E[\htgdelta_t \mid y_t] = \nabla f^{\delta_t}_t((1-\delta_t/r) y_t)$.

\begin{algorithm}[h]
\caption{$\cW_{\textnormal{BCO}}(\cB)$}\label{alg:reduction-BCO}
\SetKwInOut{Input}{Input}
\SetKwInOut{Access}{Access}
\Access{ Algorithm $\cB$ for OLO with drift penalization.}
\SetKwInOut{Parameters}{Parameters}
\Parameters{ Non-increasing sequence $(\delta_t)_{t\ge 1} \subset (0,r]$.}

\vspace{0.1cm}
Initialize $\zcur$ with initial output of $\cB$.

\vspace{0.1cm}
\textbf{For} round $t = 1, 2, \ldots, T$:
\begin{enumerate}[leftmargin=0.5cm,noitemsep]
    \item Sample $u_t\!\sim\!\Unif(\Sphere)$; predict $x_t = (1\!-\!\tfrac{\delta_t}{r}) \zcur + \delta_t u_t$. 
    \item \textbf{For} each $s$ such that $s+d_s = t$, in observation order:
    \begin{itemize}[leftmargin=0.5cm,after=\vspace{-0.1cm},noitemsep]
        \item Receive $(s, f_s(x_s))$, send $(\frac{\dk}{\delta_s} f_s(x_s) u_s, \ddual_s, \sigdual_s)$ to $\cB$, and set $\zcur$ to the new output of $\cB$.
    \end{itemize}
\end{enumerate}
\end{algorithm}

Unlike the OCO case, the resulting algorithm is not steady due to random exploration, but is steady in expectation for a fixed schedule $\delta_t = \delta$. Nevertheless, since the current base-prediction is updated only upon receiving feedback, the same analysis applies. The algorithm updates $\cB$ sequentially on tuples $(\wthtgdelta_n, \wtddual_n, \wtsigdual_n)$, where $\wthtgdelta_n = \tfrac{\dk}{\wtdelta_n} \wtf_n(\wtx_n) \wtu_n$, producing base-predictions $z_n$ that yield predictions $\wtx_n = (1-\wtdelta_n/r) \wty_n + \wtdelta_n \wtu_n$ with $\wty_n = z_{n-\wtddual_n}$. 

The following theorem bounds the expected regret of $\cW_{\textnormal{BCO}}(\cB)$ in terms of the drift-penalized regret of $\cB$, with an additional bias of order $\deltatot = \sum_{t=1}^T \delta_t$ due to smoothing.

\begin{theorem}[Regret of Algorithm~\ref{alg:reduction-BCO}]\label{thm:BCO-decomp}
    For any base-algorithm $\cB$, non-increasing sequence $(\delta_t)_{t\ge 1} \subset (0,r]$, and comparator $u \in \cK$, $\cW_{\textnormal{BCO}}(\cB)$ guarantees
    \begin{align*}
       \widebar{R}_T(u) &\le \E\sbr{\eR^{\textnormal{drift}}_{T}(u; G)} + \tfrac{6GR\deltatot}{r},
    \end{align*}
    where $\eR^{\textnormal{drift}}_{T}(u; W)$ is the drift-penalized regret \eqref{eq:drift-regret} of $\cB$ for loss vectors $c_n = \wthtgdelta_n$ and lags $\lambda_n = \wtddual_n$. 
    Moreover, under $\lambda$-strong convexity (\ref{assump:strong-convexity}), $\widebar{R}_T(u) \le \E\!\sbr{\eR^{\textnormal{drift}}_{T}(u; 3G)\!-\!\tfrac{\lambda}{2}\!\tsum_{n=1}^{T}  \tnorm{\sz_{n}\!-\!u}^2}\!+\!\tfrac{10 G R \deltatot}{r}$.
\end{theorem}

Selecting a smoothing sequence $(\delta_t)_{t\ge 1}$ and instantiating $\cB$ with P-FTRL~\eqref{eq:P-FTRL} yields our guarantees. To state them, we introduce the \emph{dimensionless ratio} $\nu = \frac{M}{Gr}$, where $M$ is the absolute value bound (Assumption~\ref{assump:max-val}). Since the $n$-th update occurs in round $\rho(n) + d_{\rho(n)}$, the observation round of feedback from round $\rho(n)$, the smoothing parameter $\delta'_n = \delta_{\rho(n) + d_{\rho(n)}}$ is available and can be incorporated into the rate $\wteta_n$.

\begin{corollary}\label{cor:BCO-regular}
     $\cW_{\textnormal{BCO}}(\cB)$, where $(\delta_t)_{t\ge 1}$ is set to $\delta_t = r\,\min\{ 1, \frac{\sqrt{\nu \dk}}{t^{1/4}}\}$ and $\cB$ runs P-FTRL~\eqref{eq:P-FTRL} with $\wteta_n = \frac{D/G}{\sqrt{n + \sum_{m=1}^n (\wtsigdual_m  + ({\nu \dk r}/{\delta'_m})^2)} }$, guarantees $\widebar{R}_T(x^*) = O\rbr{GD\sbr{\sqrt{\dtot} + T^{3/4} \sqrt{\nu \dk}} }$.
\end{corollary}

\begin{corollary}\label{cor:BCO-sc}
    Under $\lambda$-strong convexity (\ref{assump:strong-convexity}), $\cW_{\textnormal{BCO}}(\cB)$, where $(\delta_t)_{t\ge 1}$ is set to $\delta_t = r\,\min\{1,(\tfrac{\nu^2 {\dk}^2 \ln(et)}{t})^{1/3}\}$ and $\cB$ runs P-FTRL~\eqref{eq:P-FTRL} with $\wteta_n = \tfrac{1}{n\lambda}$, guarantees
    \begin{align*}
    \widebar{R}_T(x^*)
    &= O\!\left(
    \tfrac{G^2}{\lambda}
    \left[
    \min\{\sigmax \ln T,\, \sqrt{\dtot}\}
     + T^{2/3} (\ln T)^{1/3} (\nu {\dk})^{2/3}
    \right]
    \right).
    \end{align*}
\end{corollary}

The proofs of both corollaries combine Theorem~\ref{thm:linear-OCO} and Lemma~\ref{lem:P-FTRL-stability}, using a concentration bound for $\tnorm{\sum_{m=n-\wtddual_n}^{n-1} \wthtgdelta_m}$ to control the drift term in Appendix~\ref{app:OCO-BCO-blackbox}.
These bounds can also be proved directly, without invoking the reduction, by extending the delayed-FTRL analysis of \citet{flaspohler2021} from OCO to single-point bandit gradient estimates and using the drift-control techniques developed in this paper, namely, the FTRL drift bound in Lemma~\ref{lem:P-FTRL-stability} and the bound on sums of gradient estimates in Lemma~\ref{lem:BCO-meandering}. Such a direct analysis is given in \citet[Theorem~3.2]{ryabchenko2026capacity} for delayed and weighted bandit feedback, with the unweighted setting recovered by taking $w_t=1$ for all $t$. The reduction-based proof is stronger, however, since it separates the effect of delayed feedback from the underlying learning algorithm and yields a reusable delayed-to-immediate framework, rather than a proof specialized to a particular bandit update.

\section{Discussion and Future Work}

We introduced a continuous-time model for delayed feedback that yields a black-box reduction to drift-penalized online linear optimization, where any base algorithm with regret and drift guarantees produces a delayed learner whose regret is controlled by its drift-penalized regret under the induced observation ordering. The framework extends beyond OCO to any online learning protocol where prediction drift can be quantified.

Several directions remain open. First, while our first-order guarantees extend to adaptive adversaries, extending our BCO guarantees to settings where losses and delays depend on past interactions requires additional analysis. Such settings are relevant to congestion control and strategic environments. Second, delayed feedback arises naturally in distributed optimization with asynchronous communication, and our model may provide a useful abstraction for multi-agent online learning with heterogeneous, time-varying delays, including federated learning systems. Finally, developing guarantees that adapt to temporal variation in the losses is a promising direction, requiring a refined analysis of both the base learner's regret and prediction drift.

\begin{ack}
Idan Attias is supported by the National Science Foundation under Grant ECCS-2217023, through the Institute for Data, Econometrics, Algorithms, and Learning (IDEAL).
Daniel M. Roy is supported by the funding through NSERC Discovery Grant and Canada CIFAR AI Chair at the Vector Institute.
\end{ack}

\bibliographystyle{plainnat}
\bibliography{assets/refs}

\newpage
\appendix

\section{Additional Related Work}\label{app:related-work}

\paragraph{Online convex optimization.}
Online convex optimization (OCO) and its feedback models have been studied extensively. 
Early work includes \citet{zinkevich2003}, who introduced online convex programming and established $O(\sqrt{T})$-type guarantees for online gradient methods.
For first-order feedback, subsequent work established the standard $\Theta(\sqrt{T})$ minimax scaling for general convex losses (see, e.g., \citealp{abernethy2008}). 
For $\lambda$-strongly convex losses, logarithmic regret is achievable; in particular, \citet{hazan2007} gives $O(\ln T)$-type guarantees (matching the $\Omega(\ln T)$ lower bound of \citealp{abernethy2008}). 
Background and further developments can be found in \citet{hazan2023introductiononlineconvexoptimization, orabona_book}.

\paragraph{Bandit convex optimization.}
Bandit convex optimization (BCO) was initiated by \citet{flaxman2004}, who introduced the classical single-point gradient estimator and obtained regret of order $O(T^{3/4}\dk^{1/2})$.
Subsequently, \citet{agarwal2010} developed optimal algorithms for multi-point bandit feedback and sharpened guarantees: in particular, for two-point feedback they obtain $\widetilde{O}(T^{1/2}\dk^2)$ regret for convex losses and $O(\dk^2\ln T)$ regret under strong convexity; for single-point feedback, they also note improved rates under additional curvature, including a $T^{2/3}$-type guarantee for strongly convex losses (more precisely, $O((T^2 \ln T)^{1/3} \dk^{2/3})$ up to problem-dependent constants).
Complementary results for stochastic convex optimization under bandit feedback appear in \citet{agarwal2011stochasticconvexoptimizationbandit}.
For two-point feedback, \citet{shamir2015} gives an optimal algorithm with improved dimension dependence with regret $O(T^{1/2}\dk^{1/2})$.
A line of work designs (nearly) minimax-optimal algorithms for general BCO, approaching the $\Omega(\dk\sqrt{T})$ lower bound of \citet{shamir2013}; see, e.g., \citet{bubeck2015multiscaleexplorationconvexfunctions, hazan2016optimalalgorithmbanditconvex, lattimore2020improvedregretzerothorderadversarial}.
These methods typically incur prohibitive computational costs (exponential or high-degree polynomial dependence on $\dk$ and/or $T$), and we refer to \citet{lattimore_BCO_book} for a detailed treatment and further references.

\paragraph{Online convex optimization with delayed feedback.}
Early work on online learning with delayed feedback includes \citet{weinberger2002}. 
A systematic treatment of delayed-feedback online learning and reduction-style approaches appears in \citet{joulani2013}, and \citet{joulani2016} develops a black-box-style reduction that processes feedback in order of arrival.
Beyond black-box techniques, there is substantial work on algorithms tailored to delayed OCO (e.g., \citealp{langford2009slowlearnersfast, mcmahan2014, quanrud2015, flaspohler2021, wan2021onlinestronglyconvexoptimization, wu2024, wan2024nonstationary, qiu2025}). 
Delayed BCO has also received focused attention, starting from delayed extensions of bandit gradient methods \citep{heliou2020} and followed by delay-adaptive or improved-delay-dependence results \citep{bistritz2022,wan2024improvedregretbanditconvex}. Concurrently, \citet{esposito2026parameterfree} reduce delayed feedback to online learning with time-varying movement costs in the dynamic regret setting, which is conceptually related to our drift-penalization framework (Section~\ref{sec:drift-penalization}).

\paragraph{Delays in related sequential decision problems.}
Delayed observations have been studied extensively in multi-armed bandits and related models (see, e.g., \citealp{cesa16, thune2019, bistritz2022, esposito2023, masoudian2024, masoudian2024best}). 
Delays also arise naturally under resource limitations such as capacity constraints \citep{capacity_constraint2025}. Delayed feedback has also been studied in adversarial contextual bandits with general function approximation \citep{levy2025regretboundsadversarialcontextual}.
Delays also appear in reinforcement learning / MDP settings with delayed or asynchronous feedback (e.g., \citealp{bouteiller2020reinforcement, lancewicki2022, jin2023, vanderhoeven2023}).

\section{General Facts}

\begin{fact}\label{fact:sc-diameter}
    Under Assumptions~\ref{assump:gradient-norm-bound} and \ref{assump:strong-convexity}, the diameter of $\cK$ is bounded by $\frac{2G}{\lambda}$, i.e., $\sup_{x,y\in\cK} \norm{x-y} \le \frac{2 G}{\lambda}$.
\end{fact}
\begin{proof}
    Consider $f = f_t$ for arbitrary $t\in[T]$. From the $\lambda$-strong convexity of $f$ (e.g., see \citet{orabona_book}), it holds that for all $x,y \in \cK$, $\lambda \norm{x - y}^2 \le \pair{\nabla f(x) - \nabla f(y),\, x - y} \le 2G \norm{x - y}$. Hence, $\sup_{x,y\in\cK} \norm{x-y} \le \frac{2 G}{\lambda} < \infty$.
\end{proof}

\begin{fact}[Rearrangement Inequality]\label{fact:rearrangement}
    For every choice of $N \in \N$, real numbers $x_1 \le x_2 \le ... \le x_N$, $y_1 \le y_2 \le ... \le y_N$, and index permutation $\rho:[N]\to[N]$, it holds that
    \begin{align*}
        x_1 y_N + ... + x_N y_1 \le x_1 y_{\rho(1)} + ... x_N y_{\rho(N)} \le x_1 y_1 + ... + x_N y_N.
    \end{align*}
\end{fact}
\begin{proof}
    This fact restates Theorem 368 in \cite{Hardy1952}.
\end{proof}

\begin{fact}[Logarithmic Telescoping Inequality]\label{fact:log-telescope}
    Let $(X_n)_{n\ge 1}$ be non-negative reals such that $X_1 > 0$.\\ Set $S_n = \sum_{m=1}^n X_m$. Then, for every $N\ge 1$, we have
    \begin{align*}
    \tsum_{n=1}^N \tfrac{X_n}{S_n} \le \ln(e S_N/X_1).
    \end{align*}
\end{fact}

\begin{proof}
    Since $X_1 > 0$, all $S_n > 0$. For $n\ge 2$, since $S_n=S_{n-1}+X_n$,
    \begin{align*}
        \ln S_n - \ln S_{n-1}
        &= \ln\!\Bigl(1+\tfrac{X_n}{S_{n-1}}\Bigr)
        \;\ge\; \frac{\tfrac{X_n}{S_{n-1}}}{1+\frac{X_n}{S_{n-1}}}
        \;=\; \tfrac{X_n}{S_n},
    \end{align*}
    using $\ln(1+u)\ge \frac{u}{1+u}$ for $u\ge 0$. Summing from $n=1$ to $N$ gives
    \begin{align*}
    \tsum_{n=1}^N \tfrac{X_n}{S_n}
    \le 1 + \tsum_{n=2}^N (\ln S_n - \ln S_{n-1})
    = 1 + \ln S_N - \ln S_1 
    = \ln (e S_N / X_1).
    \end{align*}
     This concludes the proof.
\end{proof}

\begin{fact}[Square-Root Telescoping Inequality]\label{fact:sqrt-telescope}
     Let $(X_n)_{n\ge 1}$ be non-negative reals and $X_1 > 0$.\\ Set $S_n = \sum_{m=1}^n X_m$. Then, for every $N\ge 1$, we have
    \begin{align*}
    \tsum_{n=1}^N \tfrac{X_n}{\sqrt{S_n}} \le 2 \sqrt{S_N} - \sqrt{S_1}.
    \end{align*}
\end{fact}

\begin{proof}
    Since $X_1 > 0$, all $S_n > 0$. For $n\ge 2$, since $S_n = S_{n-1}+X_n$, it holds that
    \begin{align*}
    \sqrt{S_n}-\sqrt{S_{n-1}}
    = \frac{S_n-S_{n-1}}{\sqrt{S_n}+\sqrt{S_{n-1}}}
    = \frac{X_n}{\sqrt{S_n}+\sqrt{S_{n-1}}}
    \ge \frac{X_n}{2\sqrt{S_n}}.
    \end{align*}
    Summing over $n=1,\dots,N$, multiplying by $2$, and telescoping,
    \begin{align*}
    \tsum_{n=1}^N \tfrac{X_n}{\sqrt{S_n}}
    \le \sqrt{S_1} + 2\tsum_{n=2}^N \bigl(\sqrt{S_n}-\sqrt{S_{n-1}}\bigr)
    = 2 \sqrt{S_N} - \sqrt{S_1}.
    \end{align*}
     This concludes the proof.
\end{proof}

\section{Technical Results for the Continuous Time Model: Proofs}\label{app:CTM}

This section establishes the technical results concerning the CTM (Definition~\ref{def:CTM}). Recall the observation-ordering convention (Definition~\ref{def:observation-ordering}): $\widetilde{\mathrm{a}}_n = \mathrm{a}_{\rho(n)}$, where $\rho$ is the permutation of $[T]$ satisfying $\epb_{\rho(1)} \le \cdots \le \epb_{\rho(T)}$. Fact~\ref{fact:dual-relation} serves as a central tool throughout this section, connecting delays and backlogs with their duals. Facts~\ref{fact:dual-sums} and~\ref{fact:backlog-increments} are used in the proof of Corollary~\ref{cor:BCO-regular}, while Fact~\ref{fact:telescopic-sums-with-duals} plays a key role in the analysis of delay-adaptive learning rates.

\vspace{0.5cm}

\begin{fact}\label{fact:dual-relation}
    Define $\beta(t) = |\{s : \epb_s \le \epb_t\}|$. For all $t \in [T]$, $\sigdual_t = d_t + t - \beta(t)$ and $\ddual_t = \sigma_t + \beta(t) - t$. Furthermore, $\beta = \rho^{-1}$ as permutations of $[T]$, and for all $n \in [T]$, $\wtsigdual_n = \wtd_n + \rho(n) - n$ and $\wtddual_n = \wtsig_n + n - \rho(n)$.
\end{fact}
\begin{proof}
    From the definition of delay and backlog \eqref{eq:delay-backlog-defn} as induced by the CTM, we have
    \begin{align*}
        t + d_t &= |\{s: \epa_s < \epb_t\}| = |\{s: \epb_s \le \epb_t\}| + |\{s: \epa_s < \epb_t < \epb_s\}| = \beta(t) + \sigdual_t,\\
        t - \sigma_t &= |\{s:\epb_s < \epa_t\}\cup\{t\}| = |\{s: \epb_s \le \epb_t\}| - |\{s: \epa_t < \epb_s < \epb_t \}| = \beta(t) - \ddual_t.
    \end{align*}
    Note that $\beta = \rho^{-1}$ as permutations of $[T]$, since $\beta(\rho(n)) = |\{s : \epb_s \le \epb_{\rho(n)}\}| = n$ for all $n \in [T]$.
    The remaining identities follow by substituting $t = \rho(n)$.
\end{proof}

\vspace{0.2cm}

\restatetheorem{thm:CTM}{
    For quantities in \eqref{eq:delay-backlog-defn} and \eqref{eq:dual-delay-backlog-defn}, it holds that:
    \begin{enumerate}[label=(\alph*),ref=(\alph*),noitemsep,leftmargin=1cm,before=\vspace{-0.3cm},after=\vspace{-0.3cm}]
        \item $\sum_{t=1}^T d_t = \sum_{t=1}^T \sigma_t = \sum_{t=1}^T \ddual_t = \sum_{t=1}^T \sigdual_t$,
        \item $\sigmax = \sigdualmax, \quad \tfrac{1}{2} \dmax \le \ddualmax \le 2\dmax$,
        \item $\ddual_t = \sigma_t + \beta(t) - t, \quad \sigdual_t = d_t + t - \beta(t)$,
    \end{enumerate}
    where $\beta:[T]\to[T]$ is the permutation $\beta(t)=|\{s:\epb_s\le\epb_t\}|$.
}
\begin{proof}
    We address each claim separately.
    \begin{enumerate}[label=(\alph*),ref=(\alph*),before=\vspace{-0.2cm},after=\vspace{-0.2cm}]
        \item To show the classical result $\tsum_{t=1}^T d_t = \tsum_{t=1}^T \sigma_t$  simply swap the order of summation:
        \begin{align*}
            \tsum_{t=1}^T d_t = \tsum_{t=1}^T \tsum_{s=1}^T \ind(l_s \in I_t) =  \tsum_{s=1}^T \tsum_{t=1}^T \ind(l_s \in I_t) = \tsum_{s=1}^T \sigma_s.
        \end{align*}
        The identities $\tsum_{t=1}^T \sigdual_t = \tsum_{t=1}^T d_t$ and $\tsum_{t=1}^T \ddual_t = \tsum_{t=1}^T \sigma_t$ follow immediately from Fact~\ref{fact:dual-relation}, because $\beta$ is a permutation of $[T]$ and $\sum_{t=1}^T \beta(t) = \sum_{t=1}^T t$.
        \item 
        To show that $\sigdualmax \le \sigmax$, note that $r_t \in (l_{t+d_t}, l_{t+d_t+1})$ by Lemma~\ref{lem:CTM-props}, and so 
        \begin{align*}
            \{s: \epb_t \in I_s\} \subseteq (\{t+d_t\} \cup \{s: \epa_{t+d_t} \in I_s\}) \setminus \{t\}.
        \end{align*}
        Hence, $\sigdual_t \le \sigma_{t+d_t} + 1 - 1 = \sigma_{t+d_t}$ for every $t \in [T]$, which means $\sigdualmax \le \sigmax$.
        To show that $\sigdualmax \ge \sigmax$, note that $\wtepa_n \in (\wtepb_{n-\wtddual_n-1}, \wtepb_{n-\wtddual_n})$ by Lemma~\ref{lem:CTM-dual-props}, and so
        \begin{align*}
            \{m: \wtepa_n \in \wtI_m\} \subseteq (\{n - \wtddual_n\} \cup \{m : \wtepb_{n-\wtddual_n} \in \wtI_m\})\setminus\{n\}.
        \end{align*}
        Hence, $\wtsig_n \le \wtsigdual_{n-\wtddual_n} + 1 - 1 = \wtsigdual_{n-\wtddual_n}$ for every $n \in [T]$, which means $\sigmax \le \sigdualmax$.
        To show that $\ddualmax \le 2\dmax$, note that for each $t \in [T]$, $\{s: \epb_s \in I_t\} \subseteq \{s: \epa_s \in I_t\} \cup \{s : \epa_t \in I_s\}$. Hence, $\ddualmax \le \max_t (\sigma_t + d_t) \le \sigmax + \dmax \le 2\dmax$, where the result $\sigmax \le \dmax$ can be inferred from Lemma~\ref{lem:CTM-props}.
        To show that $\dmax \le 2\ddualmax$, note that for each $t \in [T]$, $\{s: \epa_s \in I_t\} \subseteq \{s: \epb_s \in I_t\} \cup \{s : \epb_t \in I_s\}$. Hence, $\dmax \le \max_t (\ddual_t + \sigdual_t) \le \ddualmax + \sigdualmax \le 2\ddualmax$, where the result $\sigdualmax \le \ddualmax$ can be inferred from Lemma~\ref{lem:CTM-dual-props}.
        \item These identities are established in Fact~\ref{fact:dual-relation}. \qedhere
    \end{enumerate}
\end{proof}

\vspace{0.2cm}

\begin{fact}\label{fact:dual-sums}
    For every $N \in [T]$, it holds that
    \begin{align*}
        \tsum_{n=1}^N \wtddual_n \le \tsum_{n=1}^N \wtsigdual_n.
    \end{align*}
\end{fact}
\begin{proof}
    For every $N \in [T]$, we can write
    \begin{align*}
        \tsum_{n=1}^N (\wtsigdual_n - \wtddual_n) 
        &= \tsum_{n=1}^N \tsum_{m=1}^T (\ind(\wtepb_n \in \wtI_m) - \ind(\wtepb_m \in \wtI_n))\\
        &= \tsum_{n=1}^N \tsum_{m=1}^N (\ind(\wtepb_n \in \wtI_m) - \ind(\wtepb_m \in \wtI_n))\\
        &+ \tsum_{n=1}^N \tsum_{m=N+1}^T (\ind(\wtepb_n \in \wtI_m) - \ind(\wtepb_m \in \wtI_n))\\
        &= 0 + \tsum_{n=1}^N \tsum_{m=N+1}^T  \ind(\wtepb_n \in \wtI_m),
    \end{align*}
    where the final equality used symmetry and the fact that $\wtepb_m \notin \wtI_n$ for $m > N \ge n$. 
\end{proof}

\vspace{0.2cm}

\begin{fact}\label{fact:backlog-increments}
    For every $t\in [T-1]$ and $n\in[T-1]$, it holds that
    \begin{align*}
        \sigma_{t+1} \le \sigma_t + 1, \quad \wtsigdual_{n+1} \ge \wtsigdual_{n} - 1.
    \end{align*}
\end{fact}
\begin{proof}
    Using the definition of backlog and dual-backlog, we write
    \begin{align*}
        \sigma_{t+1} &= |\{s: \epa_{t+1} \in I_s\}| = |\{s \le t: \epb_s > \epa_{t+1}\}| \le |\{s \le t: \epb_s > \epa_{t}\}| = \sigma_t + 1,\\
        \wtsigdual_{n} &= |\{m: \wtepb_{n} \in \wtI_m\}| = |\{m > n: \wtepa_m < \wtepb_{n}\}| \le |\{m > n: \wtepa_m < \wtepb_{n+1}\}| = \wtsigdual_{n+1} + 1.\qedhere
    \end{align*}
\end{proof}

\vspace{0.2cm}

\begin{fact}\label{fact:telescopic-sums-with-duals}
    The following two inequalities hold
    \begin{align*}
    \tsum_{n=1}^T \frac{\wtsigdual_n}{n}   \le \sigmax \ln(eT), \quad \tsum_{n=1}^T \frac{\wtddual_n}{n} \le \min\cbr{\sigmax \ln(eT),\, 2\sqrt{\dtot}}.
    \end{align*}
\end{fact}
\begin{proof}
    To begin, observe the following: $\sigdualmax = \sigmax$ by Theorem~\ref{thm:CTM}, $\tsum_{n=1}^T \frac{1}{n} \le \ln(eT)$ by Fact~\ref{fact:log-telescope}, and $\tsum_{n=1}^T \frac{\rho(n)}{n} \ge T$ by Fact~\ref{fact:rearrangement}. We use these observations freely throughout the following proof.
    To prove the first inequality, simply write
    \begin{align*}
        \tsum_{n=1}^T \frac{\wtsigdual_n}{n} \le \max_{n\in[T]} \wtsigdual_n \cdot \tsum_{n=1}^T \frac{1}{n} \le \sigmax \ln(eT).
    \end{align*}
    We split the proof of the second inequality into two parts. 
    \begin{itemize}[before=\vspace{-0.3cm}, after=\vspace{-0.8cm}]
        \item To show that $\tsum_{n=1}^T \frac{\wtddual_n}{n} \le \sigmax \ln(eT)$, use Fact~\ref{fact:dual-relation} to write
        \begin{align*}
            \tsum_{n=1}^T \frac{\wtddual_n}{n} = \tsum_{n=1}^T \frac{\wtsig_{n} + n - \rho(n)}{n} \le \sigmax \ln(eT) + T - \tsum_{n=1}^T \frac{\rho(n)}{n} \le \sigmax \ln(eT).
        \end{align*}
        \item To show that $\tsum_{n=1}^T \frac{\wtddual_n}{n} \le 2\sqrt{\dtot}$, observe that for all $n\in [T]$ it holds that $n \ge \wtddual_n + 1$ as the total number of observations up to and including time $\wtepb_n$ is equal to $n$. Hence, for all $n\in [T]$, $$n^2 \ge \tsum_{m=1}^n m \ge \tsum_{m=1}^n (\wtddual_m + 1) \ge 1 +  \tsum_{m=1}^n \wtddual_m.$$ By Fact~\ref{fact:sqrt-telescope}, we have
        \begin{align*}
            \tsum_{n=1}^T \frac{\wtddual_n}{n} \le \rbr{1 +\tsum_{n=1}^T \frac{\wtddual_n}{\sqrt{1+\sum_{m=1}^n \wtddual_m}}} - 1 \le 2 \sqrt{1+\tsum_{m=1}^T \wtddual_m} - 2 \le  2\sqrt{\dtot}.
        \end{align*}
    \end{itemize}
    This concludes the proof of both inequalities.
\end{proof}

\section{Online Linear Optimization with Drift Penalization: Proofs}\label{app:drift-penalization}

This section proves the results for drift-penalized OLO from Section~\ref{sec:drift-penalization}. Theorem~\ref{thm:linear-OCO} and Lemma~\ref{lem:P-FTRL-stability} establish how P-FTRL and OMD control the standard regret $\eR_N(u) = \tsum_{n=1}^N \pair{c_n, z_n - u}$ and the prediction drift $\eD_N = \tsum_{n=1}^N \tnorm{\sz_{n-\lambda_n}\!-\!\sz_n}$. Given a non-increasing learning rate sequence $\{\wteta_n\}_{n=1}^N \subset (0,\infty)$ and initial point $z_1 \in \cK$, we consider the update rules
\begin{align*}
    \textnormal{P-FTRL}:\,\, \sz_{n+1} &= \textstyle\argmin_{\sz\in\cK} \, \pair{\tsum_{m=1}^{n} c_m,\, \sz}  + \tsum_{m=1}^{n} \tfrac{1/\wteta_{m} - 1/\wteta_{m-1}}{2} \norm{\sz_m - \sz}^2,\\
    \textnormal{OMD}:\,\, \sz_{n+1} &= \textstyle\argmin_{\sz\in\cK} \, \pair{c_n,\, \sz} + \tfrac{1/\wteta_{n}}{2} \norm{\sz_n - \sz}^2,
\end{align*}
with the convention $1/\eta_0 = 0$.

\subsection{General results for FTRL and OMD}

The following results are standard in the analysis of FTRL and OMD.
\begin{theorem}[\citet{orabona_book}, adapted from Lemmas 7.1 and 7.8 for linear losses]\label{thm:FTRL-orabona}
    Let $\cK \subset \R^{\dk}$ be a convex compact set, let $\{\eta_n\}_{n=1}^N \subset (0,\infty)$ be a non-increasing sequence, and let $\psi_1, \ldots, \psi_{N} : \R^{\dk} \to \R$ be differentiable functions such that each $\psi_n$ is $(1/\eta_{n})$-strongly convex with respect to $\tnorm{\cdot}$ on $\cK$. Given loss vectors $\{c_n\}_{n=1}^N \subset \R^{\dk}$ and an initial point $z_1 \in \cK$, we can iteratively select unique minimizers $z_{n+1} = \argmin_{z \in \cK} \cbr{\psi_{n}(z) + \tsum_{m=1}^{n} \pair{c_m, z}}$. For proximal regularizers $\psi_n$ such that $z_n \in \argmin_{z\in\cK}\{\psi_{n}(z)-\psi_{n-1}(z)\}$, for all $u \in \cK$, it holds that
    \begin{align*}
       \tsum_{n=1}^N \pair{c_n, z_n - u} \le \psi_{N}(u) + \tsum_{n=1}^N \rbr{\tfrac{1}{2}\eta_{n}\tnorm{c_n}_{\star}^2 + \psi_{n-1}(z_n) - \psi_{n}(z_n)}.
    \end{align*}
\end{theorem}

\begin{lemma}[Stability lemma, modified version of Lemma A.2 in \citet{qiu2025}]\label{lem:stability-lemma}
    Let $\cK\subset \R^{\dk}$ be a convex compact set. For $i \in \{0, 1\}$, let function $\phi_i:\R^{\dk}\to\R$ be differentiable and $\lambda_i$-strongly convex with respect to $\tnorm{.}$ on $\cK$. Consider $z_i \in \argmin_{x\in\cK} \pair{w_i,\, x}+ \phi_i(x)$. It holds that
    \begin{align*}
        \tfrac{\lambda_0 + \lambda_1}{2} \norm{z_0 - z_1}^2 \le \pair{w_0 - w_1,\, z_1 - z_0} + (\phi_1(z_0) - \phi_0(z_0)) - (\phi_1(z_1) - \phi_0(z_1)).
    \end{align*}
\end{lemma}
\begin{proof}
    For $i \in \{0,1\}$, let $h_i(x) = \pair{w_i,\, x} + \phi_i(x)$ denote the $\lambda_i$-strongly convex function minimized at point $z_i$.
    Therefore, by the strong convexity of $h_i$, it holds that
    \begin{align*}
        h_i(z_{1-i}) - h_i(z_i) \ge \tfrac{\lambda_i}{2}\norm{z_{1-i} - z_{i}}^2.
    \end{align*}
    By summing the above inequality for both $i\in\{0,1\}$, we conclude
    \begin{align*}
        &\pair{w_0 - w_1,\, z_1 - z_0} + (\phi_1(z_0) - \phi_0(z_0)) - (\phi_1(z_1) - \phi_0(z_1))\\
        &= h_0(z_{1}) - h_0(z_0) + h_1(z_{0}) - h_1(z_1)\\
        &\ge \tfrac{\lambda_0 + \lambda_1}{2} \norm{z_0 - z_1}^2. \qedhere
    \end{align*}
\end{proof}

\begin{lemma}[\citet{orabona_book}, adapted from Lemma 6.10 for linear losses]\label{lem:OMD-orabona}
    Let $\cK\subset \R^{\dk}$ be a convex compact set. Consider function $\psi:\R^{\dk}\to\R$ that is proper, closed, differentiable, $1$-strongly convex function with respect to $\tnorm{.}$ in $\cK$, with $B_{\psi}:\R^{\dk}\times\R^{\dk}\to \R$ denoting its Bregman divergence, i.e., $B_{\psi}(x;y) = \psi(x)-\psi(y)-\pair{\nabla \psi(y), x-y}$. Given $\{c_n\}_{n=1}^N\subset \R^{\dk}$, non-increasing sequence $\{\eta_n\}_{n=1}^N \subset (0,\infty)$, and $z_1 \in \cK$, we can iteratively select unique minimizers $z_{n+1} = \argmin_{z \in \cK} \pair{z, c_n} + \frac{1}{\eta_n} B_{\psi}(z; z_n)$. Moreover, for all $u \in \cK$ and $n \in [N]$:
    \begin{align*}
        \eta_n \pair{c_n, z_n - u} \le B_{\psi}(u;z_n) - B_{\psi}(u;z_{n+1}) + \tfrac{\eta_n^2}{2}\tnorm{c_n}_{\star}^2.
    \end{align*}
\end{lemma}

\begin{theorem}[OMD: Regret and Drift]\label{thm:OMD}
    Under the assumptions of Lemma~\ref{lem:OMD-orabona}, for all $u\in\cK$, it holds that
    \begin{align*}
        \tsum_{n=1}^N \pair{c_n, z_n - u} \le \tsum_{n=1}^N (\tfrac{1}{\eta_n} - \tfrac{1}{\eta_{n-1}}) B_{\psi}(u;z_n) + \tfrac{1}{2}\tsum_{n=1}^N \eta_n \tnorm{c_n}_{\star}^2.
    \end{align*}
    Moreover, for all $n \in [N-1]$, it holds that $\tnorm{z_{n+1} - z_n} \le \eta_n \tnorm{c_n}_{\star}$.
\end{theorem}
\begin{proof}
     For all $u \in \cK$, using Lemma~\ref{lem:OMD-orabona}, we write
     \begin{align*}
         \tsum_{n=1}^N \pair{c_n, z_n - u} 
         &\le \tsum_{n=1}^N \rbr{\tfrac{1}{\eta_n} B_{\psi}(u;z_n) - \tfrac{1}{\eta_n} B_{\psi}(u;z_{n+1})} + \tfrac{1}{2} \tsum_{n=1}^N \eta_n \tnorm{c_n}_{\star}^2\\
         &= \tsum_{n=1}^N (\tfrac{1}{\eta_n} - \tfrac{1}{\eta_{n-1}}) B_{\psi}(u;z_n) - \tfrac{1}{\eta_N} B_{\psi}(u;z_{N+1}) + \tfrac{1}{2} \tsum_{n=1}^N \eta_n \tnorm{c_n}_{\star}^2,
     \end{align*}
     where $B_{\psi}(u;z_{N+1}) \ge 0$ for Bregman divergence. This concludes the proof of the regret bound.
     Fix arbitrary $n \in [N]$ and let $\phi_0(z) = \frac{1}{\eta_n} B_{\psi}(z; z_n)$. Let $g_0 = \nabla \phi_0(z_n)$ and $g_1 = \nabla \phi_0(z_{n+1})$. By the first-order optimality of $z_n = \argmin_{z\in\cK} \{\phi_0(z)\}$ and $z_{n+1} = \argmin_{z\in\cK} \{\pair{z, c_n} + \phi_0(z)\}$, we have
    \begin{align*}
        \pair{g_0,\, z_n - z_{n+1}} \le 0 \quad \text{and} \quad \pair{c_n + g_1,\, z_n - z_{n+1}} \ge 0. 
    \end{align*}
    The monotonicity inequality for $\tfrac{1}{\eta_n}$-strong convex  $\phi_0$ yields  $\pair{g_0 - g_1,\, z_n - z_{n+1}} \ge \tfrac{1}{\eta_n} \norm{z_n - z_{n+1}}^2$.\\
    Combining these inequalities, we have
    \begin{align*}
        \tfrac{1}{\eta_n} \norm{z_n - z_{n+1}}^2 &\le \pair{g_0 - g_1,\, z_n - z_{n+1}}\\
        &= \pair{c_n,\, z_n - z_{n+1}} + \pair{g_0,\, z_n - z_{n+1}} - \pair{c_n + g_1,\, z_n - z_{n+1}}\\
        &\le \norm{c_n}_{\star} \norm{z_n - z_{n+1}} + 0 + 0.
    \end{align*}
    Consequently, $\tnorm{z_n - z_{n+1}} \le \eta_n \tnorm{c_n}_{\star}$.
\end{proof}

\subsection{Auxiliary Results for Drift Control of P-FTRL}

\begin{lemma}\label{lem:P-FTRL-stability-core}
    Let $\{\sz_n\}_{n=1}^T$ denote predictions of P-FTRL~\eqref{eq:P-FTRL}. For all $n \in [N]$ and $\lambda \in [N-n]$, it holds that
    \begin{align*}
        \norm{z_{n+\lambda}-z_n} \le  \eta_{n+\lambda} \, \tnorm{\tsum_{m=n}^{n+\lambda-1} c_m}_{\star} + (1 - \eta_{n+\lambda}/\eta_{n})D.
    \end{align*}
    Moreover, it holds that $\norm{z_{n+1}-z_n} \le  \eta_{n} \, \tnorm{c_n}_{\star}$ for $n\in[N-1]$. 
\end{lemma}

\begin{proof}
    Let $\alpha_1 = 1/\eta_1$ and $\alpha_{n+1} = 1/\eta_{n+1} - 1/\eta_n$ for $n \in [N-1]$. 
    To apply Lemma~\ref{lem:stability-lemma}, consider the following vectors and functions
    \begin{align*}
        w_0 = \tsum_{m=1}^{n-1} c_m,&\quad 
        \phi_0(z) = \tsum_{m=1}^{n} \tfrac{\alpha_m}{2}\norm{z-z_m}^2,\\ 
        w_1 = \tsum_{m=1}^{n+\lambda-1} c_m,&\quad 
        \phi_1(z) = \tsum_{m=1}^{n+\lambda} \tfrac{\alpha_m}{2} \norm{z-z_m}^2,
    \end{align*}
    so that $z_n = \argmin_{z\in\cK} \pair{z, w_0}+ \phi_0(z)$ and $z_{n+\lambda} = \argmin_{z\in\cK} \pair{z, w_1}+ \phi_1(z)$, while $\phi_0$ and $\phi_1$ are $1/\eta_{n}$ and $1/\eta_{n+\lambda}$-strongly convex, respectively. 
    Applying Lemma~\ref{lem:stability-lemma}, we have
    \begin{align*}
        &\tfrac{\frac{1}{\eta_{n+\lambda}} + \frac{1}{\eta_{n}}}{2} \norm{z_n\!-\!z_{n+\lambda}}^2\\
        &\le \pair{\tsum_{m=n}^{n+\lambda-1} c_m,\, z_n\!-\!z_{n+\lambda}} + \tsum_{m=n+1}^{n+\lambda} \tfrac{\alpha_m}{2} \rbr{\norm{z_n-z_m}^2 - \norm{z_{n+\lambda}-z_m}^2}\\
        &= \pair{\tsum_{m=n}^{n+\lambda-1} c_m,\, z_n\!-\!z_{n+\lambda}} + \pair{\tsum_{m=n+1}^{n+\lambda} \tfrac{\alpha_m}{2} (z_n\!+\!z_{n+\lambda}\!-\!2 z_m),\, z_n\!-\!z_{n+\lambda}}\\
        &= \pair{\tsum_{m=n}^{n+\lambda-1} c_m,\, z_n\!-\!z_{n+\lambda}} + \pair{\tsum_{m=n+1}^{n+\lambda} \alpha_m (z_{n} - z_m),\, z_n\!-\!z_{n+\lambda}}\\ 
        & -\tfrac{\frac{1}{\eta_{n+\lambda}} - \frac{1}{\eta_{n}}}{2} \norm{z_n\!-\!z_{n+\lambda}}^2.
    \end{align*}
    Grouping terms containing $\norm{z_n - z_{n+\lambda}}^2$ on the left side and using Assumptions~\ref{assump:finite-diameter}, we write
    \begin{align*}
        \tfrac{1}{\eta_{n+\lambda}} \norm{z_n - z_{n+\lambda}}^2
        &\le \pair{\tsum_{m=n}^{n+\lambda-1} c_m,\, z_n - z_{n+\lambda}} + \pair{\tsum_{m=n+1}^{n+\lambda} \alpha_m (z_{n} - z_m),\, z_n - z_{n+\lambda}}\\
        &\le \rbr{\tnorm{\tsum_{m=n}^{n+\lambda-1} c_m}_{\star} + (1/\eta_{n+\lambda} - 1/\eta_{n})D} \, \norm{z_n - z_{n+\lambda}}.
    \end{align*}
    The first inequality immediately follows. Alternatively, for $\lambda = 1$, we have
    \begin{align*}
        \tfrac{1}{\wteta_n}  \norm{z_n - z_{n+1}}^2
        &= (\tfrac{1}{\eta_{n+1}} - \alpha_{n+1}) \norm{z_n - z_{n+1}}^2\\ 
        &\le \pair{c_n,\, z_n - z_{n+1}} + \pair{\alpha_{n+1} (z_{n} - z_{n+1}),\, z_n - z_{n+1}} - \alpha_{n+1} \norm{z_n - z_{n+1}}^2\\
        &= \pair{c_n,\, z_n - z_{n+1}}\\
        &\le \tnorm{c_n}_{\star}\, \tnorm{z_n - z_{n+1}}.
    \end{align*}
    The second result follows from this, concluding the proof. 
\end{proof}

\begin{lemma}\label{lem:H-eta-bound}
    For learning rates $\{\wteta_n\}_{n=1}^N$ of P-FTRL~\eqref{eq:P-FTRL}, let $H_{\eta} = \tsum_{n=1}^N (1 - \tfrac{\wteta_n}{\wteta_{n-\lambda_n}})$.
    Then,
    \begin{align*}
         H_{\eta}
        \le (\textstyle\max_{n\in[N]} \nu_n+1)\ln\rbr{\tfrac{e\wteta_1}{\wteta_N}}
        \qquad\text{and}\qquad
        H_{\eta}
        \le \frac{\max_{n\in[N]} \sum_{m=n}^{n+\nu_n}\wteta_m}{\wteta_N}.
    \end{align*}
\end{lemma}

\begin{proof}
    Let $\alpha_1 = 1/\wteta_1 > 0$ and $\alpha_{n+1} = 1/\wteta_{n+1} - 1/\wteta_n \ge 0$ for $n\in [N-1]$. Then, we can write
    \begin{align*}
        H_{\eta}
        = \tsum_{n=1}^N  \wteta_n \Bigl(\tfrac{1}{\wteta_n} - \tfrac{1}{\wteta_{n-\lambda_n}}\Bigr)
        &= \tsum_{n=1}^N \wteta_n \tsum_{m=1}^{N}\alpha_m\, \ind(n-\lambda_n < m \le n)\\
        &= \tsum_{m=1}^N \alpha_m \tsum_{n=1}^N \wteta_n \ind(n-\lambda_n < m \le n).
    \end{align*}
    Since the sequence $\{\wteta_n\}_{n=1}^N$ is non-increasing and by Lemma~\ref{lem:CTM-dual-props},
    \begin{align*}
        \tsum_{n=1}^N \ind(n-\lambda_n < m \le n) \le \tsum_{n=1}^N \ind(n-\lambda_n \le m < n) + 1 = \nu_m + 1,
    \end{align*}
    we can further write, with $\nu_{\textnormal{max}} = \max_{n \in [N]} \nu_n$:
    \begin{align*}
        H_{\eta}
        &\le \tsum_{m=1}^N \alpha_m \wteta_m \, (\nu_m+1)
        \le (\nu_{\textnormal{max}}+1) \tsum_{m=1}^N \alpha_m \wteta_m\\
        &= (\nu_{\textnormal{max}}+1) \tsum_{m=1}^N \frac{\alpha_m}{\sum_{k=1}^m \alpha_k}
        \le (\nu_{\textnormal{max}}+1) \ln\rbr{\tfrac{e\wteta_1}{\wteta_N}},
    \end{align*}
    where the final inequality follows from Fact~\ref{fact:log-telescope} with $X_n = \alpha_n$ and the identity
    $\wteta_m = 1/(\sum_{k=1}^m \alpha_k)$.
    Using a different approach, we can also write
    \begin{align*}
        H_{\eta}
        &\le \tsum_{m=1}^N \alpha_m \tsum_{n=m}^{m+\nu_m} \wteta_n
        \le \rbr{\tsum_{m=1}^N \alpha_m} \, \max_{n\in [N]} {\tsum_{m=n}^{n+\nu_n} \wteta_m }
        = \tfrac{\max_{n\in [N]} {\sum_{m=n}^{n+\nu_n} \wteta_m } }{\wteta_N}.
    \end{align*}
    This concludes the proof.
\end{proof}

\subsection{Proofs of Theorem~\ref{thm:linear-OCO} and Lemma~\ref{lem:P-FTRL-stability}}

\restatetheorem{thm:linear-OCO}{
    Predictions $\{\sz_n\}_{n=1}^N$ generated by either P-FTRL~\eqref{eq:P-FTRL} or OMD~\eqref{eq:OMD} satisfy
    \begin{align*}
        \eR_N(u) \le \tsum_{n=1}^{N} \tfrac{1/\wteta_{n} - 1/\wteta_{n-1}}{2} \norm{\sz_n - u}^2 + \frac12 \sum_{n=1}^{N} \wteta_n \tnorm{c_n}_{\star}^2 \quad \text{and} \quad \eD_N \le \tsum_{n=1}^N \wteta_n \nu_n \tnorm{c_n}_{\star}.
    \end{align*}
}
\begin{proof}
    The regret bound for P-FTRL follows from Theorem~\ref{thm:FTRL-orabona} with regularizer functions $\psi_n(z) = \tsum_{m=1}^{n} \tfrac{1/\wteta_{m} - 1/\wteta_{m-1}}{2} \norm{\sz_m - \sz}^2$. The regret bound for OMD follows from Theorem~\ref{thm:OMD} with $\psi(z) = \tfrac{1}{2}\tnorm{z}^2$ so that $B_{\psi}(x;z) = \tfrac{1}{2}\tnorm{x-z}^2$.
    For both algorithms, it holds that $\tnorm{z_{n+1}-z_n} \le \eta_n \tnorm{c_n}_{\star}$ for all $n \in [N-1]$, as shown in Theorem~\ref{thm:FTRL-orabona} and Lemma~\ref{lem:P-FTRL-stability-core}. Using triangle inequality and the fact that $\nu_m = |\{n: n-\lambda_n \le m < n\}|$, we can write
    \begin{align*}
        \eD_N 
        &= \tsum_{n=1}^N \tnorm{z_n - z_{n-\lambda_n}}\\
        &\le \tsum_{n=1}^N \tsum_{m=1}^N \ind(n - \lambda_n \le  m < n) \tnorm{z_{m+1}-z_m}\\
        &= \tsum_{m=1}^N  (\tsum_{n=1}^N \ind(n - \lambda_n \le  m < n)) \tnorm{z_{m+1}-z_m} \\
        &= \tsum_{m=1}^N  \nu_m \tnorm{z_{m+1}-z_m}\\
        &\le \tsum_{m=1}^N  \eta_m \nu_m \tnorm{c_m}_{\star},
    \end{align*}
    where the final inequality plugs-in $\tnorm{z_{m+1}-z_m} \le \eta_m \tnorm{c_m}_{\star}$ for single-step drifts.
\end{proof}

\restatelemma{lem:P-FTRL-stability}{
    Predictions $\{\sz_n\}_{n=1}^N$ generated by P-FTRL~\eqref{eq:P-FTRL} satisfy
    \begin{align*}
        \eD_N \le \tsum_{n=1}^N \wteta_{n} \tnorm{\tsum_{m=n-\lambda_n}^{n-1} c_m}_{\star} + D\, H_{\eta} \quad \text{where} \quad H_{\wteta} = \tsum_{n=1}^N (1 - \wteta_{n}/\wteta_{n - \lambda_n}).
    \end{align*}
}

\begin{proof}
    From Lemma~\ref{lem:P-FTRL-stability-core}, we write
    \begin{align*}
        \eD_N 
        &= \tsum_{n=1}^N \tnorm{z_n - z_{n-\lambda_n}}\\
        &\le \tsum_{n=1}^N \rbr{\eta_{n} \, \tnorm{\tsum_{m=n-\lambda_n}^{n-1} c_m}_{\star} + (1 - \eta_{n}/\eta_{n-\lambda_n})D}\\
        &= \tsum_{n=1}^N \eta_{n} \, \tnorm{\tsum_{m=n-\lambda_n}^{n-1} c_m}_{\star} + D H_{\eta}.\qedhere
    \end{align*}
\end{proof}

\section{Guarantees for OCO and BCO with Delays: Proofs}\label{app:OCO-BCO-blackbox}

\subsection{Proof of Theorem~\ref{thm:OCO-decomp}}

\restatetheorem{thm:OCO-decomp}{
    For any base algorithm $\cB$ and comparator $u \in \cK$, $\cW_{\textnormal{OCO}}(\cB)$ guarantees
    \begin{align*}
        R_T(u) \le \eR^{\textnormal{drift}}_{T}(u; G),
    \end{align*}
    where $\eR^{\textnormal{drift}}_{T}(u; W)$ is the drift-penalized regret of $\cB$ with loss vectors $c_n = \wtg_n$ and lags $\lambda_n = \wtddual_n$. 
    
    Under $\lambda$-strong convexity (\ref{assump:strong-convexity}), it further holds that
    \begin{align*}
        R_T(u) \le \eR^{\textnormal{drift}}_{T}(u; 3G) - \tfrac{\lambda}{2} \tsum_{n=1}^{T} \tnorm{\sz_{n}\!-\!u}^2.
    \end{align*}
}
\begin{proof}
    The drift-penalized regret of $\cB$ for these loss vectors and lags is defined as
    \begin{align*}
        \eR^{\textnormal{drift}}_{T}(u; W) = \tsum_{n=1}^T \pair{\wtg_n,\, \sz_{n} - u} + W\,\tsum_{n=1}^T \tnorm{\sz_{n-\wtddual_n} - \sz_n}.
    \end{align*}
    We treat the general convex case as $0$-strongly convex. For $\lambda \ge 0$, the $\lambda$-strong convexity of the loss functions $f_t$ gives
    \begin{align*}
        &\tsum_{t=1}^T (f_t(x_t) - f_t(u))\\
        &\le \tsum_{t=1}^T \pair{g_t,\, x_t - u} - \tfrac{\lambda}{2} \tsum_{t=1}^T \norm{x_t - u}^2\\
        &\myeq{a} \tsum_{n=1}^T \pair{\wtg_n,\, \wtx_n - u} - \tfrac{\lambda}{2} \tsum_{n=1}^T \norm{\wtx_n - u}^2\\
        &\myeq{b} \tsum_{n=1}^T \pair{\wtg_n,\, \sz_{n-\wtddual_n} - u} - \tfrac{\lambda}{2} \tsum_{n=1}^T\tnorm{\sz_{n-\wtddual_n} - u}^2\\
        &\le \tsum_{n=1}^T \pair{\wtg_n,\, \sz_{n} - u} + \tsum_{n=1}^T \tnorm{\wtg_n}_{\star} \tnorm{\sz_{n-\wtddual_n} - \sz_n}  - \tfrac{\lambda}{2} \tsum_{n=1}^T\tnorm{\sz_{n-\wtddual_n} - u}^2\\
        &\myle{c} \tsum_{n=1}^T \pair{\wtg_n,\, \sz_{n} - u} + G\,\tsum_{n=1}^T \tnorm{\sz_{n-\wtddual_n} - \sz_n}  - \tfrac{\lambda}{2} \tsum_{n=1}^T\tnorm{\sz_{n-\wtddual_n} - u}^2,
    \end{align*}
    where (a) replaces sequences $\{x_t\}_{t=1}^T$, $\{g_t\}_{t=1}^T$ with their observation-orderings (Definition~\ref{def:observation-ordering}), (b) substitutes $\wtx_n = z_{n-\wtddual_n}$ which follows from Lemma~\ref{lem:CTM-dual-props}, and (c) applies the gradient norm bound (Assumption~\ref{assump:gradient-norm-bound}).
    For the general convex case ($\lambda = 0$), this immediately gives $R_T(u) \le \eR^{\textnormal{drift}}_{T}(u; G)$.
    For the strongly convex case ($\lambda > 0$), we convert the distance terms from $\sz_{n-\wtddual_n}$ to $\sz_n$. Since $\tnorm{\sz_{n-\wtddual_n} - u}, \tnorm{\sz_n - u} \le \frac{2G}{\lambda}$ by Fact~\ref{fact:sc-diameter}, we have
    \begin{align*}
        \tfrac{\lambda}{2}(\tnorm{\sz_{n} - u}^2 - \tnorm{\sz_{n-\wtddual_n} - u}^2) = \tfrac{\lambda}{2}\,\pair{\sz_{n-\wtddual_n} + \sz_n - 2u,\, \sz_n - \sz_{n-\wtddual_n}} \le 2G\, \tnorm{\sz_{n-\wtddual_n} - \sz_n}. 
    \end{align*}
    Summing over $n$ and rearranging yields $R_T(u) \le \eR^{\textnormal{drift}}_{T}(u; 3G) - \tfrac{\lambda}{2} \tsum_{n=1}^{T} \tnorm{\sz_{n} - u}^2$.
\end{proof}

\subsection{Proofs of Corollaries~\ref{cor:OCO-regular} and \ref{cor:OCO-sc}}

\begin{theorem}\label{thm:OCO-decomp-FTRL}
    $\cW_{\textnormal{OCO}}(\cB)$, where $\cB$ executes P-FTRL \eqref{eq:P-FTRL} updates, guarantees that in the 
    \begin{align*}
    \text{general convex case} \colon& \, R_T(u) \le \tfrac{D^2}{2\wteta_T} + G^2\,\tsum_{n=1}^{T} \wteta_n\, (\wtsigdual_n+1),\\
    \text{strongly convex case (\ref{assump:strong-convexity})} \colon& \, R_T(u) \le \tsum_{n=1}^T \tfrac{1/\wteta_n - 1/\wteta_{n-1} - \lambda}{2} \norm{\sz_n\!-\!u}^2\\
    &\qquad\qquad + 3G^2\tsum_{n=1}^{T} \wteta_n (\wtddual_n\!+\!1) + 3GD H_{\eta},
    \end{align*}
    where $H_{\eta} = \tsum_{n=1}^T (1 -  \wteta_n/{\wteta_{n-\wtddual_n}})$ denotes the total learning rate misalignment.
\end{theorem}

\begin{proof}
    Both results follow from Theorem~\ref{thm:OCO-decomp} once we bound the standard regret $\eR_T(u) = \tsum_{n=1}^T \pair{\wtg_n,\, \sz_{n} - u}$ and prediction drift $\eD_T = \tsum_{n=1}^T \tnorm{\sz_{n-\wtddual_n} - \sz_n}$ components.
    To control $\eR_T(u)$ in both cases, apply Theorem~\ref{thm:linear-OCO} for $\tnorm{\wtg_n}_{\star} \le G$ and $\tnorm{z_n - u} \le D$, as follows
    \begin{align*}
        \eR_T(u) 
        &\le \tsum_{n=1}^{T} \tfrac{1/\wteta_n - 1/\wteta_{n-1}}{2} \norm{\sz_n - u}^2 + \frac12 \sum_{n=1}^{T} \wteta_n \tnorm{g_n}_{\star}^2 \le \tfrac{D^2}{2\wteta_T} + G^2 \tsum_{n=1}^T \wteta_n.
    \end{align*}
    To control $\eD_T$ in the general convex case, use Theorem~\ref{thm:linear-OCO} to write
    \begin{align*}
        \eD_T 
        &\le \tsum_{n=1}^T \wteta_n \wtsigdual_n \tnorm{\wtg_n}_{\star} \le G \tsum_{n=1}^T \wteta_n \wtsigdual_n.
    \end{align*}
    Then, the general convex case follows from the corresponding inequality in Theorem~\ref{thm:OCO-decomp}:
    \begin{align*}
        R_T(u) 
        &\le \eR_T(u) + G\,\eD_T
        \le \tfrac{D^2}{2\wteta_T} + G^2 \tsum_{n=1}^T \wteta_n + G^2 \tsum_{n=1}^{T} \wteta_n\, \wtsigdual_n
        \\
        &\le \tfrac{D^2}{2\wteta_T} + G^2\,\tsum_{n=1}^{T} \wteta_n\, (\wtsigdual_n+1).
    \end{align*}
    To control $\eD_T$ in the strongly convex case, use Lemma~\ref{lem:P-FTRL-stability} to write
    \begin{align*}
        \eD_T \le \tsum_{n=1}^T \wteta_{n} \tnorm{\tsum_{m=n-\wtddual_n}^{n-1} \wtg_m}_{\star} + D\, H_{\eta} \le G\,\tsum_{n=1}^T \wteta_{n} \wtddual_n  + D\, H_{\eta},
    \end{align*}
    where $H_{\eta} = \tsum_{n=1}^T (1 - \wteta_n/{\wteta_{n-\wtddual_n}})$ in this case with lags $\lambda_n = \wtddual_n$.
    Finally, the strongly convex case follows from the corresponding inequality in Theorem~\ref{thm:OCO-decomp}:
    \begin{align*}
        R_T(u) 
        &\le \eR_T(u) - \tfrac{\lambda}{2} \tsum_{n=1}^{T} \tnorm{\sz_{n} - u}^2 + 3G\,\eD_T\\
        &\le \tsum_{n=1}^T \tfrac{1/\wteta_n - 1/\wteta_{n-1} - \lambda}{2} \norm{\sz_n\!-\!u}^2 + G^2\tsum_{n=1}^{T} \wteta_n + 3G^2\tsum_{n=1}^{T} \wteta_n \wtddual_n + 3GD H_{\eta}\\
        &\le \tsum_{n=1}^T \tfrac{1/\wteta_n - 1/\wteta_{n-1} - \lambda}{2} \norm{\sz_n\!-\!u}^2 + 3G^2\tsum_{n=1}^{T} \wteta_n (\wtddual_n\!+\!1) + 3GD H_{\eta}.
    \end{align*}
    This concludes the proof for both cases.
\end{proof}

\begin{theorem}\label{thm:OCO-decomp-OMD}
    $\cW_{\textnormal{OCO}}(\cB)$, where $\cB$ executes OMD \eqref{eq:OMD} updates, guarantees that in the 
    \begin{align*}
    \text{general convex case} \colon& \,\, R_T(u) \le \tfrac{D^2}{2\wteta_T} + G^2\,\tsum_{n=1}^{T} \wteta_n\, (\wtsigdual_n+1),\\
    \text{strongly convex case (\ref{assump:strong-convexity})} \colon& \,\, R_T(u) \le \tsum_{n=1}^T \tfrac{1/\wteta_n - 1/\wteta_{n-1} - \lambda}{2} \norm{\sz_n - u}^2 + 3G^2\tsum_{n=1}^{T} \wteta_n (\wtsigdual_n+1).
    \end{align*}
\end{theorem}
\begin{proof}
    As in the proof of Theorem~\ref{thm:OCO-decomp-FTRL}, we use Theorem~\ref{thm:linear-OCO} to show that
    \begin{align*}
        \eR_T(u) 
        \le \tsum_{n=1}^{T} \tfrac{1/\wteta_n - 1/\wteta_{n-1}}{2} \norm{\sz_n - u}^2 + G^2 \tsum_{n=1}^T \wteta_n
        \quad \text{and} \quad
        \eD_T \le G\,\tsum_{n=1}^{T} \wteta_n\, \wtsigdual_n.
    \end{align*}
    The general convex case similarly follows from its corresponding inequality in Theorem~\ref{thm:OCO-decomp}:
    \begin{align*}
        R_T(u) 
        &\le \eR_T(u) + G\,\eD_T
        \le \tfrac{D^2}{2\wteta_T} + G^2 \tsum_{n=1}^T \wteta_n + G^2 \tsum_{n=1}^{T} \wteta_n\, \wtsigdual_n
        \\
        &\le \tfrac{D^2}{2\wteta_T} + G^2\,\tsum_{n=1}^{T} \wteta_n\, (\wtsigdual_n+1).
    \end{align*}
    The strongly convex case also follows from its corresponding inequality in Theorem~\ref{thm:OCO-decomp}:
    \begin{align*}
        R_T(u) 
        &\le \eR_T(u) - \tfrac{\lambda}{2} \tsum_{n=1}^{T} \tnorm{\sz_{n} - u}^2 + 3G\,\eD_T\\
        &\le \tsum_{n=1}^T \tfrac{1/\wteta_n - 1/\wteta_{n-1} - \lambda}{2} \norm{\sz_n\!-\!u}^2 + G^2\tsum_{n=1}^{T} \wteta_n + 3G^2\tsum_{n=1}^{T} \wteta_n \wtsigdual_n\\
        &\le \tsum_{n=1}^T \tfrac{1/\wteta_n - 1/\wteta_{n-1} - \lambda}{2} \norm{\sz_n\!-\!u}^2 + 3G^2\tsum_{n=1}^{T} \wteta_n (\wtsigdual_n + 1).
    \end{align*}
    This concludes the proof for both cases.
\end{proof}

\restatecorollary{cor:OCO-regular}{
    $\cW_{\textnormal{OCO}}(\cB)$, where $\cB$ runs P-FTRL~\eqref{eq:P-FTRL} or OMD~\eqref{eq:OMD} with $\wteta_n = \frac{D/G}{\sqrt{n + \sum_{m=1}^{n} \wtsigdual_m}}$, guarantees
    \begin{align*}
        R_T(x^*) = O\rbr{ G D\, \sbr{\sqrt{\dtot} + \sqrt{T}}}.
    \end{align*}
}
\begin{proof}
    Using the general convex case of Theorem~\ref{thm:OCO-decomp-FTRL} for P-FTRL and Theorem~\ref{thm:OCO-decomp-OMD} for OMD, we can write 
    \begin{align*}
        R_T(x^*) 
        &\le \tfrac{D^2}{2 \wteta_{T}} +  G^2 \tsum_{n=1}^T \wteta_n (\wtsigdual_n+1)\\
        &\myle{a} GD\, \sqrt{T + \tsum_{n=1}^T \wtsigdual_n} + 2 GD\, \sqrt{\tsum_{n=1}^T (\wtsigdual_n+1)}\\
        &\myeq{b} 3 GD\, \sqrt{T + \dtot}\\
        &\le 6 GD \sbr{\sqrt{\dtot} +\sqrt{T}},
    \end{align*}
    where (a) substitutes learning rate values and applies Fact~\ref{fact:sqrt-telescope} for $X_n = \wtsigdual_n + 1$ and (b) uses the identity $\tsum_{t=1}^T \sigdual_t = \dtot$ from Theorem~\ref{thm:CTM}.
\end{proof}

\restatecorollary{cor:OCO-sc}{
    Under $\lambda$-strong convexity (\ref{assump:strong-convexity}), $\cW_{\textnormal{OCO}}(\cB)$, where $\cB$ runs P-FTRL~\eqref{eq:P-FTRL} with $\wteta_n = \tfrac{1}{n\lambda}$, guarantees
    \begin{align*}
         R_T(x^*) =  O \rbr{\tfrac{G^2}{\lambda}\sbr{\min\cbr{\sigmax\ln T, \, \sqrt{\dtot}} + \ln T} }.
    \end{align*}
    Running OMD~\eqref{eq:OMD} with $\wteta_n = \frac{1}{n\lambda}$ instead guarantees
    \begin{align*}
         R_T(x^*) =  O\rbr{\tfrac{G^2}{\lambda} (\sigmax\!+\!1)\ln T}.
    \end{align*}
}

\begin{proof}
    We prove the result for P-FTRL first. For learning rates $\eta_n = \tfrac{1}{n\lambda}$, by letting $D \le \frac{2G}{\lambda}$ (see Fact~\ref{fact:sc-diameter}), term $H_{\eta} = \tsum_{n=1}^T \wteta_n (1/\wteta_n - 1/{\wteta_{n-\wtddual_n}})$ can be bounded from above as
    \begin{align*}
        H_{\eta} = \tsum_{n=1}^T \wteta_n (\lambda n - \lambda (n-\wtddual_n)) = \lambda \tsum_{n=1}^T \wteta_n \wtddual_n \le (2G/D) \tsum_{n=1}^T \wteta_n \wtddual_n.
    \end{align*}
    Then, using the strongly convex case of Theorem~\ref{thm:OCO-decomp-FTRL}, we write 
    \begin{align*}
        R_T(x^*) 
        &\le \tsum_{n=1}^T \tfrac{1/\wteta_n - 1/\wteta_{n-1} - \lambda}{2} \norm{\sz_n\!-\!x^*}^2 + 3G^2\tsum_{n=1}^{T} \wteta_n (\wtddual_n\!+\!1) + 3GD H_{\eta}\\
        &\myle{a} 0 + 3G^2\tsum_{n=1}^{T} \wteta_n (\wtddual_n\!+\!1) + 6G^2 \tsum_{n=1}^T \wteta_n \wtddual_n\\
        &\myle{b} 9G^2 \tsum_{n=1}^T \tfrac{\wtddual_n + 1}{\lambda n}\\
        &\myle{c} \tfrac{9G^2}{\lambda} \rbr{\min\{\sigmax \ln(eT),\, 2\sqrt{\dtot}\} + \ln(eT)},
    \end{align*}
    where (a) substitutes learning rate values in the first term and applies the above bound on $H_{\eta}$, (b)  substitutes $\wteta_n = 1/(\lambda n)$, and (c) uses Fact~\ref{fact:telescopic-sums-with-duals} for $\wtddual_n$ and Fact~\ref{fact:log-telescope} to get the desired bound.
    For OMD, using the strongly convex case of Theorem~\ref{thm:OCO-decomp-OMD}, we write 
    \begin{align*}
        R_T(x^*) 
        &\le \tsum_{n=1}^T \tfrac{1/\wteta_n - 1/\wteta_{n-1} - \lambda}{2} \norm{\sz_n\!-\!x^*}^2 + 3G^2\tsum_{n=1}^{T} \wteta_n (\wtsigdual_n+1)\\
        &\myle{a} 0 + 3G^2 \tsum_{n=1}^T \tfrac{\wtsigdual_n + 1}{\lambda n}\\
        &\myle{b} \tfrac{3G^2}{\lambda} \rbr{\sigmax \ln(eT) + \ln(eT)},
    \end{align*}
    where (a) substitutes learning rate values and (b) uses Fact~\ref{fact:telescopic-sums-with-duals} for $\wtsigdual_n$ to get the final bound.
\end{proof}

\subsection{Proof of Theorem~\ref{thm:BCO-decomp}}

In this section, we prove Theorem~\ref{thm:BCO-decomp}. We begin by recalling notation from the main body.

Algorithm $\cB$ generates base-predictions $\{z_n\}_{n=1}^T$, and in round $t$ the wrapper $\cW_{BCO}(\cB)$ produces prediction $x_t = (1 - \delta_t/r)\, y_t + \delta_t u_t$, where $y_t = \zcur$ is the current base-prediction and $u_t \sim \Unif(\Sphere)$ is sampled independently. Define the shrunk suggestions $\ydelta_t = (1 - \delta_t/r)\, y_t \in (1 - \delta_t/r)\cK$, so that $x_t = \ydelta_t + \delta_t u_t$.

Under observation-ordering (Definition~\ref{def:observation-ordering}), these quantities become
\begin{align*}
    \wty_n = \sz_{n-\wtddual_n}, \quad
    \wtydelta_n = (1 - \tfrac{\wtdelta_n}{r})\, \sz_{n-\wtddual_n}, \quad
    \wtx_n = \wtydelta_n + \wtdelta_n \wtu_n.
\end{align*}
For the smoothed losses $\wtf_n^{\delta}(x) = \E_{v\sim \Unif(\Ball)}[\wtf_n(x + \wtdelta_n v)]$ defined on $(1 - \tfrac{\wtdelta_n}{r})\cK$, let $\wtgdelta_n = \nabla \wtf^{\delta}_n(\wtydelta_n)$.\\ By Theorem~\ref{thm:SPGE}, the single-point estimator $\wthtgdelta_n = \tfrac{\dk}{\wtdelta_n} \wtf_n(\wtx_n) \wtu_n$ satisfies $\E[\wthtgdelta_n \mid \sz_{n-\wtddual_n}] = \wtgdelta_n$.

\restatetheorem{thm:BCO-decomp}{
    For any base-algorithm $\cB$, non-increasing sequence $(\delta_t)_{t\ge 1} \subset (0,r]$, and comparator $u \in \cK$, $\cW_{\textnormal{BCO}}(\cB)$ guarantees
    \begin{align*}
       \widebar{R}_T(u) &\le \E\sbr{\eR^{\textnormal{drift}}_{T}(u; G)} + \tfrac{6GR\deltatot}{r},
    \end{align*}
    where $\eR^{\textnormal{drift}}_{T}(u; W)$ is the drift-penalized regret of $\cB$ on loss vectors $c_n = \wthtgdelta_n$ with lags $\lambda_n = \wtddual_n$, and $\deltatot = \tsum_{t=1}^T \delta_t$.

    Under $\lambda$-strong convexity (\ref{assump:strong-convexity}), it further holds that
    \begin{align*}
        \widebar{R}_T(u) &\le \E\!\sbr{\eR^{\textnormal{drift}}_{T}(u; 3G)\!-\!\tfrac{\lambda}{2}\!\tsum_{n=1}^{T}  \tnorm{\sz_{n}\!-\!u}^2}\!+\!\tfrac{10 G R \deltatot}{r}.
    \end{align*}
}

\begin{proof}
    Fix comparator $v \in \cK$. For $n\in [T]$, let $\wtvdelta_n = (1-\wtdelta_n / r) v$ which satisfies $\tnorm{\wtvdelta_n - v} = \tfrac{\wtdelta_n \norm{v}}{r}  \le \tfrac{R \wtdelta_n}{r}$. This gives
    \begin{align}
        \widebar{R}_T(v)
        &= \E\sbr{\tsum_{n=1}^{T} (\wtf_n(\wtx_n) - \wtf_n(v))}\notag\\
        &= \E\sbr{\tsum_{n=1}^{T} (\wtf^{\delta}_n(\wtydelta_n) - \wtf^{\delta}_n(\wtvdelta_n))} \notag\\
        &+ \E\sbr{\tsum_{n=1}^{T} (\wtf_n(\wtx_n) - \wtf_n(\wtydelta_n))} + \E\sbr{\tsum_{n=1}^{T} (\wtf_n(\wtydelta_n) - \wtfdelta_n(\wtydelta_n))}\notag\\
        &+ \E\sbr{\tsum_{n=1}^{T} (\wtfdelta_n(\wtvdelta_n) - \wtf_n(\wtvdelta_n))} + \E\sbr{\tsum_{n=1}^{T} (\wtf_n(\wtvdelta_n) - \wtf_n(v))}\notag\\
        &\myle{a} \E\sbr{\tsum_{n=1}^{T} (\wtfdelta_n(\wtydelta_n) - \wtfdelta_n(\wtvdelta_n))} + \tsum_{n=1}^T \rbr{G\wtdelta_n + G\wtdelta_n  + G\wtdelta_n  + \tfrac{GR}{r}\wtdelta_n}\notag\\
        &\myle{b} \E\sbr{\tsum_{n=1}^{T} \rbr{\pair{\wtgdelta_{n},\, \wtydelta_n - \wtvdelta_n} - \tfrac{\lambda}{2} \tnorm{\wtydelta_n - \wtvdelta_n}^2}} + \tfrac{4GR \deltatot}{r}\notag\\
        &\myeq{c} \E\sbr{\tsum_{n=1}^{T} \rbr{(1-\tfrac{\wtdelta_n}{r})\,\pair{\wtgdelta_{n},\, \sz_{n-\wtddual_n} - v} - \tfrac{\lambda}{2} (1-\tfrac{\wtdelta_n}{r})^2\, \tnorm{\sz_{n-\wtddual_n} - v}^2}} + \tfrac{4GR \deltatot}{r}\notag\\
        &= \E\sbr{\tsum_{n=1}^{T} \rbr{\pair{\wtgdelta_{n},\, \sz_{n} - v} - \tfrac{\lambda}{2} \tnorm{\sz_{n-\wtddual_n} - v}^2 + \pair{\wtgdelta_{n},\, \sz_{n-\wtddual_n} - \sz_n}}}  + \tfrac{4GR\deltatot}{r} \notag\\
        &+ \E\sbr{\tsum_{n=1}^T (\tfrac{\wtdelta_n}{r} \pair{\wtgdelta_{n},\, v - \sz_{n-\wtddual_n}} + (\tfrac{\lambda \wtdelta_n}{r} - \tfrac{\lambda \wtdelta_n^2}{2r^2})\,  \tnorm{\sz_{n-\wtddual_n} - v}^2 )}.\label{eq:BCO-proof-main}
    \end{align}
    where (a) uses $G$-Lipschitzness of $\wtf_n$ and Theorem~\ref{thm:SPGE} to bound the last four terms in the previous line, (b) uses $\lambda$-strong convexity of $\wtf^{\delta}_n$, and (c) substitutes $\wtydelta_{n} = (1-\tfrac{\wtdelta_n}{r}) \sz_{n-\wtddual_n}$ and $\wtvdelta_n = (1-\tfrac{\wtdelta_n}{r}) v$.
    The final term in \eqref{eq:BCO-proof-main} can be bounded using Cauchy-Schwarz inequality as follows:
    \begin{align}
        &\E\sbr{\tsum_{n=1}^T (\tfrac{\wtdelta_n}{r} \pair{\wtgdelta_{n},\, v- \sz_{n-\wtddual_n}} + (\tfrac{\lambda \wtdelta_n}{r} - \tfrac{\lambda \wtdelta_n^2}{2r^2})\,  \tnorm{\sz_{n-\wtddual_n} - v}^2 )}\notag\\
        &\quad \quad \quad \quad \quad \le \E\sbr{\tsum_{n=1}^T (\tfrac{\wtdelta_n}{r} \tnorm{\wtgdelta_{n}}_{\star} \tnorm{\sz_{n-\wtddual_n}-v} + \tfrac{\lambda \wtdelta_n}{r}\,  \tnorm{\sz_{n-\wtddual_n} - v}^2 )}\notag\\
        &\quad \quad \quad \quad \quad \myle{d} \tsum_{n=1}^T \rbr{\tfrac{GD\wtdelta_n}{r} + \tfrac{\lambda D^2\wtdelta_n}{r} } = \tfrac{(G+\lambda D) D\deltatot}{r},\label{eq:BCO-proof-last-term}
    \end{align}
    where (d) follows from the facts that $\tnorm{\sz_{n-\wtddual_n} - v} \le D$ and $ \tnorm{\wtgdelta_{n}}_{\star} \le G$ due to Theorem~\ref{thm:SPGE}.
    For every $n \in [T]$, the direction $\wtu_n$ and base-prediction $\sz_n$ are independent: $\wtu_n$ is sampled when making the $n$-th prediction and only affects the $n$-th observation $\wtf_n(\wtx_n) = \wtf_n((1-\wtdelta_n/r)\, \sz_{n-\wtddual_n} + \wtdelta_n \wtu_n)$, whereas $\sz_n$ is the output of $\cB$ after processing the $(n-1)$-th observation. By Theorem~\ref{thm:SPGE},
    \begin{align*}
        \E[\wthtgdelta_n \mid \sz_n] = \E\sbr{\tfrac{\dk}{\wtdelta_n}\wtf_n(\wtx_n)\wtu_n \mid \sz_n} = \nabla \wtfdelta_n\rbr{(1-\tfrac{\wtdelta_n}{r})\, \sz_{n-\wtddual_n}} = \wtgdelta_n.
    \end{align*}
    Hence, the expected regret term $\E[\eR(u)]$ satisfies
    \begin{align}  
        \E\sbr{\eR(u)} = \E\sbr{\tsum_{n=1}^{T} \pair{\wthtgdelta_{n},\, \sz_{n} - u}} \myeq{e} \E\sbr{\tsum_{n=1}^{T} \pair{\wtgdelta_{n},\, \sz_{n} - u}}.\label{eq:BCO-proof-aux-regret}
    \end{align}
    Combining \eqref{eq:BCO-proof-main}, \eqref{eq:BCO-proof-last-term}, \eqref{eq:BCO-proof-aux-regret} and the facts that $D \le 2R$ and $\pair{\wtgdelta_{n},\, \sz_{n-\wtddual_n} - \sz_n} \le G \tnorm{\sz_{n-\wtddual_n} - \sz_n}$, we have
    \begin{align}
        \widebar{R}_T(v)\le \E\sbr{\eR_T(v) - \tfrac{\lambda}{2} \tsum_{n=1}^{T}  \tnorm{\sz_{n-\wtddual_n} - v}^2 + G \eD_T}  + \tfrac{(6G + 2\lambda D) R \deltatot}{r}.\label{eq:BCO-proof-prefinal}
    \end{align}
    This concludes the proof for the convex case ($\lambda = 0$). It remains to consider the strongly convex case ($\lambda > 0$).
    For all $n \in [T]$,
    \begin{align*}
        \tfrac{\lambda}{2}\rbr{\tnorm{\sz_{n} - u}^2 - \tnorm{\sz_{n-\wtddual_n} - u}^2} 
        &= \tfrac{\lambda}{2}\pair{\sz_{n-\wtddual_n} + \sz_n - 2u,\, \sz_n - \sz_{n-\wtddual_n}} \le \lambda D\, \tnorm{\sz_{n-\wtddual_n} - \sz_n}.
    \end{align*}
    Combining with \eqref{eq:BCO-proof-prefinal} and letting $D \le \frac{2G}{\lambda}$ (see Fact~\ref{fact:sc-diameter}) yields
    \begin{align*}
        \widebar{R}_T(u) \le \E\sbr{\eR_T(u) - \tfrac{\lambda}{2} \tsum_{n=1}^{T} \tnorm{\sz_{n} - u}^2 + 3G\, \eD_T} + \tfrac{10 G R \deltatot}{r}.
    \end{align*}
    This completes the proof.
\end{proof}

\subsection{Proofs of Corollaries~\ref{cor:BCO-regular} and \ref{cor:BCO-sc}}

\begin{theorem}\label{thm:BCO-decomp-FTRL}
    $\cW_{\textnormal{BCO}}(\cB)$, where $\cB$ executes P-FTRL \eqref{eq:P-FTRL} updates, guarantees that in the 
    \begin{align*}
    \text{general convex case} \colon \, \widebar{R}_T(u) &\le \tfrac{D^2}{2\wteta_T} + \tsum_{n=1}^T \wteta_n \rbr{G\,\E[\tnorm{\Gamma_n}] + ({\dk M}/{\wtdelta_n})^2}\\
    &\quad + GD H_{\eta} + \tfrac{6GR\deltatot}{r},\\
    \text{strongly convex case (\ref{assump:strong-convexity})} \colon \, \widebar{R}_T(u) &\le \tsum_{n=1}^T \tfrac{1/\wteta_n - 1/\wteta_{n-1} - \lambda}{2} \,\E[\norm{\sz_n\!-\!u}^2]\\
    &+ \tsum_{n=1}^T \wteta_n \rbr{3G\,\E[\tnorm{\Gamma_n}] + ({\dk M}/{\wtdelta_n})^2}\\
    &+ 3GD H_{\eta} + \tfrac{10 G R \deltatot}{r},
    \end{align*}
    where $\Gamma_n = \tsum_{m = n - \wtddual_n}^{n-1} \wthtgdelta_m$ and $H_{\eta} = \tsum_{n=1}^T (1 - \wteta_n/{\wteta_{n-\wtddual_n}})$.
\end{theorem}

\begin{proof}
    Both results follow from Theorem~\ref{thm:BCO-decomp} once we bound $\eR_T(u)$ and $\eD_T$ components of $\eR^{\textnormal{drift}}_T(u;W)$.
    To control $\eR_T(u)$ in both cases, apply Theorem~\ref{thm:linear-OCO} for $\tnorm{\wthtgdelta_n} \le \dk M/\wtdelta_n$ and $\tnorm{z_n - u} \le D$, as follows
    \begin{align*}
        \eR_T(u) 
        \le \tsum_{n=1}^{T} \tfrac{1/\wteta_n - 1/\wteta_{n-1}}{2} \norm{\sz_n - u}^2 + \frac12 \sum_{n=1}^{T} \wteta_n \tnorm{\wthtgdelta_n}^2
        \le \tfrac{D^2}{2\wteta_T} + \tsum_{n=1}^T \wteta_n (\dk M/\wtdelta_n)^2.
    \end{align*}
    To control $\eD_T = \tsum_{n=1}^T \tnorm{\sz_{n}\!-\!\sz_{n-\wtddual_n}}$, we use Lemma~\ref{lem:P-FTRL-stability} to write
    \begin{align*}
        \eD_T \le \tsum_{n=1}^T \wteta_{n} \rbr{ \tnorm{\tsum_{m=n-\wtddual_n}^{n-1} \wthtgdelta_m}_{\star} + (1/\wteta_{n} - 1/\wteta_{n-\wtddual_n})D} \le \tsum_{n=1}^T \wteta_n \tnorm{\Gamma_n} + D H_{\eta}. 
    \end{align*}
    Then, the general convex case follows from the corresponding inequality in Theorem~\ref{thm:BCO-decomp}, as follows:
    \begin{align*}
        \widebar{R}_T(u) 
        &\le \E\sbr{\eR_T(u) + G\,\eD_T} + \tfrac{6GR\deltatot}{r}\\
        &\le \tfrac{D^2}{2\wteta_T} + \tsum_{n=1}^T \wteta_n (\dk M/\wtdelta_n)^2 + G \tsum_{n=1}^T \wteta_n \E[\tnorm{\Gamma_n}] + G D H_{\eta} + \tfrac{6GR\deltatot}{r}\\
        &= \tfrac{D^2}{2\wteta_T} + \tsum_{n=1}^T \wteta_n \rbr{G\,\E[\tnorm{\Gamma_n}] + ({\dk M}/{\wtdelta_n})^2} + GD H_{\eta} + \tfrac{6GR\deltatot}{r}.
    \end{align*}
    The strongly convex case follows from the corresponding inequality in Theorem~\ref{thm:BCO-decomp}:
    \begin{align*}
        \widebar{R}_T(u) 
        &\le \E\sbr{\eR_T(u) - \tfrac{\lambda}{2} \tsum_{n=1}^{T}  \tnorm{\sz_{n} - u}^2 + 3G \eD_T}  + \tfrac{10 G R \deltatot}{r}\\
        &\le \tsum_{n=1}^T \tfrac{1/\wteta_n - 1/\wteta_{n-1} - \lambda}{2} \E[\norm{\sz_n\!-\!u}^2]\!+\!\tsum_{n=1}^T \wteta_n (\tfrac{\dk M}{\wtdelta_n})^2\\
        &\quad \!+\!3G\rbr{\tsum_{n=1}^T \wteta_n \E[\tnorm{\Gamma_n}]\!+\!D H_{\eta}}\!+\!\tfrac{10GR\deltatot}{r}\\
        &= \tsum_{n=1}^T \tfrac{1/\wteta_n - 1/\wteta_{n-1} - \lambda}{2} \,\E[\norm{\sz_n\!-\!u}^2]\\
        &\quad + \tsum_{n=1}^T \wteta_n \rbr{3G\,\E[\tnorm{\Gamma_n}]\!+\!({\dk M}/{\wtdelta_n})^2}\!+\!3GD H_{\eta}\!+\!\tfrac{10 G R \deltatot}{r}.
    \end{align*}
    This concludes the proof for both cases.
\end{proof}

\begin{lemma}\label{lem:BCO-meandering}
    For $\Gamma_n = \tsum_{m = n - \wtddual_n}^{n-1} \wthtgdelta_m$ from Theorem~\ref{thm:BCO-decomp-FTRL} and $\delta'_n = \delta_{\rho(n)+d_{\rho(n)}}$, it holds that
    \begin{align*}
        G\,\E\sbr{\norm{\Gamma_n}} \le 4 (G^2 \, \wtddual_n + ({\dk M}/{\delta'_n})^2).
    \end{align*} 
\end{lemma}

\begin{proof}
    Note that $\tnorm{\wtgdelta_n} \le G$ and $\tnorm{\wthtgdelta_n} \le \frac{\dk M}{\wtdelta_n}$. With probability one, it holds that
    \begin{align*}
        \norm{\Gamma_n}
        \le \tsum_{m=n-\wtddual_n}^{n-1} \tnorm{\wtgdelta_m} + \tnorm{\underbrace{\tsum_{m=n-\wtddual_n}^{n-1} (\wthtgdelta_m - \wtgdelta_m)}_{=: M_n} }
        \le G \wtddual_n + \tnorm{M_n}.
    \end{align*}
    Since $\delta'_n = \delta_{\rho(n)+d_{\rho(n)}}$, it holds that $\delta'_n \le \delta_t$ for all $t \in \{s: \epb_s \in I_{\rho(n)}\}$, because for all such $t$ we have $\epa_t < \epb_t < \epb_{\rho(n)} < \epa_{\rho(n)+d_{\rho(n)}+1}$ (i.e., $t \le \rho(n)+d_{\rho(n)}$) and sequence $(\delta_t)_{t=1}^T$ is non-increasing.
    Let $G' = G +\dk M / \delta'_n$ so that $\tnorm{\htgdelta_t - g^{\delta}_t} \le G'$ for all $t \in \{s: \epb_s \in I_{\rho(n)}\}$.
    Let $v_t = (\htgdelta_t - g^{\delta}_t) \ind(\epb_t \in I_{\rho(n)})$ for $t\in [T]$, so that $M_n = \tsum_{t=1}^T v_t$, $\E[v_t] = 0$, and $\tnorm{v_t} \le G'\, \ind(\epb_t \in I_{\rho(n)})$. Consider sigma-algebras $\cF^{u}_t = \sigma(u_1, ..., u_{t-1})$. Then, for all $t\in[T]$, $v_t \in \cF^u_{t+1}$ and $\E[v_{t}|\cF_t^u] = 0$. Therefore, it holds that
    \begin{align*}
        \E\sbr{\norm{M_n}^2} &= \tsum_{t=1}^T \E\sbr{\tnorm{v_t}^2} + 2 \tsum_{t=1}^T \tsum_{s=t+1}^T \E\sbr{ \pair{v_t, v_s} }\\
        &\le \tsum_{t=1}^T (G')^2 \ind(\epb_t \in I_{\rho(n)}) + 2 \tsum_{t=1}^T \tsum_{s=t+1}^T \E\sbr{ \pair{v_t, \E[v_s | \cF^{u}_{s}]}}\\
        &= (G')^2 \,\wtddual_n + 0.
    \end{align*}
    Combining these results and applying Jensen's inequality, we prove the inequality
    \begin{align*}
        G\,\E\sbr{\norm{\Gamma_n}} 
        &\le G^2\, \wtddual_n + G\sqrt{\E[\norm{M_n}^2]} 
        \le G^2\, \wtddual_n  + G (G + \dk M/\delta'_n) \sqrt{\wtddual_n}\\
        &\le 2G^2 \, \wtddual_n + (G^2 \wtddual_n)^{1/2}  (\dk M/\delta'_n)
        \le 4 (G^2 \, \wtddual_n + ({\dk M}/{\delta'_n})^2).
    \end{align*}
    This concludes the proof of the lemma.
\end{proof}

\restatecorollary{cor:BCO-regular}{
    $\cW_{\textnormal{BCO}}(\cB)$, where $(\delta_t)_{t\ge 1}$ is set to either a round-dependent schedule $\delta_t = r\,\min\cbr{ 1, \frac{\sqrt{\nu \dk}}{t^{1/4}}}$ or a fixed schedule $\delta_t = r\,\min\cbr{ 1, \frac{\sqrt{\nu \dk}}{T^{1/4}}}$ and $\cB$ runs P-FTRL~\eqref{eq:P-FTRL} with $\wteta_n = \frac{D/G}{\sqrt{n + \sum_{m=1}^n (\wtsigdual_m  + ({\nu \dk r}/{\delta'_m})^2)} }$, guarantees
    \begin{align*}
        \widebar{R}_T(x^*) = O\rbr{GD\sbr{\sqrt{\dtot} + T^{3/4} \sqrt{\nu \dk}} }.
    \end{align*}
}
\begin{proof}
    The proof for a fixed schedule is the same as for a varying schedule, so we only present the latter.
    From Theorem~\ref{thm:BCO-decomp-FTRL} and Lemma~\ref{lem:BCO-meandering}, we have the following 
     \begin{align*}
        \widebar{R}_T(x^*) 
        &\le \tfrac{D^2}{2\wteta_T} + \tsum_{n=1}^T \wteta_n \rbr{G\,\E[\tnorm{\Gamma_n}] + ({\dk M}/{\wtdelta_n})^2} + GD H_{\eta} + \tfrac{6GR\deltatot}{r}\\
        &\le \tfrac{D^2}{\wteta_T} + 5G^2\tsum_{n=1}^T \wteta_n \rbr{\wtddual_n + ({\nu \dk r}/{\delta'_n})^2} + GD H_{\eta} +  \tfrac{6GD\deltatot}{r}.
    \end{align*}
    We bound the terms $1/\wteta_T$, $\tsum_{n=1}^T \wteta_n \rbr{\wtddual_n + ({\nu \dk r}/{\delta'_n})^2}$, and $H_{\eta}$ separately.
    \begin{enumerate}[before=\vspace{-0.2cm}]
        \item As $\delta_T = \min_{t\in[T]}\delta_t$ and $\tsum_{n=1}^T \wtsigdual_n = \dtot$, it holds that
        \begin{align*}
            \tfrac{1}{\wteta_T} =  \tfrac{G}{D} \sqrt{T + \tsum_{m=1}^T (\wtsigdual_m  + ({\nu \dk r}/{\delta'_m})^2)}  \le \tfrac{G}{D} \rbr{\sqrt{T + \dtot} + \tfrac{\nu \dk r}{\delta_T}\sqrt{T}}.
        \end{align*}
        \item As $\tsum_{m=1}^n \wtddual_m \le \tsum_{m=1}^n \wtsigdual_n$ for all $n\in [T]$ according to Fact~\ref{fact:dual-sums}, we write
        \begin{align*}
            \tsum_{n=1}^T \wteta_n \rbr{\wtddual_n + ({\nu \dk r}/{\delta'_n})^2} 
            &\le \tfrac{D}{G} \tsum_{n=1}^T \frac{\wtddual_n + ({\nu \dk r}/{\delta'_n})^2}{\sqrt{\sum_{m=1}^n (\wtddual_m + ({\nu \dk r}/{\delta'_m})^2) }} \\
            &\le \tfrac{2D}{G}  \sqrt{\tsum_{n=1}^T (\wtddual_n + ({\nu \dk r}/{\delta'_n})^2)} \\
            &\le \tfrac{2D}{G} \rbr{ \sqrt{T + \dtot} +  ({\nu \dk r}/{\delta_T})\sqrt{T}},
        \end{align*}
        where the second inequality follows from Fact~\ref{fact:sqrt-telescope} with $X_n = \wtddual_n + ({\nu \dk r}/{\delta'_n})^2$.
        \item For all $n \in [T]$, it holds that
        \begin{align*}
            \tsum_{m=n}^{n+\wtsigdual_n} \wteta_m 
            &\le \tsum_{m=n}^{n+\wtsigdual_n} \tfrac{D/G}{\sqrt{m + \sum_{k=n}^{m} \wtsigdual_k }}
            \myle{a} \tsum_{m=0}^{\wtsigdual_n} \tfrac{D/G}{\sqrt{(m+1)+\sum_{k=0}^{m} (\wtsigdual_n - k)}}\\
            &\myle{b} \tsum_{m=0}^{\wtsigdual_n} \tfrac{2D/G}{\sqrt{(m+1)(\wtsigdual_n+1)}}
            \myle{c} \tfrac{4D}{G},
        \end{align*}
        where (a) follows from Fact~\ref{fact:backlog-increments}, (b) applies inequality $\tsum_{k=0}^{m} (\wtsigdual_n - k) = \tfrac{(m+1)(2\wtsigdual_n - m)}{2} \ge \tfrac{(m+1)\wtsigdual_n}{2}$, and (c)  uses Fact~\ref{fact:sqrt-telescope} to conclude $\tsum_{m=0}^{\wtsigdual_n} \frac{1}{\sqrt{m+1}} \le 2\sqrt{\wtsigdual_n+1}$. 
        Therefore, by Lemma~\ref{lem:H-eta-bound} and the bound on $1/\wteta_T$ above, we have
        \begin{align*}
            H_{\eta}
            &\le \tfrac{\max_{n\in [T]} {\sum_{m=n}^{n+\wtsigdual_n} \wteta_m } }{\wteta_T}
            \le 8 \rbr{\sqrt{T + \dtot} + ({\nu \dk r}/{\delta_T}) \sqrt{T}}.
        \end{align*}
    \end{enumerate}
    Combining all these results, we have
    \begin{align*}
        \widebar{R}_T(x^*) = O\rbr{GD \rbr{\sqrt{T + \dtot} + ({\nu \dk r}/{\delta_T})\sqrt{T} + \deltatot/r}}.
    \end{align*}
    Next, we provide the bounds on $r/\delta_T$ and $\deltatot/r$ in order to derive the final bound. Trivially, it holds that ${r}/{\delta_T} \le \max\cbr{1, \tfrac{T^{1/4}}{\sqrt{\nu \dk}}}$. Also, note that $\sum_{t=1}^T \frac{1}{t^{1/4}} \le 1 + \int_1^T x^{-1/4}\,dx = 1 + \frac{4}{3}(T^{3/4}-1) \le \frac{4}{3}T^{3/4}$, so that
    \begin{align*}
        \deltatot / r = \tsum_{t=1}^T \min\{ 1, \frac{\sqrt{\nu \dk}}{t^{1/4}}\} \le \tsum_{t=1}^T \tfrac{\sqrt{\nu \dk}}{t^{1/4}} \le \frac{4}{3}T^{3/4} \sqrt{\nu \dk}.
    \end{align*}
    Thus, we have $r/\delta_T \le \max\cbr{1, \tfrac{T^{1/4}}{\sqrt{\nu \dk}}}$ and ${\deltatot}/{r} \le \tfrac{4}{3} T^{3/4} \sqrt{\nu \dk}$.
    Finally, substituting the bounds for $r/\delta_T$ and $\deltatot/r$, we have
    \begin{align*}
        \widebar{R}_T(x^*)
        = O\rbr{GD \rbr{\sqrt{T + \dtot} + T^{3/4}\sqrt{\nu \dk} + T^{1/2}(\nu {\dk})}} = O\rbr{GD \rbr{\sqrt{\dtot} + T^{3/4}\sqrt{\nu \dk}}}. 
    \end{align*}
\end{proof}

\restatecorollary{cor:BCO-sc}{
    Under $\lambda$-strong convexity (\ref{assump:strong-convexity}), $\cW_{\textnormal{BCO}}(\cB)$, where $(\delta_t)_{t\ge 1}$ is set to either a round-dependent schedule $\delta_t = r\,\min\cbr{1,(\tfrac{\nu^2 {\dk}^2 \ln(et)}{t})^{1/3} }$ or a fixed schedule $\delta_t = r\,\min\cbr{1,(\tfrac{\nu^2 {\dk}^2 \ln(eT)}{T})^{1/3}}$ and $\cB$ runs P-FTRL~\eqref{eq:P-FTRL} with $\wteta_n = \tfrac{1}{n\lambda}$, guarantees
    \begin{align*}
    \widebar{R}_T(x^*)
    &= O\!\left(
    \tfrac{G^2}{\lambda}
    \left[
    \min\{\sigmax \ln T,\, \sqrt{\dtot}\}
     + T^{2/3} (\ln T)^{1/3} (\nu {\dk})^{2/3}
    \right]
    \right).
    \end{align*}
}
\begin{proof}
    The proof for a fixed schedule is the same as for a varying schedule, so we only present the latter.
    From Theorem~\ref{thm:BCO-decomp-FTRL} and Lemma~\ref{lem:BCO-meandering}, noting that $1/\wteta_n - 1/\wteta_{n-1}-\lambda = 0$, we have the following 
     \begin{align*}
        \widebar{R}_T(x^*) 
        &\le  0 + \tsum_{n=1}^T \wteta_n \rbr{3G\,\E[\tnorm{\Gamma_n}] + ({\dk M}/{\wtdelta_n})^2} + 3GD H_{\eta} + \tfrac{10 G R \deltatot}{r}, \\
        &\le 13 G^2\tsum_{n=1}^T \wteta_n \rbr{\wtddual_n + ({\nu \dk r}/{\delta'_n})^2} + 3GD H_{\eta} +  \tfrac{10GD\deltatot}{r}.
    \end{align*}
    We bound the terms $\tsum_{n=1}^T \wteta_n \rbr{\wtddual_n + ({\nu \dk r}/{\delta'_n})^2}$ and $H_{\eta}$ separately.
    \begin{enumerate}[before=\vspace{-0.2cm}]
        \item Using Fact~\ref{fact:telescopic-sums-with-duals} for dual delays $\wtddual_n$ and Fact~\ref{fact:log-telescope}, we write
        \begin{align*}
            \tsum_{n=1}^T \wteta_n \rbr{\wtddual_n + ({\nu \dk r}/{\delta'_n})^2} 
            &\le \tfrac{1}{\lambda} \min\cbr{\sigmax \ln(eT),\, 2\sqrt{\dtot}} + \tfrac{1}{\lambda} ({\nu \dk r}/{\delta_T})^2 \ln(eT).
        \end{align*}
        \item Substituting learning rate values into $H_{\eta}$ and applying Fact~\ref{fact:telescopic-sums-with-duals} for dual delays $\wtddual_n$ yields 
        \begin{align*}
            H_{\eta} = \tsum_{n=1}^T (1 - \tfrac{\wteta_n}{\wteta_{n-\wtddual_n}}) = \tsum_{n=1}^T \tfrac{\wtddual_n}{n} \le \min\cbr{\sigmax \ln(eT),\, 2\sqrt{\dtot}}.
        \end{align*}
    \end{enumerate}
    Combining all these results and letting $D \le \frac{2G}{\lambda}$ (see Fact~\ref{fact:sc-diameter}), we have
    \begin{align*}
        \widebar{R}_T(x^*) 
        &\le \rbr{\tfrac{13G^2}{\lambda} + 3GD} \min\cbr{\sigmax \ln(eT),\, 2\sqrt{\dtot}} +  \tfrac{13G^2}{\lambda} ({\nu \dk r}/{\delta_T})^2 \ln(eT) + \tfrac{10GD\deltatot}{r}\\
        & = O\rbr{\tfrac{G^2}{\lambda} \rbr{\min\cbr{\sigmax \ln(eT),\, \sqrt{\dtot}} + ({\nu \dk r}/{\delta_T})^2 \ln(eT) + \deltatot/r}}.
    \end{align*}
    Next, we provide the bounds on $r/\delta_T$ and $\deltatot/r$ in order to derive the final bound. Trivially, it holds that ${r}/{\delta_T} \le \max\cbr{1, (\tfrac{T}{\nu^2 {\dk}^2 \ln(eT)})^{1/3}}$. Also, note that $\sum_{t=1}^T \frac{1}{t^{1/3}} \le 1 + \int_1^T x^{-1/3}\,dx = 1 + \frac{3}{2}(T^{2/3}-1) \le \frac{3}{2}T^{2/3}$, so that
    \begin{align*}
        \deltatot / r = \tsum_{t=1}^T \min\cbr{1,(\tfrac{\nu^2 {\dk}^2 \ln(et)}{t})^{1/3} }\le \tsum_{t=1}^T \tfrac{(\nu^2 {\dk}^2 \ln(eT))^{1/3}}{t^{1/3}} \le \frac{3}{2} T^{2/3} \ln(eT)^{1/3} (\nu {\dk})^{2/3}.
    \end{align*}
    Thus, we have $r/\delta_T \le \max\cbr{1, (\tfrac{T}{\nu^2 {\dk}^2 \ln(eT)})^{1/3}}$ and ${\deltatot}/{r} \le \frac{3}{2} T^{2/3} \ln(eT)^{1/3} (\nu {\dk})^{2/3}$.
    Finally, substituting the bounds for $r/\delta_T$ and $\deltatot/r$, we have
    \begin{align*}
        \widebar{R}_T(x^*)
        &= O\rbr{\tfrac{G^2}{\lambda} \rbr{\min\{\sigmax \ln(eT),\, \sqrt{\dtot}\} + T^{2/3} \ln(eT)^{1/3} (\nu \dk)^{2/3} + \nu^2 {\dk}^2 \ln(eT)}}\\
        &= O\rbr{\tfrac{G^2}{\lambda} \rbr{\min\{\sigmax \ln(eT),\, \sqrt{\dtot}\} + T^{2/3} \ln(eT)^{1/3} (\nu \dk)^{2/3}}}.
    \end{align*}
\end{proof}

\section{Bandit Convex Optimization with Two-Point Delayed Feedback}\label{app:two-point-BCO}

In this section, we extend the analysis of BCO with delayed feedback (see Section~\ref{sec:setting}) to the two-point feedback setting \cite{agarwal2010}. In each round $t$, the player selects a pair of points $x_t^1,x_t^2 \in \cK$ (the same
domain as Section~\ref{sec:setting}), incurs loss $\tfrac{1}{2}\bigl(f_t(x_t^1)+f_t(x_t^2)\bigr)$, and later receives the delayed feedback $(t, f_t(x_t^1), f_t(x_t^2))$ in round $t + d_t$.

\begin{figure}[H]
\centering
\fbox{\parbox{0.97\textwidth}{
    \textbf{BCO with Two-Point Delayed Feedback}\vspace{5pt}\\
    $\bullet$ \textit{Latent Parameters:} number of rounds $T$.\\
    $\bullet$ \textit{Pre-game:} adversary selects loss functions $f_{t}:\cK \to \R$ and delays $d_t \in \{0,\ldots,T-t\}$ for all $t \in [T]$.
    
    \vspace{5pt}
    \noindent For each round $t=1, 2, \ldots, T$:
    \begin{enumerate}[leftmargin=1cm,noitemsep,before=\vspace{-0.0cm},after=\vspace{-0.0cm}]
        \item The player predicts points $(x^1_t, x^2_t)$ and incurs loss $\frac{1}{2}(f_t(x^1_t) + f_t(x^2_t))$.
        \item For all $s \le t$ such that $s+d_s = t$, the environment reveals $(s, f_s(x_s^1), f_s(x_s^2))$.
    \end{enumerate}
}}
\captionsetup{font=footnotesize}
\caption{\small Bandit convex optimization with two-point delayed feedback.}
\label{fig:setting-2pt}
\end{figure}

The two-point regret against a comparator $u\in\cK$ is defined as
 $\Rtwo_T(u) = \tsum_{t=1}^T (\frac{1}{2}(f_t(x^1_t) + f_t(x^2_t)) - f_t(u))$ and its expectation as $\wbRtwo_T(u) = \E[\Rtwo_T(u)]$. Letting $x^* \in \argmin_{x\in\cK} \tsum_{t=1}^T f_t(x)$ denote the best action in hindsight, the goal is to minimize $\wbRtwo_T(x^*)$.

\paragraph{Preliminaries on two-point gradient estimation.}
We briefly review the standard two-point gradient estimator. By Theorem~\ref{thm:SPGE}, the gradient of the $\delta$-smoothed loss $f^{\delta}$ admits a spherical expectation representation; by symmetry of $u\sim\Unif(\Sphere)$, this can be written in the two-point form
\[
\nabla f^{\delta}(x)
= \E_{u\sim \Unif(\Sphere)}\!\left[\tfrac{\dk}{2\delta}\bigl(f(x+\delta u)-f(x-\delta u)\bigr)u\right].
\]
Accordingly, the estimator $\tfrac{\dk}{2\delta}\bigl(f(x+\delta u)-f(x-\delta u)\bigr)u$ is unbiased for $\nabla f^{\delta}(x)$ \citep{agarwal2010,shamir2015}. Moreover, unlike the one-point estimator, its second moment does not scale with $1/\delta$, as captured by the following lemma.

\begin{lemma}[Lemma 5, \citet{shamir2015}]\label{lem:shamir}
There exists a numerical constant $c>0$ such that for every $G$-Lipschitz function $f:\cK\to\R$, $\delta \in (0,r]$, and $x \in (1-\delta/r)\cK$, it holds that
\begin{align*}
\E_{u\sim\Unif(\Sphere)}\!\left[\left\|\tfrac{\dk}{2\delta}\bigl(f(x+\delta u)-f(x-\delta u)\bigr)u\right\|^2\right]
\le c\,\dk G^2.
\end{align*}
\end{lemma}

\paragraph{Two-point Delayed BCO via Drift-Penalized OLO.} The wrapper $\cW_{\textnormal{2p-BCO}}$ (Algorithm~\ref{alg:reduction-TP-BCO}) follows the structure of $\cW_{\textnormal{BCO}}$, with two-point gradient estimates instead of single-point estimates. As before, the algorithm takes as input a non-increasing sequence of smoothing parameters $(\delta_t)_{t\ge 1} \subset (0,r]$ fixed in advance. In round $t$, given the current base-prediction $y_t = \zcur$ from $\cB$, the algorithm samples an exploration direction $u_t \sim \Unif(\Sphere)$ and predicts $x^1_t = x_t + \delta_t u_t$, $x^2_t = x_t - \delta_t u_t$ for $x_t = (1-\delta_t/r)y_t$. Upon receiving observations $(s, f_s(x^1_s), f_s(x^2_s))$, it constructs the two-point gradient estimate $\htgdelta_s = \tfrac{\dk}{2\delta_s} (f_s(x^1_s) - f_s(x^2_s)) u_s$. This estimate satisfies $\E[\htgdelta_t \mid x_t] = \nabla f^{\delta}_t(x_t)$ with reduced variance compared to single-point estimation.

\begin{algorithm}[H]
\caption{$\cW_{\textnormal{2p-BCO}}(\cB)$}\label{alg:reduction-TP-BCO}
\SetKwInOut{Input}{Input}
\SetKwInOut{Access}{Access}
\Access{ Algorithm $\cB$ for OLO with drift penalization.}
\SetKwInOut{Parameters}{Parameters}
\Parameters{Non-increasing sequence $(\delta_t)_{t\ge 1} \subset (0,r]$.}

\vspace{0.1cm}
Initialize $\zcur$ with initial output of $\cB$.

\vspace{0.1cm}
\textbf{For} round $t = 1, 2, \ldots, T$:
\begin{enumerate}[leftmargin=0.65cm,noitemsep]
    \item Select $x_t = (1-\delta_t/r)\zcur$, sample $u_t \sim \Unif(\Sphere)$, and predict $x_t^1 = x_t + \delta_t u_t$, $x_t^2 = x_t - \delta_t u_t$.
    \item For each round $s$ such that $s+d_s = t$, in the order of observation:
    \begin{itemize}[leftmargin=0.15cm,before=\vspace{-0.0cm},noitemsep]
        \item Receive $(s, f_s(x^1_s), f_s(x^2_s))$, forward $(\frac{\dk}{2\delta_s} (f_s(x^1_s) - f_s(x^2_s)) u_s, \ddual_s, \sigdual_s)$ to $\cB$, and set $\zcur$ to the new output of $\cB$.
    \end{itemize}
\end{enumerate}
\end{algorithm}

The analysis follows the same structure as for $\cW_{\textnormal{BCO}}$. The algorithm updates $\cB$ sequentially on tuples $(\wthtgdelta_n, \wtddual_n, \wtsigdual_n)$ for $n \in [T]$, where $\wthtgdelta_n = \tfrac{\dk}{2\wtdelta_n} (\wtf_n(\wtx^1_n) - \wtf_n(\wtx^2_n)) \wtu_n$, producing base-predictions $z_n$ that yield predictions $\wtx^1_n = \wtx_n + \wtdelta_n \wtu_n$ and $\wtx^2_n = \wtx_n - \wtdelta_n \wtu_n$ for $\wtx_n = (1-\wtdelta_n/r) \wty_n$ and $\wty_n = z_{n-\wtddual_n}$. The following theorem bounds the expected two-point regret of $\cW_{\textnormal{2p-BCO}}(\cB)$ in terms of the drift-penalized regret of $\cB$, with an additional bias of order $\deltatot = \sum_{t=1}^T \delta_t$ due to smoothing.

\begin{theorem}[Regret of Algorithm~\ref{alg:reduction-TP-BCO}]\label{thm:TP-BCO-decomp}
    For any base-algorithm $\cB$, non-increasing sequence $(\delta_t)_{t\ge 1} \subset (0,r]$, and comparator $u \in \cK$, $\cW_{\textnormal{2p-BCO}}(\cB)$ guarantees
    \begin{align*}
       \wbRtwo_T(u) &\le \E\sbr{\eR^{\textnormal{drift}}_{T}(u; G)} + \tfrac{7GR\deltatot}{r},
    \end{align*}
    where $\eR^{\textnormal{drift}}_{T}(u; W)$ is the drift-penalized regret \eqref{eq:drift-regret} of $\cB$ for loss vectors $c_n = \wthtgdelta_n$ and lags $\lambda_n = \wtddual_n$. 

    Under $\lambda$-strong convexity (\ref{assump:strong-convexity}), it further holds that
    \begin{align*}
        \wbRtwo(u) &\le \E\!\sbr{\eR^{\textnormal{drift}}_{T}(u; 3G)\!-\!\tfrac{\lambda}{2}\!\tsum_{n=1}^{T}  \tnorm{\sz_{n}\!-\!u}^2}\!+\!\tfrac{11 G R \deltatot}{r}.
    \end{align*}
\end{theorem}
\begin{proof}
The argument follows the proof of Theorem~\ref{thm:BCO-decomp}. Here the played center $x_t$ plays the role of $y_t$ in that theorem, so the same steps bound $\tsum_{t=1}^T\bigl(f_t(x_t)-f_t(u)\bigr)$. It remains to relate the two-point loss to the center loss: by $G$-Lipschitzness,
\[
\left|\tfrac{1}{2}\bigl(f_t(x_t^1)+f_t(x_t^2)\bigr)-f_t(x_t)\right|
\le \tfrac{G}{2}\bigl(\|x_t^1-x_t\|+\|x_t^2-x_t\|\bigr)
= G\delta_t.
\]
Summing over $t$ yields
$\left|\tsum_{t=1}^T\left[\tfrac{1}{2}\bigl(f_t(x_t^1)+f_t(x_t^2)\bigr)-f_t(x_t)\right]\right|\le G\deltatot \le \frac{GR\deltatot}{r}$,
which gives the claim.
\end{proof}

Selecting an appropriate smoothing sequence $(\delta_t)_{t=1}^T$ and instantiating $\cB$ with P-FTRL~\eqref{eq:P-FTRL} yields concrete guarantees. 

\begin{corollary}\label{cor:TP-BCO}
    $\cW_{\textnormal{2p-BCO}}(\cB)$ with $\delta_t = \frac{r}{\sqrt{t}}$, where $\cB$ runs P-FTRL~\eqref{eq:P-FTRL} with $\wteta_n = \tfrac{D/G}{\sqrt{n\dk + \sum_{m=1}^n \wtsigdual_m} }$, guarantees
    \begin{align*}
        \wbRtwo_T(x^*) = O\rbr{GD\sbr{\sqrt{\dtot} + \sqrt{T\dk}} }.
    \end{align*}
    Composing with the skipping scheme (Algorithm~\ref{alg:skipping-scheme}), $\cS_{\textnormal{skip}}(\cW_{\text{2p-BCO}}(\cB))$ gets $\wbRtwo_T(x^*) = O(GD [\min_{Q\subseteq [T]}\{|Q| + (\sum_{t\notin Q} d_t)^{\frac{1}{2}}\} + \sqrt{T\dk}])$.
\end{corollary}

\begin{corollary}\label{cor:TP-BCO-sc}
    Under $\lambda$-strong convexity (\ref{assump:strong-convexity}), $\cW_{\textnormal{2p-BCO}}(\cB)$ with $\delta_t = \frac{r}{t}$, where $\cB$ runs P-FTRL~\eqref{eq:P-FTRL} with $\wteta_n = \frac{1}{n\lambda}$, guarantees
    \begin{align*}
        \wbRtwo_T(x^*) = O\rbr{\tfrac{G^2}{\lambda} \sbr{\min\{\sigmax \ln(T),\, \sqrt{\dtot}\} + \dk \ln T}}.
    \end{align*}
\end{corollary}

\subsection{Proofs of Corollaries \ref{cor:TP-BCO} and \ref{cor:TP-BCO-sc}}

\begin{theorem}\label{thm:TP-BCO-decomp-FTRL}
    $\cW_{\textnormal{2p-BCO}}(\cB)$, where $\cB$ executes P-FTRL \eqref{eq:P-FTRL} updates, guarantees that in the 
    \begin{align*}
    \text{general convex case} \colon \, \wbRtwo_T(u) &\le \tfrac{D^2}{2\wteta_T} + \tsum_{n=1}^T \wteta_n \rbr{G\,\E[\tnorm{\Gamma_n}] + G^2\, c\dk} + GD H_{\eta} + \tfrac{7GR\deltatot}{r},\\
    \text{strongly convex case (\ref{assump:strong-convexity})} \colon \, \wbRtwo_T(u) &\le \tsum_{n=1}^T \tfrac{1/\wteta_n - 1/\wteta_{n-1} - \lambda}{2} \,\E[\norm{\sz_n\!-\!u}^2]\\
    &+ \tsum_{n=1}^T \wteta_n \rbr{3G\,\E[\tnorm{\Gamma_n}] +  G^2\, c\dk} + 3GD H_{\eta} + \tfrac{11 G R \deltatot}{r}, 
    \end{align*}
    where $\Gamma_n = \tsum_{m = n - \wtddual_n}^{n-1} \wthtgdelta_m$, $H_{\eta} = \tsum_{n=1}^T (1 - \wteta_n /{\wteta_{n-\wtddual_n}})$, and $c$ is the constant from Lemma~\ref{lem:shamir}.
\end{theorem}
\begin{proof}
    Follows identically to Theorem~\ref{thm:BCO-decomp-FTRL}, using the bound $\E[\tnorm{\wthtgdelta_n}^2] \le c\,\dk G^2$ from Lemma~\ref{lem:shamir}.
\end{proof}

\begin{lemma}\label{lem:TP-BCO-meandering}
    For the constant $c$ from Lemma~\ref{lem:shamir}, $\Gamma_n = \tsum_{m = n - \wtddual_n}^{n-1} \wthtgdelta_m$ satisfies
    \begin{align*}
        \E\sbr{\norm{\Gamma_n}_2} \le  2G\,(\wtddual_n + c \dk).
    \end{align*} 
\end{lemma}

\begin{proof}
    Note that $\tnorm{\wtgdelta_n} \le G$, $\E[\wthtgdelta_n] = \wtgdelta_n$ by Theorem~\ref{thm:SPGE} and $\E[\tnorm{\wthtgdelta_n - \wtgdelta_n}^2] \le c\dk G^2$ by Lemma~\ref{lem:shamir}.
    With probability one, we write
    \begin{align*}
        \norm{\Gamma_n}
        \le \tsum_{m=n-\wtddual_n}^{n-1} \tnorm{\wtgdelta_m} + \tnorm{\underbrace{\tsum_{m=n-\wtddual_n}^{n-1} (\wthtgdelta_m - \wtgdelta_m)}_{=: M_n} }
        \le G \wtddual_n + \tnorm{M_n}.
    \end{align*}
    Let $v_t = (\htgdelta_t - g^{\delta}_t) \ind(\epb_t \in I_{\rho(n)})$, so that $M_n = \tsum_{t=1}^T v_t$, $\E[v_t] = 0$, and $\E[\tnorm{v_t}^2] \le c\dk G^2 \ind(\epb_t \in I_{\rho(n)})$.
    Let $\cF^{u}_t = \sigma(u_1, ..., u_{t-1})$. Then, for all $t\in[T]$, $v_t \in \cF^u_{t+1}$ and $\E[v_{t}|\cF_t^u] = 0$. Therefore, it holds that
    \begin{align*}
        \E\sbr{\norm{M_n}^2} &= \tsum_{t=1}^T \E\sbr{\tnorm{v_t}^2} + 2 \tsum_{t=1}^T \tsum_{s=t+1}^T \E\sbr{ \pair{v_t, v_s} }\\
        &\le \tsum_{t=1}^T c\dk G^2 \ind(\epb_t \in I_{\rho(n)}) + 2 \tsum_{t=1}^T \tsum_{s=t+1}^T \E\sbr{ \pair{v_t, \E[v_s | \cF^{u}_{s}]}}\\
        &= c\,\dk G^2 \,\wtddual_n + 0.
    \end{align*}
    Combining these results and applying Jensen's inequality, we prove the inequality
    \begin{align*}
        \E\sbr{\norm{\Gamma_n}} 
        \le G\, \wtddual_n + \sqrt{\E[\norm{M_n}^2]}
        \le G\, \wtddual_n  + G (c\dk)^{1/2} \sqrt{\wtddual_n}
        \le 2G\, (\wtddual_n + c\dk).
    \end{align*}
    This concludes the proof of the lemma.
\end{proof}

\restatecorollary{cor:TP-BCO}{
    $\cW_{\textnormal{2p-BCO}}(\cB)$ with $\delta_t = \frac{r}{\sqrt{t}}$, where $\cB$ runs P-FTRL~\eqref{eq:P-FTRL} with $\wteta_n = \tfrac{D/G}{\sqrt{n\dk + \sum_{m=1}^n \wtsigdual_m} }$, guarantees
    \begin{align*}
        \wbRtwo_T(x^*) = O\rbr{GD\sbr{\sqrt{\dtot} + \sqrt{T\dk}} }.
    \end{align*}
    Composing with the skipping scheme, $\cS_{\textnormal{skip}}(\cW_{\textnormal{2p-BCO}}(\cB))$ gets $\wbRtwo_T(x^*) = O(GD [\min_{Q\subseteq [T]}\{|Q| + (\sum_{t\notin Q} d_t)^{\frac{1}{2}}\} + \sqrt{T\dk}])$.
}
\begin{proof}
    From Theorem~\ref{thm:TP-BCO-decomp-FTRL} and Lemma~\ref{lem:TP-BCO-meandering}, we have the following 
     \begin{align*}
        \wbRtwo_T(x^*) 
        &\le \tfrac{D^2}{2\wteta_T} + \tsum_{n=1}^T \wteta_n \rbr{G\,\E[\tnorm{\Gamma_n}] + G^2\, c\dk} + GD H_{\eta} + \tfrac{7GR\deltatot}{r}\\
        &\le \tfrac{D^2}{\wteta_T} + 5(c+1)G^2\tsum_{n=1}^T \wteta_n (\wtddual_n + \dk) + GD H_{\eta} +  \tfrac{7GD\deltatot}{r}.
    \end{align*}
    We bound the terms $1/\wteta_T$, $\tsum_{n=1}^T \wteta_n (\wtddual_n + \dk)$, and $H_{\eta}$ separately.
    \begin{enumerate}[before=\vspace{-0.2cm}]
        \item As $\tsum_{n=1}^T \wtsigdual_n = \dtot$ according to Theorem~\ref{thm:CTM}, it holds that
        \begin{align*}
            \tfrac{1}{\wteta_T} =  \tfrac{G}{D} \sqrt{T\dk + \tsum_{m=1}^T \wtsigdual_m} \le \tfrac{G}{D} \rbr{\sqrt{\dtot} + \sqrt{T\dk}}.
        \end{align*}
        \item As $\tsum_{m=1}^n \wtddual_m \le \tsum_{m=1}^n \wtsigdual_n$ for all $n\in [T]$ according to Fact~\ref{fact:dual-sums}, we write
        \begin{align*}
            \tsum_{n=1}^T \wteta_n (\wtddual_n + \dk) 
            &\le \tfrac{D}{G} \tsum_{n=1}^T \frac{\wtddual_n + \dk}{\sqrt{\sum_{m=1}^n (\wtddual_m + \dk) }} \\
            &\le \tfrac{2D}{G}  \sqrt{\tsum_{n=1}^T (\wtddual_n + \dk)} \\
            &\le \tfrac{2D}{G} \rbr{ \sqrt{\dtot} +  \sqrt{T\dk}},
        \end{align*}
        where the second inequality follows from Fact~\ref{fact:sqrt-telescope} with $X_n = \wtddual_n + \dk$.
        \item Following the similar argument from the proof of Corollary~\ref{cor:BCO-regular}, it holds that $\max_{n\in[T]} \tsum_{m=n}^{n+\wtsigdual_n} \wteta_m \le \tfrac{4D}{G}$. Then, by Lemma~\ref{lem:H-eta-bound} and the bound on $1/\wteta_T$ above, we have
        \begin{align*}
            H_{\eta}
            &\le \tfrac{\max_{n\in [T]} {\sum_{m=n}^{n+\wtsigdual_n} \wteta_m } }{\wteta_T}
            \le 4 \rbr{\sqrt{\dtot} + \sqrt{T\dk}}.
        \end{align*}
    \end{enumerate}
    For our choice of $\delta_t$, we have $\deltatot/r = \tsum_{t=1}^T \tfrac{1}{t^{1/2}} \le 2\sqrt{T}$, which follows immediately from Fact~\ref{fact:sqrt-telescope}.
    Combining all these results, we have
    \begin{align*}
        \wbRtwo_T(x^*) = O\rbr{GD \sbr{\sqrt{\dtot} + \sqrt{T\dk}}}.
    \end{align*}
    The result for the skipping scheme follows from Theorem~\ref{thm:skipping}.
\end{proof}

\restatecorollary{cor:TP-BCO-sc}{
    Under $\lambda$-strong convexity (\ref{assump:strong-convexity}), $\cW_{\textnormal{2p-BCO}}(\cB)$ with $\delta_t = \frac{r}{t}$, where $\cB$ runs P-FTRL~\eqref{eq:P-FTRL} with $\wteta_n = \frac{1}{n\lambda}$, guarantees
    \begin{align*}
        \wbRtwo_T(x^*) = O\rbr{\tfrac{G^2}{\lambda} \sbr{\min\{\sigmax \ln T,\, \sqrt{\dtot}\} + \dk \ln T}}.
    \end{align*}
}
\begin{proof}
    From Theorem~\ref{thm:TP-BCO-decomp-FTRL} and Lemma~\ref{lem:TP-BCO-meandering}, noting that $1/\wteta_n - 1/\wteta_{n-1}-\lambda = 0$, we have the following 
     \begin{align*}
        \wbRtwo_T(x^*) 
        &\le  0 + \tsum_{n=1}^T \wteta_n \rbr{3G\,\E[\tnorm{\Gamma_n}] + G^2\, c\dk} + 3GD H_{\eta} + \tfrac{11 G R \deltatot}{r}, \\
        &\le 13 (c+1) G^2\tsum_{n=1}^T \wteta_n (\wtddual_n + \dk) + 3GD H_{\eta} +  \tfrac{11GD\deltatot}{r}.
    \end{align*}
    We bound the terms $\tsum_{n=1}^T \wteta_n (\wtddual_n + \dk)$ and $H_{\eta}$ separately.
    \begin{enumerate}[before=\vspace{-0.2cm}]
        \item Using Fact~\ref{fact:telescopic-sums-with-duals} for dual delays $\wtddual_n$ and Fact~\ref{fact:log-telescope}, we write
        \begin{align*}
            \tsum_{n=1}^T \wteta_n (\wtddual_n + \dk) 
            &\le \tfrac{1}{\lambda} \min\cbr{\sigmax \ln(eT),\, 2\sqrt{\dtot}} + \tfrac{1}{\lambda} \dk \ln(eT).
        \end{align*}
        \item Substituting learning rate values into $H_{\eta}$ and applying Fact~\ref{fact:telescopic-sums-with-duals} for dual delays $\wtddual_n$ yields 
        \begin{align*}
            H_{\eta} = \tsum_{n=1}^T (1 - \tfrac{\wteta_n}{\wteta_{n-\wtddual_n}}) = \tsum_{n=1}^T \tfrac{\wtddual_n}{n} \le \min\cbr{\sigmax \ln(eT),\, 2\sqrt{\dtot}}.
        \end{align*}
    \end{enumerate}
    For our choice of $\delta_t$, we have $\deltatot/r = \tsum_{t=1}^T \tfrac{1}{t} \le \ln(eT)$, which follows immediately from Fact~\ref{fact:log-telescope}.
    Combining all these results and letting $D \le \frac{2G}{\lambda}$ (see Fact~\ref{fact:sc-diameter}), we have
    \begin{align*}
        \wbRtwo_T(x^*) 
        & = O\rbr{\tfrac{G^2}{\lambda} \sbr{\min\{\sigmax \ln T ,\, \sqrt{\dtot}\} + \dk \ln T  }}.
    \end{align*}
    This concludes the proof of the corollary.
\end{proof}

\section{Skipping Scheme}\label{app:skipping}

The $O(\sqrt{\dtot})$ term in Corollaries~\ref{cor:OCO-regular} and~\ref{cor:BCO-regular} is pessimistic when delays are highly unbalanced: if a few rounds have exceptionally large delays, it is preferable to skip them (i.e., accept full loss and ignore their feedback), improving the bound to $O(\min_{Q \subseteq [T]} \{|Q| + (\sum_{t \notin Q} d_t)^{\frac{1}{2}}\})$. For instance, following \citet{thune2019}, if the first $\tfloor{\sqrt{T}}$ rounds have delay $\Theta(T)$ and the rest have  $1$, we get improvement from $O(T^{3/4})$ to $O(\sqrt{T})$. However, as delays are revealed upon feedback arrival, achieving this adaptively is nontrivial.

To address this, we adopt the skipping technique of \citet{zimmert2020} as an external wrapper, $\cS_{\textnormal{skip}}$ (Algorithm~\ref{alg:skipping-scheme}), making it applicable to any delayed OCO algorithm $\cA$. The wrapper acts as an interface between $\cA$ and the environment, maintaining a tracking set $S$ of rounds awaiting feedback. When feedback for $s \in S$ arrives, it is forwarded to $\cA$; if $s$ remains pending for too long, it is ``skipped'' by forwarding zero-valued feedback and removing $s$ from $S$. Thus, in effect, $\cA$ is presented with a different delayed OCO problem, one with shorter delays and zero losses on skipped rounds.

\begin{algorithm}[h]
\caption{$\cS_{\textnormal{skip}}(\cA)$}\label{alg:skipping-scheme}
\SetKwInOut{Input}{Input}
\SetKwInOut{Access}{Access}
\Access{Algorithm $\cA$ for OCO (BCO) with delays.}
\vspace{0.1cm}
Initialize tracking set $S = \emptyset$ and set $\cD_0 = 0$.

\textbf{For} round $t = 1, 2, \ldots, T$:
\begin{enumerate}[leftmargin=0.5cm,noitemsep,before=\vspace{-0.0cm},after=\vspace{-0.2cm}]
    \item Set $\cD_{t} = \cD_{t-1} + |S|$.
    \item \textbf{For} all $s \in S$ such that $t - s > \sqrt{\cD_{t}}$:
    \begin{itemize}[leftmargin=0.5cm,before=\vspace{-0.1cm},noitemsep]
        \item Forward $(s, 0)$ to $\cA$ and remove $s$ from $S$.
    \end{itemize}
    \item Query $\cA$ for the next decision, obtain $x_t$, play $x_t$, 
          and insert $t$ into the tracking set $S$.
    \item \textbf{For} all $s$ such that $s+d_s=t$, in observation order:
    \begin{itemize}[leftmargin=0.5cm,before=\vspace{-0.1cm},noitemsep]
        \item Receive $(s, v_s)$ from the environment.
        \item \textbf{If} $s\!\in\!S$ \textbf{then} forward $(s, v_s)$ to $\cA$, remove $s$ from $S$.
    \end{itemize}
\end{enumerate}
\end{algorithm}

By construction, for every $t \in [T]$, the wrapper passes feedback with index $t$ to $\cA$ at round $t + d'_t$ for some $d'_t \in \{0, \ldots, d_t\}$: either the true feedback $(t, v_t)$ upon arrival or dummy feedback $(t, 0)$ forwarded preemptively. Let $Q^* \subseteq [T]$ denote the set of skipped rounds, i.e., those for which dummy feedback was forwarded, and define $f'_t = f_t \, \ind(t \notin Q^*)$. Then, in effect, $\cS_{\textnormal{skip}}$ presents $\cA$ with a delayed OCO problem having losses $\{f'_t\}_{t=1}^T$ and delays $\{d'_t\}_{t=1}^T$. The following theorem bounds the regret of $\cS_{\textnormal{skip}}(\cA)$, with proof in the following subsection.

\begin{theorem}\label{thm:skipping}
    The regret $R_T(u)$ of $\cS_{\textnormal{skip}}(\cA)$ against any comparator $u \in \cK$ satisfies $R_T(u) \leq R'_T(u) + G D |Q^*|$, where $R'_T(u) = \tsum_{t=1}^T (f'_t(x_t)-f'_t(u))$ is the regret of $\cA$ on delayed OCO problem with losses and delays $\{f'_t, d'_t\}_{t=1}^T$. Moreover, the number of skipped rounds $|Q^*|$ and the total modified delay $\dtot' = \tsum_{t=1}^T d'_t$ satisfy
    \vspace{-0.1cm}
    \begin{align*}
        |Q^*| + \sqrt{\dtot'} &= O\rbr{\textstyle\min_{Q \subseteq [T]} \cbr{ |Q| + \sqrt{\tsum_{t \notin Q} d_t} }}.
    \end{align*}
\end{theorem}

\vspace{-0.2cm}

Composing $\cS_{\textnormal{skip}}$ with wrappers $\cW_{\textnormal{OCO}}$ and $\cW_{\textnormal{BCO}}$ immediately improves upon Corollaries~\ref{cor:OCO-regular} and~\ref{cor:BCO-regular}, respectively.

\begin{corollary}\label{cor:skipping-OCO-BCO}
    Under the conditions of Corollaries~\ref{cor:OCO-regular} and~\ref{cor:BCO-regular}, applying the skipping wrapper yields:
    \begin{align*}
        &R_T(x^*) = O\bigg(\!GD \bigg[\!\min_{Q \subseteq [T]} \Big\{\!|Q|\!+\!\sqrt{\!\tsum_{t\notin Q} d_t\!}\Big\}\!+\!\sqrt{T}\bigg]\!\bigg),\\
        &\widebar{R}_T(x^*) = O\bigg(\!GD \bigg[\!\min_{Q \subseteq [T]} \Big\{\!|Q|\!+\!\sqrt{\!\tsum_{t\notin Q} d_t\!}\Big\}\!+\!T^{\frac{3}{4}}\sqrt{\!\nu \dk}\bigg]\!\bigg),
    \end{align*}
    for $\cS_{\textnormal{skip}}(\cW_{\textnormal{OCO}}(\cB))$ and $\cS_{\textnormal{skip}}(\cW_{\textnormal{BCO}}(\cB))$, respectively.
\end{corollary}
\begin{proof}
    Follows from Theorem~\ref{thm:skipping} combined with Corollaries~\ref{cor:OCO-regular} and \ref{cor:BCO-regular} for controlling $R'_T(x^*)$.    
\end{proof}

This skipping scheme applies in any setting where a zero loss is admissible, e.g., the general convex case. In the strongly convex case, one could instead skip rounds by substituting $\frac{\lambda}{2}x^2$ loss values. However, $\min_{Q\subseteq[T]}\{|Q|+(\sum_{t\notin Q} d_t)^{1/2}\} = \Omega(\sigmax)$ as shown by \citet{qiu2025}, so our strongly convex bound $O(\min\{\sigmax \ln T, \sqrt{\dtot}\})$ already matches this benchmark up to a $\ln T$ factor. 

\subsection{Proof of Theorem~\ref{thm:skipping}}

\restatetheorem{thm:skipping}{
    The regret $R_T(u)$ of $\cS_{\textnormal{skip}}(\cA)$ against any comparator $u \in \cK$ satisfies
    \begin{align*}
        R_T(u) &\leq R'_T(u) + G D |Q^*|,
    \end{align*}
    where $R'_T(u) = \tsum_{t=1}^T (f'_t(x_t)-f'_t(u))$ is the regret of $\cA$ on delayed OCO problem with losses and delays $\{f'_t, d'_t\}_{t=1}^T$. Moreover, it holds that
    \begin{align*}
        |Q^*| + \sqrt{\tsum_{t=1}^T d'_t} &= O\rbr{\min_{Q \subseteq [T]} \cbr{ |Q| + \sqrt{\tsum_{t \notin Q} d_t} }}.
    \end{align*}
}
\begin{proof}
    The regret bound follows directly from the observation that each skipped round contributes at most $GD$ to the regret. Specifically, for all $t \in [T]$, $f_t(x_t) - f_t(u) \le f'_t(x_t) - f'_t(u) + GD \cdot \ind(t \in Q^*)$. Summing over $t$ yields the first claim.
    It remains to establish the second claim regarding the skipped set $R$. We introduce additional notation: let $S_t$ denote the state of the tracking set $S$ at the start of round $t$. By construction, $d'_t = \tsum_{s=1}^T \ind(t \in S_s) \le d_t$. Then, $\dtot' = \tsum_{t=1}^T d'_t = \tsum_{t=1}^T |S_t|$. As $\cD_t = \cD_{t-1} + |S_t|$, it also holds that $\dtot' = \cD_T$. Furthermore, by the preemption rule, round $t$ is not removed before round $t + d'_t - 1$, which implies $d'_t \le \sqrt{\cD_{t-1}} + 1 \le \sqrt{\dtot'} + 1$ for all $t \in [T]$.
    \paragraph{Part 1 (Bounding $|Q^*|$).} We show that $|Q^*| \le 2\sqrt{\dtot'}$. Let $Q^* = \{t_1, \ldots, t_{|Q^*|}\}$ ordered by time of skipping. We claim that $d'_{t_s} \ge s/2$ for all $s \in [|Q^*|]$. Indeed, for each $s \in [|Q^*|]$,
    \begin{align*}
        d'_{t_s} > \sqrt{\cD_{t_s + d'_{t_s}}} = \sqrt{\tsum_{t=1}^T \min\{d'_t, t_s + d'_{t_s} - t\}} \ge \sqrt{\tsum_{k=1}^{s} d'_{t_k}},
    \end{align*}
    which rearranges to $d'_{t_s} \ge \frac{1 + \sqrt{1 + 4 \sum_{k=1}^{s-1} d'_{t_k}}}{2}$ by solving the quadratic inequality.
    We proceed by induction. The base case gives $d'_{t_1} \ge 1$. Assuming $d'_{t_k} \ge k/2$ for all $k < s$, we obtain
    \begin{align*}
        d'_{t_s} \ge \frac{1 + \sqrt{1 + 4 \sum_{k=1}^{s-1} \frac{k}{2}}}{2} = \frac{1 + \sqrt{1 + s(s-1)}}{2} \ge \frac{s}{2}.
    \end{align*}
    Therefore, we have
    \begin{align*}
        \sqrt{\dtot'} \ge \sqrt{\tsum_{s=1}^{|Q^*|} d'_{t_s}} \ge \sqrt{\tsum_{s=1}^{|Q^*|} \frac{s}{2}} = \sqrt{\tfrac{|Q^*|(|Q^*|+1)}{4}} \ge \frac{|Q^*|}{2}.
    \end{align*}
    \paragraph{Part 2 (Bounding $\dtot'$).} Fix an arbitrary $Q \subseteq [T]$. Using $d'_t \le \sqrt{\dtot'} + 1$ for all $t$, we have
    \begin{align*}
        \tsum_{t \notin Q} d_t \ge \tsum_{t \notin Q} d'_t \ge \max\cbr{0,\, \dtot' - |Q|(\sqrt{\dtot'} + 1)}.
    \end{align*}
    Consequently, we can write
    \begin{align*}
        |Q| + \sqrt{\tsum_{t \notin Q} d_t} 
        &\ge \min_{q \ge 0} \cbr{q + \sqrt{\max\{0,\, \dtot' - q(\sqrt{\dtot'} + 1)\}}} \\
        &\ge \min\cbr{\sqrt{\dtot'} - 1,\, \min_{q \in [0, \sqrt{\dtot'} - 1]} \cbr{q + \sqrt{\dtot' - q(\sqrt{\dtot'} + 1)}}} \\
        &= \sqrt{\dtot'} - 1.
    \end{align*}
    Combining both parts and minimizing over $Q \subseteq [T]$, we conclude
    \begin{align*}
        |Q^*| + \sqrt{\tsum_{t=1}^T d'_t} \le 2\sqrt{\dtot'} + \sqrt{\dtot'} = 3\sqrt{\dtot'} = O\rbr{\min_{Q \subseteq [T]} \cbr{|Q| + \sqrt{\tsum_{t \notin Q} d_t}}},
    \end{align*}
    completing the proof.
\end{proof}

\end{document}